\documentclass[nohyperref]{article}

\usepackage{microtype}
\usepackage{graphicx}
\usepackage{subfigure}
\usepackage{booktabs} 

\usepackage{hyperref}

\usepackage[accepted]{icml2022}

\usepackage{amsmath}
\usepackage{amssymb}
\usepackage{mathtools}
\usepackage{amsthm}

\usepackage[capitalize,noabbrev]{cleveref}

\theoremstyle{plain}
\newtheorem{theorem}{Theorem}[section]
\newtheorem{proposition}[theorem]{Proposition}
\newtheorem{lemma}[theorem]{Lemma}
\newtheorem{corollary}[theorem]{Corollary}
\theoremstyle{definition}

\newtheorem{assumption}[theorem]{Assumption}
\theoremstyle{remark}
\newtheorem{remark}[theorem]{Remark}

\usepackage[textsize=tiny]{todonotes}

\usepackage{times}  
\usepackage{helvet}  
\usepackage{courier}  
\usepackage{graphicx} 
\usepackage{natbib}  
\usepackage{caption} 

\usepackage{colortbl,color}
\usepackage{algorithm}
\usepackage[algo2e]{algorithm2e} 

\usepackage{newfloat}

\usepackage{graphicx}
\usepackage{booktabs}
\usepackage{caption}
\usepackage{comment}
\usepackage{amsfonts, amsmath, amssymb}
\usepackage{mathletters} 
\usepackage{csquotes}

\usepackage{tikz}
\usetikzlibrary {arrows.meta}
\usepackage{thmtools, thm-restate}
\usepackage{bbm}
\usepackage{hyperref}
\usepackage{centernot}

\usepackage{makecell}
\usepackage{cancel}
\usepackage{wrapfig}

\newcommand{\si}{\sum\limits_{i=1}^{M}}
\newcommand{\sj}{\sum\limits_{j=1}^{M}}
\newcommand{\sm}{\sum\limits_{m=1}^{M}}

\newcommand{\Q}{\tilde{Q}}
\newcommand{\A}{\tilde{A}}

\newcommand{\sa}{\sum\limits_{a\in \cA}}

\newcommand{\del}{\frac{d}{d\theta}}
\newcommand{\tr}{\tilde{r}}

\usepackage[algo2e]{algorithm2e}

\icmltitlerunning{Actor-Critic based Improper Reinforcement Learning}

\begin{document}

\twocolumn[
\icmltitle{Actor-Critic based Improper Reinforcement Learning}



\icmlsetsymbol{equal}{*}

\begin{icmlauthorlist}
\icmlauthor{Mohammadi Zaki}{yyy}
\icmlauthor{Avinash Mohan}{comp}
\icmlauthor{Aditya Gopalan}{yyy}
\icmlauthor{Shie Mannor}{sch}
\end{icmlauthorlist}

\icmlaffiliation{yyy}{Department of ECE, IISc, Bangalore, India}
\icmlaffiliation{comp}{Boston University, Massachusetts, USA}
\icmlaffiliation{sch}{Faculty of Electrical Engineering, Technion, Haifa, Israel and NVIDIA Research, Israel.}

\icmlcorrespondingauthor{Mohammadi Zaki}{mohammadi@iisc.ac.in}

\icmlkeywords{Improper Reinforcement Learning, Gradient-based Learning, Actor-critic}

\vskip 0.3in
]




\begin{abstract}
We consider an improper reinforcement learning setting where a
learner is given $M$ base controllers for an unknown Markov
decision process, and wishes to combine them optimally to produce
a potentially new controller that can outperform each of the base
ones. This can be useful in tuning across controllers, learnt possibly in mismatched or simulated  environments, to obtain a good controller for a given target environment with relatively few trials. 
Towards this, we propose two algorithms: (1) a Policy Gradient-based approach; and (2) an algorithm that can switch between a simple Actor-Critic (AC) based scheme and a Natural Actor-Critic (NAC) scheme depending on the available information. Both algorithms operate over a
class of improper mixtures of the given controllers. 
For the first case, we derive convergence rate guarantees assuming access to a gradient oracle. 
For the AC-based approach we provide convergence rate guarantees to a stationary point in the basic AC case and to a global optimum in the NAC case. Numerical results on (i) the standard control theoretic benchmark of stabilizing an cartpole; and (ii) a  constrained
queueing task show that our improper policy optimization algorithm can stabilize the system even when the base policies at its disposal are unstable.

\end{abstract}

\section{Introduction}\label{sec:Introduction}

A natural approach to design effective controllers for large, complex systems is to first approximate the system using a tried-and-true Markov decision process (MDP) model, such as the Linear Quadratic Regulator (LQR) \cite{SarahDean} or tabular MDPs \cite{UCRL}, and then compute (near-) optimal policies for the assumed model. Though this yields favorable results in principle, it is quite possible that errors in describing or understanding the system --  leading to misspecified models -- may lead to `overfitting', resulting in subpar controllers in practice. 
Moreover, in many cases, the stability of the designed controller may be  crucial and more desirable than optimizing a fine-grained cost function. From the controller design standpoint, it is often easier, cheaper and more interpretable to specify or hardcode control policies based on domain-specific principles, e.g., anti-lock braking system (ABS) controllers \cite{ABScitation}. For these reasons, we investigate in this paper a promising,  \emph{general-purpose} reinforcement learning (RL) approach towards designing controllers\footnote{We use the terms 'policy' and 'controller' interchangeably in this article.} given  pre-designed ensembles of \emph{basic} or \emph{atomic} controllers, which (a) allows for flexibly combining the given controllers to obtain richer policies than the atomic policies, and, at the same time, (b) can preserve the basic structure of the given class of controllers and confer a high degree of interpretability on the resulting hybrid policy. 

{\bf Overview of the approach.} We consider a situation where we are given `black-box access' to $M$ controllers (maps from state to action distributions) $\{K_1, \ldots, K_M\}$ for an unknown MDP. By this we mean that we can choose to invoke any of the given controllers at any point during the operation of the system. With the understanding that the given family of controllers is 'reasonable,' we frame the problem of learning the best combination of the controllers by trial and error. We first set up an {\color{blue}improper policy class} of all randomized mixtures of the $M$ given controllers -- each such mixture is parameterized by a probability distribution over the $M$ base controllers. Applying an improper policy in this class amounts to selecting independently at each time a base controller according to this distribution and implementing the recommended action as a function of the present state of the system.  The learner's goal is to find the best performing mixture policy by iteratively testing from the pool of given controllers and observing the resulting state-action-reward trajectory. 

Note that the underlying parameterization in our setting is over a set of given \emph{controllers} which could be potentially abstract and defined for complex MDPs with continuous state/action spaces, instead of the (standard) policy gradient (PG) view where the parameterization directly defines the policy in terms of the state-action map. Our problem, therefore, hews more closely to a \emph{meta RL} framework, in that we operate over a set of controllers that have themselves been designed using some optimization framework to which we are agnostic. This has the advantage of  conferring a great deal of generality, since the class of controllers can now be chosen to promote any desirable secondary characteristic such as interpretability, ease of implementation or cost effectiveness.

It is also worth noting that our approach is different from treating each of the base controllers as an `expert' and applying standard mixture-of-experts algorithms, e.g., Hedge or Exponentiated Gradient \cite{EXP-WTS, Exp3,KocakExpIX,Neu15a}. Whereas the latter approach is tailored to converge to the best single controller (under the usual gradient approximation framework) and hence qualifies as a 'proper' learning algorithm, the former optimization problem is in the improper class of mixture policies which not only contains each atomic controller but also allows for a true mixture (i.e., one which puts positive probability on at least two elements) of many atomic controllers to achieve optimality; we exhibit concrete examples where this is indeed possible. \\
\noindent {\bf Our Contributions.} We make the following contributions in this context:
\vspace{-5pt}
\begin{list}{\labelitemi}{\leftmargin=0.0em}
\item We develop a gradient-based RL algorithm to iteratively tune a softmax parameterization of an improper (mixture) policy defined over the base controllers (Algorithm~\ref{alg:mainPolicyGradMDP}). While this algorithm, \emph{\color{blue}Softmax Policy Gradient} (or Softmax PG), relies on the availability of value function gradients, we later propose a modification that we call  \emph{\color{blue}GradEst} (see Alg.~\ref{alg:gradEst} in appendix) to Softmax PG to rectify this. GradEst uses a combination of rollouts and Simultaneously Perturbed Stochastic Approximation (SPSA) \cite{Borkar} to estimate the value gradient at the current mixture distribution. 
 \vspace{-10pt}
 \item We show a convergence rate of $\mathcal{O}(1/t)$ to the optimal value function for finite state-action MDPs. To do this, we employ a novel Non-uniform Łojasiewicz-type inequality \cite{lojasiewicz1963equations}, that lower bounds the 2-norm of the value gradient in terms of the suboptimality of the current mixture policy's value. Essentially, this helps establish that when the gradient of the value function hits zero, the value function is itself close to the optimum.

\vspace{-5pt}
\item Policy-gradient methods are well-known to suffer from high variance \cite{PetersNAC,BhatnagarNAC}. To circumvent this issue, we develop an algorithm that can switch between a simple Actor-Critic (AC) based scheme and a Natural Actor-Critic (NAC) scheme depending on the available information. The algorithm, `{ACIL}' (Sec.~\ref{sec:Actor-Critic based Improper Learning}), executes on a single sample path, without requiring any \textit{forced} resets, as is common in many RL algorithms. We provide convergence rate guarantees to a stationary point in the basic AC case and to a global optimum in the NAC case, under some additional (but standard) assumptions (of uniform ergodicty). The  total complexity of AC is measured to attain an $(\epsilon +$ $\tt{Critic\_error}$)-accurate stationary point. The  total complexity of NAC is measured to attain an $(\epsilon + \tt{Critic\_error} + \tt{Actor\_error})$-accurate stationary point. We use linear function approximation to approximate the value function and our convergence analysis show exactly how this approximation affects the final complexity bound.
\item We corroborate our theory using extensive simulation studies. For the PG based method we use \textit{GradEst} in two different settings (a) the well-known \emph{CartPole} system and (b) a scheduling task in a constrained queueing system. We discuss both these settings in detail in Sec.~\ref{sec:motivating-examples}, where we also demonstrate the power of our improper learning approach in finding control policies with provably good performance. In our experiments (see Sec.~\ref{sec:simulation-results}), we eschew access to exact value gradients and instead rely on a combination of roll outs and SPSA to \emph{estimate} them. For the actor-critic based learner, we demonstrate simulations on various queuing theoretic simulations using the natural-actor-critic based \emph{ACIL}.  All the results show that our proposed algorithms quickly converge to the correct mixture of available atomic controllers.

    
\end{list}
\vspace{-10pt}
{\bf Related Work (brief).} We provide a quick survey of relevant literature. A detailed survey is deferred to the appendix.
\newline\textbf{Policy gradient.} The basic policy gradient method has become a cornerstone of modern RL and given birth to an entire class of highly efficient policy search techniques such as CPI \cite{kakade2002approximately}, TRPO \cite{TRPO}, PPO \cite{PPO}, and MADDPG \cite{MADDPG}. A growing body of recent work shows promising results about convergence rates for PG algorithms over finite state-action MDPs \cite{Agarwal2020,Shani2020AdaptiveTR,Bhandari2019GlobalOG,Mei2020}, where the parameterization is over the entire space of state -action pairs, i.e., $\Real^{S\times A}$. 
These advances, however, are partially offset by negative results such as those in \citet{li-etal21softmax-policy-gradient-exponential-converge-time}, which show that the convergence time is $\Omega \left(\abs{\cS}^{2^{1/(1-\gamma)}}\right)$, where $\cS$ is the state space of the MDP and $\gamma$ the discount factor, even with exact gradient knowledge. 
\newline\textbf{Improper learning.} The above works concern \emph{proper} learning, where the policy search space is usually taken to be the set of all deterministic policies for an MDP. {\em Improper} learning, on the other hand, has been studied in statistical learning theory for the IID setting \cite{Improper1, Improper2}. In this \emph{representation independent} learning framework, the learning algorithm is not restricted to output a hypothesis from a given set of hypotheses. 
\newline\textbf{Boosting.} \citet{hazan-etal20boosting-control-dynamical} attempts to frame and solve policy optimization over an improper class by boosting a given class of controllers. This work, however, is situated in the context of non-stochastic control and assumes perfect knowledge of (i) the memory-boundedness of the MDP, and (ii) the state noise vector in every round, which amounts to essentially knowing the MDP transition dynamics. We work in the stochastic MDP setting and assume no access to the MDP's transition kernel. Further, it is assumed in  \cite{hazan-etal20boosting-control-dynamical} that all the atomic controllers available  are \emph{stabilizing} which, when working with an unknown MDP, is a very strong assumption to make. While making no such assumptions on our atomic controller class; we show our algorithms can begin with provably unstable controllers and yet succeed in stabilizing the system (Sec.~\ref{sec:motivating-queueing-example} and ~\ref{sec:simulation-results}). 
\newline\textbf{Options framework.} Our work differs from the options framework \cite{options1, options2} for hierarchical RL {\em in spirit}, in that we allow for each controller to be applied in each round rather than waiting for a sub-task to complete. The current work deals with finding an optimal mixture of basic controllers to solve a particular task. However, if we allow for a state-dependent choice of controllers, then the methods proposed can be generalized for solving hierarchical RL tasks.
\newline \textbf{Ensemble policy-based RL.} Our current work deals with accessing {\em  given} (possibly separately trained) controllers as \textit{black-boxes} and {\em learning to combine} them optimally. In contrast, in ensemble RL approaches \cite{ensembleRL1,ensembleRL2,ensembleRL3} the base policies are learnt on the fly (e.g., Q-learning, SARSA) by the agent whereas the {\em combining rule is fixed upfront} (e.g., majority voting, rank voting, Boltzmann multiplication, etc.). Moreover, the base policies \emph{have access} to the new system in Ensemble RL, which gives them a distinct advantage. Our method can serve as a meta-RL adaptation framework with theoretical guarantees which can use such pre-trained models to combine them optimally. 
To the best of our knowledge, ensemble RL works like \cite{ensembleRL2,ensembleRL3} do not provide  theoretical guarantees on the learnt combined policy. Our work on the other hand provides a firm theoretical as well as empirical basis for the methods we propose.
\newline\textbf{Improper learning with given base controllers.} Probably the closest resemblance with our work is that of \citet{pmlr-v89-banijamali19a} which aims at finding the best convex combination of a given set of base controllers for a given MDP. They however frame it as a \textit{planning} problem where the transition kernel $P$ is known to the agent. Furthermore, we treat the base controllers as black-box entities, whereas they exploit their structure to compute the state-occupancy measures. 
\newline\textbf{Actor-critic methods.} Actor-critic (AC) methods were first introduced in \citet{Konda_NIPS_AC}. Natural actor-critic methods were first introduced in \cite{PetersNAC,BhatnagarNAC}. While many studies are available for the asymptotic convergence of AC and NAC, we use the new techniques proposed by \citet{Improving-AC-NAC} and \citet{Target-AC} for showing convergence results. 
\vspace{-5pt}
\section{Motivating Examples}\label{sec:motivating-examples}
We begin with two examples that help illustrate the need for improper learning over a given set of atomic controllers. These examples concretely  demonstrate the power of this approach to find (improper) control policies that go well beyond what the atomic set can accomplish, while retaining some of their desirable properties (such as interpretability and simplicity of implementation).
\vspace{-5pt}
\subsection{Ergodic Control of the Cartpole System}\label{sec:motivating-cartpole-example}

Consider the Cartpole system which has, over the years, become a benchmark for testing control strategies \cite{khalil15nonlinear-control}. 
The system's dynamics, evolving in $\Real^4$, can be approximated via a Linear Quadratic Regulator around an (unstable) equilibrium state vector that we designate the origin ($\mathbf{x}=\mathbf{0}$).
The objective now reduces to finding a (potentially randomized) control policy $u\equiv \{u(t),t\geq0\}$ that solves $\inf_{u} J\left(\mathbb{E}_{u}\sum_{t=0}^\infty \mathbf{x}^\intercal(t) Q \mathbf{x}(t) + R u^2(t) \right)$ subject to $\mathbf{x}(t+1)=A_{open}\mathbf{x}(t) + \mathbf{b} u(t)$ at all times $t\geq0$.

\looseness=-1 Under standard assumptions of controllability and observability, this optimization has a stationary, linear solution $u^*(t) = -\mathbf{K}^\intercal\mathbf{x}(t)$ ( \cite{bertsekas11dynamic}). Moreover, setting $A:=A_{open}-\mathbf{b}\mathbf{K}^\intercal$, it is well know that the dynamics ${\mathbf{x}}(t+1) = A \mathbf{x}(t),~t\geq0,$
are stable. 
The usual design strategy for a given Cartpole involves a combination of system identification, followed by linearization and computing the controller gain $\mathbf{K}$. This would typically produce a controller with tolerable performance fairly quickly, but would also suffer from nonidealities of parameter estimation.

To alleviate this problem, first consider a generic (ergodic) control policy that builds on this strategy by switching across a menu of controllers $\{K_1,\cdots,K_N\}$ produced as above. That is, at any time $t$, this policy chooses $K_i,~i\in[N],$ w.p. $p_i$, so that the control input at time $t$ is $u(t)=-\mathbf{K}_i^\intercal\mathbf{x}(t)$ w.p. $p_i.$ Let $A(i) := A_{open}-\mathbf{b}\mathbf{K}_i^\intercal$. The resulting {\em controlled} dynamics are given by $\mathbf{x}(t+1) = A(r(t)) \mathbf{x}(t),~t\geq0,$
where $r(t)=i$ w.p. $p_i,$ IID across $t$.  

This is an example of an \emph{\color{black} ergodic parameter linear system} (EPLS) \cite{bolzern-etal08almost-sure-stability-ergodic-linear}, which is said to be \emph{\color{blue}Exponentially Almost Surely Stable} (EAS) if the state norm decays at least exponentially fast with time: $\mathbb{P}\left\lbrace \limsup_{t\rightarrow\infty}\frac{1}{t}\log{\norm{\mathbf{x}(t)}}\leq-\rho\right\rbrace = 1$ for some $\rho > 0$. Let the random variable $\lambda(\omega):=\limsup_{t\rightarrow\infty}\frac{1}{t}\log{\norm{\mathbf{x}(t,\omega)}}$. For our dynamics $\mathbf{x}(t+1) = A(r(t)) \mathbf{x}(t),~t\geq0$, it is seen that the {\em Lyapunov exponent}  $\frac{1}{t}\log{\norm{\mathbf{x}(t)}}$ is at most the quantity $\sum_{i=1}^N p_i \log{\norm{A(i)}} a.s.$  (see appendix for details).  

A good mixture controller can now be designed by choosing $\{p_1,\cdots,p_N\}$ such that $\lambda(\omega)<-\rho$ for some $\rho>0$, ensuring exponentially almost sure stability (subject to $\log \norm{A(i)} < 0$ for some $i$). As we show in the sequel, our policy gradient algorithm (SoftMax PG) learns an improper mixture $\{p_1,\cdots,p_N\}$ that (i) can stabilize the system even when a majority of the constituent \emph{atomic} controllers $\{K_1,\cdots,K_N\}$ are \emph{unstable}, i.e., converges to a mixture that ensures that the average exponent $\lambda(\omega)<0$,  and (ii) shows better performance than that each of the atomic controllers. 

\vspace{-5pt}
\subsection{Scheduling in Constrained Queueing Networks} \label{sec:motivating-queueing-example}

\begin{wrapfigure}{r}{0.3\textwidth}
    \begin{minipage}{0.3\textwidth}
        \vspace{-10pt}
        \tikzset{every picture/.style={line width=0.75pt}} 
\tikzset{fontscale/.style = {font=\relsize{#1}}
    }
\resizebox{5.00cm}{3.5cm}{

\begin{tikzpicture}[x=0.57pt,y=0.57pt,yscale=-1,xscale=1]

\draw [line width=2.25]  (22,380.5) -- (652.8,380.5)(85.08,6.97) -- (85.08,422) (645.8,375.5) -- (652.8,380.5) -- (645.8,385.5) (80.08,13.97) -- (85.08,6.97) -- (90.08,13.97)  ;
\draw [line width=1.5]    (85.8,22.98) -- (573.8,380.97) ;
\draw [color={rgb, 255:red, 74; green, 99; blue, 226 }  ,draw opacity=1 ][line width=2.25]    (85.08,379.4) -- (115.27,94.95) ;
\draw [shift={(115.8,89.98)}, rotate = 456.06] [fill={rgb, 255:red, 74; green, 99; blue, 226 }  ,fill opacity=1 ][line width=0.08]  [draw opacity=0] (16.07,-7.72) -- (0,0) -- (16.07,7.72) -- (10.67,0) -- cycle    ;
\draw [color={rgb, 255:red, 74; green, 99; blue, 226 }  ,draw opacity=1 ][line width=2.25]    (85.08,379.4) -- (461.82,349.37) ;
\draw [shift={(466.8,348.98)}, rotate = 535.44] [fill={rgb, 255:red, 74; green, 99; blue, 226 }  ,fill opacity=1 ][line width=0.08]  [draw opacity=0] (16.07,-7.72) -- (0,0) -- (16.07,7.72) -- (10.67,0) -- cycle    ;
\draw  [fill={rgb, 255:red, 128; green, 128; blue, 128 }  ,fill opacity=0.3 ] (85.08,89.98) -- (115.8,89.98) -- (115.8,379.4) -- (85.08,379.4) -- cycle ;
\draw  [fill={rgb, 255:red, 128; green, 128; blue, 128 }  ,fill opacity=0.3 ] (85.08,348.98) -- (466.8,348.98) -- (466.8,379.4) -- (85.08,379.4) -- cycle ;
\draw  [fill={rgb, 255:red, 87; green, 19; blue, 254 }  ,fill opacity=0.3 ][dash pattern={on 2.53pt off 3.02pt}][line width=2.25]  (115.8,89.98) -- (466.8,348.98) -- (115.8,348.98) -- cycle ;
\draw    (48.8,145) -- (77.8,145) -- (98.8,145) ;
\draw [shift={(100.8,145)}, rotate = 180] [fill={rgb, 255:red, 0; green, 0; blue, 0 }  ][line width=0.08]  [draw opacity=0] (12,-3) -- (0,0) -- (12,3) -- cycle    ;
\draw    (412.8,413.97) -- (412.8,366.97) ;
\draw [shift={(412.8,364.97)}, rotate = 450] [fill={rgb, 255:red, 0; green, 0; blue, 0 }  ][line width=0.08]  [draw opacity=0] (12,-3) -- (0,0) -- (12,3) -- cycle    ;
\draw    (287.8,90.2) -- (189.45,188.55) ;
\draw [shift={(188.04,189.96)}, rotate = 315] [color={rgb, 255:red, 0; green, 0; blue, 0 }  ][line width=0.75]    (10.93,-3.29) .. controls (6.95,-1.4) and (3.31,-0.3) .. (0,0) .. controls (3.31,0.3) and (6.95,1.4) .. (10.93,3.29)   ;
\draw    (484,198) -- (427.08,266.44) ;
\draw [shift={(425.8,267.97)}, rotate = 309.75] [color={rgb, 255:red, 0; green, 0; blue, 0 }  ][line width=0.75]    (10.93,-3.29) .. controls (6.95,-1.4) and (3.31,-0.3) .. (0,0) .. controls (3.31,0.3) and (6.95,1.4) .. (10.93,3.29)   ;

\draw (628,388.4) node [anchor=north west][inner sep=0.75pt]  [font=\LARGE]  {$\mathbf{\lambda _{1}}$};
\draw (48,22.4) node [anchor=north west][inner sep=0.75pt]  [font=\LARGE]  {$\mathbf{\lambda _{2}}$};
\draw (4,82.4) node [anchor=north west][inner sep=0.75pt] [font=\LARGE]    {$( 0,1-\epsilon )$};
\draw (435,388.4) node [anchor=north west][inner sep=0.75pt]  [font=\LARGE]   {$( 1-\epsilon ,0)$};
\draw (26,341.4) node [anchor=north west][inner sep=0.75pt]  [font=\LARGE]   {$( 0,\epsilon )$};
\draw (96,389.4) node [anchor=north west][inner sep=0.75pt]  [font=\LARGE]   {$( \epsilon ,0)$};
\draw (124,320) node [anchor=north west][inner sep=0.75pt]   [align=left] {\textbf{{\Large A}}};
\draw (121,69) node [anchor=north west][inner sep=0.75pt]   [align=left] {\textbf{{\large B}}};
\draw (468,335) node [anchor=north west][inner sep=0.75pt]   [align=left] {\textbf{{\large C}}};
\draw (13,138.4) node [anchor=north west][inner sep=0.75pt]    {$\mathbf{C_{2}}$};
\draw (401.8,418.37) node [anchor=north west][inner sep=0.75pt]    {$\mathbf{C_{2}}$};
\draw (237,46) node [anchor=north west][inner sep=0.75pt]   [align=left] {\begin{minipage}[lt]{126.871pt}\setlength\topsep{0pt}
\begin{center}
\begin{LARGE}
extra capacity achieved\\through \textbf{improper learning}
\end{LARGE}
\end{center}

\end{minipage}};
\draw (472,170.4) node [anchor=north west][inner sep=0.75pt]  [font=\LARGE]  {$\mathbf{\lambda _{1} \ +\ \lambda _{2} =1}$};

\end{tikzpicture}
}
        \vspace{-5pt}
        \caption{\footnotesize{$K_1$ and $K_2$ by themselves can only stabilize $\mathcal{C}_1\cup\mathcal{C}_2$ (gray rectangles). With improper learning, we enlarge the set of stabilizable arrival rates by the triangle $\Delta ABC$ shown in purple, above.}}
        \label{fig:capacity-and-stability-regions}
    \end{minipage}
    \vspace{-0.5cm}
\end{wrapfigure}
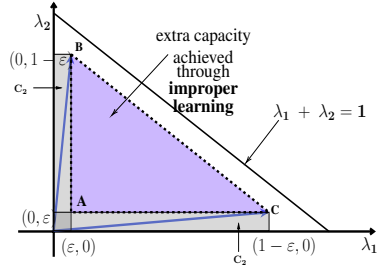

We consider a system that comprises two queues fed by independent, stochastic arrival processes $A_i(t),i\in\{1,2\},t\in\mathbb{N}$. The length of queue~$i$, measured at the beginning of time slot $t,$ is denoted by $Q_i(t)\in\mathbb{Z}_+$. A common server serves both queues and can drain at most one packet from the system in a time slot. The server, therefore, needs to decide which of the two queues it intends to serve in a given slot (we assume that once the server chooses to serve a packet, service succeeds with probability 1). The server's decision is denoted by the vector $\mathbf{D}(t)\in\mathcal{A}:=\left\lbrace[0,0],[1,0],[0,1]\right\rbrace,$ where a \enquote{$1$} denotes service and a \enquote{$0$} denotes lack thereof.
Let $\mathbb{E}A_i(t)=\lambda_i,$ and note that the arrival rate $\boldsymbol{\lambda}=[\lambda_1,\lambda_2]$ is {\color{blue}unknown} to the learner. 
We aim to find a (potentially randomized) policy $\pi$
to minimize the discounted system backlog given by 
$
     J_\pi(\mathbf{Q}(0)):=\mathbb{E}^\pi_{\mathbf{Q}(0)}\sum_{t=0}^{\infty}\gamma^t\left(Q_1(t)+Q_2(t)\right).
$

Any policy with $J_\pi(\cdot)<\infty,$ 
is said to be \emph{\color{blue}stabilizing} (or, equivalently, a \emph{stable} policy). It is well known that there exist stabilizing policies iff $\lambda_1+\lambda_2<1$ \cite{tassiulas-ephremides92stability-queueing}. A policy $\pi_{\mu_1,\mu_2}$ that chooses Queue~$i$ w.p. $\mu_i$ in every slot, 
can provably stabilize a system iff $\mu_i>\lambda_i,\forall~i\in\{1,2\}$. Now, assume our control set consists of two stationary policies $K_1,K_2$ with $K_1\equiv\pi_{\epsilon,1-\epsilon}$, $K_1\equiv\pi_{1-\epsilon,\epsilon}$ and sufficiently small $\epsilon>0.$ That is, we have $M=2$ controllers $K_1,K_2.$ Clearly, neither of these can, by itself, stabilize a network with $\boldsymbol{\lambda}=[0.49,0.49].$ 

However, an \emph{improper} mixture of the two that selects $K_1$ and $K_2$ each with probability $1/2$ can. In fact, as Fig.~\ref{fig:capacity-and-stability-regions} shows, our improper learning algorithm can stabilize {\em all} arrival rates in $\mathcal{C}_1\cup\mathcal{C}_2\cup\Delta ABC,$ without prior knowledge of $[\lambda_1,\lambda_2].$ In other words, our algorithm enlarges the stability region by the triangle $\Delta ABC,$ over and above $\mathcal{C}_1\cup\mathcal{C}_2$. 
We will return to these examples in Sec.~\ref{sec:simulation-results}, and show, using experiments, (1) how our improper learner converges to the stabilizing mixture of the available policies and (2) if the optimal policy is among the available controllers, how our algorithm can find and converge to it.

\vspace{-5pt}
\section{Problem Statement and Notation}\label{sec:problemStatementAndNotation}

A (finite) Markov Decision Process {\color{blue}$(\cS, \cA, \tP, r,\rho, \gamma)$} is specified by a finite state
space $\cS$, a finite action space $\cA$, a transition probability matrix $\tP$, where $\tP\left(\Tilde{s}|s, a\right)$ is the probability of transitioning into state $\Tilde{s}$ upon taking action $a\in\cA$ in state $s$, a single stage reward function $r:\cS\times\cA\rightarrow\mathbb{R}$, a starting state distribution $\rho$ over $\cS$ and a discount factor $\gamma \in (0, 1)$.
A (stationary) \emph{policy} or \emph{\color{blue}controller} $\pi : \cS \rightarrow \cP(\cA)$ specifies a decision-making strategy in which the learner chooses actions ($a_t$) adaptively based on the current state ($s_t$), i.e., $a_t\sim\pi(s_t)$. $\pi$ and $\rho$, together with $\tP,$ induce a probability measure $\mathbb{P}^{\pi}_{\rho}$ on the space of all sample paths of the underlying Markov process and we denote by $\mathbb{E}^{\pi}_{\rho}$ the associated expectation operator. The {\color{blue}value function} of policy $\pi$ (also called the value of policy $\pi$), denoted by $V^{\pi}$ is the total discounted reward obtained by following $\pi$, i.e., $V^{\pi}(\rho):=\mathbb{E}^{\pi}_{\rho}\sum_{t=0}^{\infty}\gamma^tr(s_t,a_t).$

\noindent{\bfseries{Improper Learning.}} We assume that the learner is provided with a finite number of (stationary) controllers $\mathcal{C}:=\{K_1,\cdots,K_M\}$ and, as described below, set up a parameterized improper policy class ${\color{blue}\mathcal{I}_{soft}(\mathcal{C})}$ that depends on $\mathcal{C}.$ The aim therefore, is to identify the best policy for the given MDP within this class, i.e.,
\begin{equation}\label{eq:main optimization problem}
\pi^* = \argmax\limits_{\pi\in \mathcal{I}_{soft}(\mathcal{C})} V^{\pi}(\rho).     
\end{equation}
We now describe the construction of the class $\mathcal{I}_{soft}(\mathcal{C}).$
\newline{\bfseries{The Softmax Policy Class.}} We assign weights $\theta_m\in \Real$, to each controller $K_m\in\mathcal{C}$ and define  $\theta:=[\theta_1,\cdots,\theta_M]$. The improper class $\mathcal{I}_{soft}$ is parameterized by $\theta$ as follows. In each round, the policy $\pi_{\theta}\in\mathcal{I}_{soft}(\mathcal{C})$ chooses a controller drawn from {\color{blue}$\softmax(\theta)$}, i.e., the probability of choosing Controller~$K_m$ is given by, $\pi_\theta(m):= {e^{\theta_m}}/\left({\sum_{m'=1}^M e^{\theta_{m'}}}\right).$
Note, therefore, that in every round, our algorithm interacts with the MDP only \emph{through} the controller sampled in that round. In the rest of the paper, we will deal exclusively with a fixed and given $\mathcal{C}$ and the resultant  $\mathcal{I}_{soft}$. therefore, we overload the notation $\pi_{\theta_t}(a|s)$ for any $a\in \cA$ and $s\in \cS$ to denote the probability with which the algorithm chooses action $a$ in state $s$ at time $t$. For ease of notation,  whenever the context is clear, we will also drop the subscript $\theta$ i.e., $\pi_{\theta_t}\equiv \pi_t$. Hence, we have at any time $t\geq 0: \pi_{\theta_t}(a|s) = \sum_{m=1}^{M}\pi_{\theta_t}(m)K_m(s,a).$ 
Since we deal with gradient-based methods in the sequel, we define the \emph{\color{blue}value gradient} of policy $\pi_{\theta}\in\mathcal{I}_{soft},$ by $\nabla_{\theta}V^{\pi_\theta}\equiv\frac{dV^{\pi_{\theta_t}}}{d\theta^t}$. We say that $V^{\pi_\theta}$ is {\color{blue}$\beta$-smooth} if $\nabla_{\theta}V^{\pi_\theta}$ is $\beta$-Lipschitz \cite{Agarwal2020}. Finally, let for any two integers $a$ and $b$, $\ind_{ab}$ denote the indicator that $a=b$.
\newline\textbf{Comparison to the standard PG setting.}
This problem we define is different from the usual policy gradient setting where the parameterization completely defines the policy in terms of the state-action mapping. One can use the methodology followed in \cite{Mei2020}, by assigning a parameter $\theta_{s,m}$ for every $s\in \cS, m\in [M]$. With some calculation, it can be shown that this is equivalent to the tabular setting with $S$ states and $M$ actions, with the new `reward' defined by $r(s,m)\bydef \sum_{a\in \cA} K_m(s,a)r(s,a)$ where $r(s,a)$ is the usual expected reward obtained at state $s$ and playing action $a\in \cA$. By following the approach in \cite{Mei2020} on this modified setting, it can be shown that the policy converges for each $s\in \cS$, $\pi_\theta(m^*(s)\given s) \to 1$, for every $s\in \cS$, which is the optimum policy. However, the problem that we address, is to select \emph{a single} controller (from within $\mathcal{I}_{soft}$, the convex hull of the given $M$ controllers) , which would guarantee maximum return if one plays that single mixture for all time, from among the given set of controllers. 
\vspace{-5pt}
\section{Improper Learning using Gradients}\label{sec:PG theory}
\setlength{\textfloatsep}{5pt}
\begin{algorithm}[tb]
   \caption{{\color{blue}SoftMax PG}}
   \label{alg:mainPolicyGradMDP}
\begin{algorithmic}
   \STATE {\bfseries Input:} learning rate $\eta>0$, initial state distribution $\mu$
   \STATE {\bfseries Initialize:} each $\theta^1_m=1$, for all $m\in [M]$, $s_1\sim \mu$.
   \FOR{$t=1$ {\bfseries to} $T$}
   \STATE Choose controller $m_t\sim \pi_t$
   \STATE Play action $a_t \sim K_{m_t}(s_{t},:)$
   \STATE Observe $s_{t+1}\sim \tP(.|s_t,a_t)$
   \STATE Update: $\theta_{t+1} = \theta_{t} + \eta 
   \nabla_{\theta_t}V^{\pi_{\theta_t}}$\;\label{algo:softmaxpg-gradAscent}
   \ENDFOR
\end{algorithmic}
\end{algorithm}


In this and the following sections, we propose and analyze a policy gradient-based algorithm that provably finds the best, potentially improper, mixture of controllers for the given MDP. While we employ gradient ascent to optimize the mixture weights, the fact that this procedure works at all is far from obvious. We begin by noting that $V^{\pi_\theta}$, as described in Section \ref{sec:problemStatementAndNotation}, is \emph{\color{blue}nonconcave} in $\theta$ for both direct and softmax parameterizations, which renders analysis with standard tools of convex optimization inapplicable. 
\begin{lemma}(Non-concavity of Value function)\label{lemma:nonconcavity of V main}
There is an MDP and a set of controllers, for which the maximization problem  of the value function (i.e. \eqref{eq:main optimization problem}) is non-concave for both the SoftMax and direct parameterizations, i.e., $\theta\mapsto V^{\pi_{\theta}}$ is non-concave.
\end{lemma}
The proof follows from a counterexample whose construction we show  in the appendix. Our PG algorithm, {\color{blue}SoftMax PG}, is shown in Algorithm~\ref{alg:mainPolicyGradMDP}. The parameters $\theta\in \Real^M$ which define the policy are updated by following the gradient of the value function at the current policy parameters. 

\noindent{\bfseries{Convergence Guarantees.}}
The following result shows that with SoftMax PG, the value function converges to that of the \emph{best in-class} policy at a rate $\cO\left(1/t\right)$. Furthermore, the theorem shows an explicit dependence on the number of controllers $M$, in place of the usual $\abs{\cS}$. Note that with perfect gradient knowledge the algorithm becomes deterministic. This is a standard assumption in the analysis of PG algorithms \cite{fazel2019global, Agarwal2020, Mei2020}. 
\begin{restatable}[Convergence of Policy Gradient]{theorem}{maintheorem}\label{thm:convergence of policy gradient}
With $\{\theta_t \}_{t\geq 1}$ generated as in Algorithm \ref{alg:mainPolicyGradMDP} and using a learning rate $\eta = \frac{\left(1-\gamma\right)^2}{7\gamma^2+4\gamma+5}$, for all $t\geq 1$,
    $V^*(\rho) -V^{\pi_{\theta_t}}(\rho) = \mathcal{O}\left({\color{blue}\frac{1}{t}}  \frac{M\gamma^2}{c_t^2(1-\gamma)^3}\right)$, 
where $c_t\bydef \min\limits_{1\leq s\leq t} \min\limits_{m:\pi^*(m)>0}\pi_{\theta_s}(m)$.
\end{restatable}
\begin{remark}
The quantity $c_t$ in the statement is the minimum probability that SoftMax PG puts on the controllers for which the best mixture $\pi^*$ has positive probability mass. Empirical evidence (Sec.~\ref{sec:simulation-results}) makes us conjecture that $\lim\limits_{t\to \infty}c_t$ is {\color{blue}{positive}}, which shows a convergence rate of $\mathcal{O}\left(1/t\right)$. 
\end{remark}
\begin{remark}
The proof of the above theorem uses the $\beta-$ smoothness property of the value function under the softmax parameterization along with a new non-uniform {\L}ojaseiwicz-type inequality (NU{\L}I) for our probabilistic mixture class, which lower bounds the magnitude of the gradient of the value function, which we mention below.
\end{remark}
\begin{restatable}[NU{\L}I]{lemma}{nonuniformLE}\label{lemma:nonuniform lojaseiwicz inequality}
$\norm{\frac{\partial}{\partial\theta}V^{\pi_\theta}(\mu)}_2 \geq \frac{1}{\sqrt{M}}\left(\min\limits_{m:\pi^*_{\theta_{m}}>0} \pi_{\theta_m} \right) \times \norm{\frac{d_{\rho}^{\pi^*}}{d_{\mu}^{\pi_\theta}}}_{\infty}^{-1}
\times\left[V^*(\rho) -V^{\pi_\theta}(\rho) \right]. $
\end{restatable}
The proof of Theorem \ref{thm:convergence of policy gradient}, then follows by  an induction argument over $t\geq 1$. 

\textbf{Technical Challenges.} We note here that while the basic recipe for the analysis of Theorem \ref{thm:convergence of policy gradient} is similar to \cite{Mei2020}, our setting does not directly inherit the intuition of standard PG (sPG) analysis. 
{\color{blue}\bf(1)} With $\abs{\mathcal{S}\times\mathcal{A}}<\infty,$ the sPG analysis critically depends on the fact that a  deterministic optimal policy exists and shows convergence to it. In contrast, in our setting, $\pi^*$ could be a strictly randomized mixture of the base controllers (see Sec. \ref{sec:motivating-examples}). {\color{blue}\bf(2)} A crucial step in sPG analysis is establishing that the value function $V^\pi(s), \forall s\in \cS$ increases monotonically with time such that parameter of the optimal action $\theta_{s,a^*}\uparrow\infty$. In the appendix, we supply a simple counterexample showing that monotonicity of the $V$ function is not guaranteed in our setting for every $s\in \cS$. 
{\color{blue}\bf(3)} The value function gradient in sPG has no `cross contamination' from other states, in the sense that modifying the parameter at one state does not affect the values of the others. This plays a crucial part in simplifying the proof of global convergence to the optimal policy in sPG analysis. Our setting cannot leverage this property since the value function gradient at a given {\em controller} possesses contributions from \emph{all} states.

For the special case of $S=1$, which is the Multiarmed Bandits, each controller is a probability distribution over the $A$ arms of the bandit. 
We call this special case \emph{\color{blue}Bandit-over-Bandits}. 
We obtain a convergence rate of $\mathcal{O}\left(M^2/t\right)$  to the optimum and recover ${\color{blue}{M^2}}\log T$ regret bound when our softmax PG algorithm is applied to this special case. We refer to the appendix for details.
\newline{\bfseries{Discussion on $c_t$.}}
Convergence in Theorem~\ref{thm:convergence of policy gradient} depends inversely on $c^2_t$. It follows that in order for SoftMax PG to converge, $c_t$ must either (a) converge to a positive constant, or (b) decay (to $0$) slower than $\mathcal{O}\left(1/\sqrt{t} \right)$. The technical challenges discussed above, render proving this extremely hard analytically. Hence, while we currently do not show this theoretically, our experiments in Sec.~\ref{sec:simulation-results} repeatedly confirm that its empirical analog, i.e., $\bar{c}_t$ (defined formally in Sec.~\ref{sec:simulation-results})  approaches a {\em positive} value. Hence, we conjecture that the rate of convergence in Thm~\ref{thm:convergence of policy gradient} is $\mathcal{O}(1/t)$. 

\section{Actor-Critic based Improper Learning}\label{sec:Actor-Critic based Improper Learning}
Softmax PG follows a gradient ascent scheme to solve the optimization problem \eqref{eq:main optimization problem}, but is limited by the requirement of the true gradient in every round. To address situations where this might be unavailable, we resort to a Monte-carlo sampling based procedure (see appendix: Alg~\ref{alg:gradEst}), which may lead to high variance. In this section, we take an alternative approach and provide a new algorithm based on an actor-critic framework for solving our problem. Actor-Critic methods are well-known to have low variance than their Monte-carlo counterparts \cite{Konda_NIPS_AC}.

We begin by proposing modifications to the standard $Q$-function and advantage function definitions. Recall that we wish to solve for the following optimization problem: $\max_{\pi\in \cI_{\tt{soft}}}\mathbb{E}_{s\sim\rho}[V^{\pi}(s)]$, where $\pi$ is some distribution over the $M$ base controllers. Let $\tilde{Q}^\pi(s,m):=\sum_{a\in \cA}K_m(s,a)Q^\pi(s,a)$. Let $\tilde{A}^\pi(s,m) := \sum_{a\in \cA}K_m(s,a)A^\pi(s,a) = \sum_{a\in \cA}K_m(s,a)Q^\pi(s,a)-V^\pi(s)$, where $Q^\pi$ and $A^\pi$ are the usual action-value functions and advantage functions respectively. We also define the new reward function $\tilde{r}(s,m):=\sum_{a\in \cA}K_m(s,a)r(s,a)$ and a new transition kernel $\tilde{P(s'|s,m)}:=\sum_{a\in \cA}K_m(s,a)P(s'|s,a)$. Then, following the distribution $\pi$ over the controllers induces a Markov Chain on the state space $\cS$. Define $\nu_\pi(s,m)$ as the state-controller visitation measure induced by the \textit{policy} $\pi$: $\nu_\pi(s,m):=(1-\gamma)\sum_{t\geq 0}\gamma^t\mathbb{P}^\pi(s_t=s,m_t=m) = d_\rho^\pi(s)\pi(m)$. With these definitions, we have the following variant of the policy-gradient theorem.

\begin{restatable}[Modified Policy Gradient Theorem]{lemma}{maintheorem}
$\nabla_\theta V^{\pi_\theta}(\rho) = \mathbb{E}_{(s,m)\sim \nu_{\pi_\theta}}[\tilde{Q}^{\pi_\theta}(s,m)\psi_\theta(m)] = \mathbb{E}_{(s,m)\sim\nu_{\pi_\theta}}[\tilde{A}^{\pi_\theta}(s,m) \psi_\theta(m)],$
 where $\psi_\theta(m):=\nabla_{\theta}\log(\pi_{\theta}(m))$.
\end{restatable}

 Note the independence of the score function $\psi$ from the state $s$.  For the gradient ascent update of the parameters $\theta$ we need to estimate $\tilde{A}^{\pi_\theta}(s,m)$ where $(s,m)$ are drawn according to $\nu_{\pi_{\theta}}(\cdot,\cdot)$. We recall how to sample from $\nu_\pi$. Following  \citet{Konda_NIPS_AC} and the recent works like \citet{Improving-AC-NAC, Target-AC} and casting into our setting, observe that $\nu_\pi$ is a stationary distribution of a Markov chain over the pair $(s,m)$ with state-to-state transition kernel defined by $\bar{P}(s'|s,m):=\gamma \tilde{P}(s'|s,m) + (1-\gamma) \rho(s')$ and $m\sim \pi(.)$.
 
 \textbf{Algorithm Description.} We present the algorithm in detail in Algorithm \ref{alg:actor-critic improper RL alg} along with a subroutine Alg~\ref{alg:Critic-TD} which updates the critic's parameters. ACIL is a single-trajectory based algorithm, in the sense that it does not require a forced reset along the run. We begin with the critic's updates. The \textbf{critic} uses linear function approximation $V_w(s):=\phi(s)^\top w$, and uses TD learning to update its parameters $w\in \Real^d$. We assume that $\phi(\cdot):\cS\to\Real^d$ is a known feature mapping. Let $\Phi$ be the corresponding $\abs{S}\times d$ matrix. We assume that the columns of $\Phi$ are linearly independent. Next, based on the critic's parameters, the \textbf{actor} approximates the $\A(s,m)$ function  using the TD error: $\cE_{w}(s, m, s')= \tr(s, m) + \left(\gamma\phi(s')- \phi(s)\right)^\top w$.  


\setlength{\textfloatsep}{5pt}


\begin{algorithm}[tb]
\caption{{\color{blue}Actor-Critic based Improper RL (ACIL)}}
\label{alg:actor-critic improper RL alg}
\begin{algorithmic}
\STATE {\bfseries Input:} {$\phi$, actor stepsize $\alpha$, critic stepsize $\beta$, regularization parameter $\lambda$, '{\tt AC}' or '{\tt NAC}' }
\STATE {\bfseries Initialize:} $\theta_0=(1,1,\ldots,1)_{M\times 1}$, $s_0\sim \rho$
\STATE ${\tt{flag}} = \mathbbm{1}\{\texttt{NAC} \}$ {\color{blue}\COMMENT{Selects AC or NAC}}
\FOR{$t\leftarrow 0$ {\bfseries{ to }} $T-1$}
\STATE $s_{init} = s_{t-1,B}$ (when $t=0$, $s_{init}=s_0$)
\STATE $w_t, s_{t,0} \gets {\color{blue}\tt{Critic-TD}}(s_{init},    \pi_{\theta_t}, \phi, \beta, T_c, H)$ 
\STATE $F_t(\theta_t) \gets 0$.
\FOR{$i\gets0$ {\bfseries{ to }} $B-1$}
\STATE $m_{t,i}\sim \pi_{\theta_t}$, $a_{t,i}\sim K_{m_{t,i}}(s_{t,i},.)$ 
\STATE $s_{t,{i+1}}\sim \bar{P}(.|s_{t,i}, m_{t,i})$
\STATE $\cE_{w_t}(s_{t,i}, m_{t,i}, s_{t,i+1})= \tr(s_{t,i}, m_{t,i}) + \left(\gamma\phi(s_{t,i+1})- \phi(s_{t,i})\right)^\top w_t    $
\STATE $F_t(\theta_t) \gets F_t(\theta_t) + \frac{1}{B} \psi_{\theta_t}(m_{t,i})\psi_{\theta_t}(m_{t,i})^\top$
\ENDFOR
\IF {\{{\tt{flag}}\}}
\STATE $G_t := [F_t(\theta_t)+\lambda I]$
\STATE $\theta_{t+1} = \theta_t + G_t^{-1} \frac{\alpha}{B} \sum\limits_{i=0}^{B-1} \cE_{w_t}(s_{t,i}, m_{t,i}, s_{t,i+1}) \psi_{\theta_t}(m_{t,i}) $
%
\ELSE 
\STATE $\theta_{t+1} = \theta_t + \frac{\alpha}{B} \sum\limits_{i=0}^{B-1} \cE_{w_t}(s_{t,i}, m_{t,i}, s_{t,i+1})\psi_{\theta_t}(m_{t,i}) $

\ENDIF
\STATE $\pi_{\theta_{t+1}}=\tt{softmax}(\theta_{t+1})$
\ENDFOR
\STATE {\bfseries{Output:}}{ $\theta_{\hat{T}}$ with $\hat{T}$ chosen uniformly at random from $\{1,\ldots, T\}$}
\end{algorithmic}
\end{algorithm}

\setlength{\textfloatsep}{5pt}
\begin{algorithm}[tb]
\caption{{\color{blue}Critic-TD Subroutine}}
\label{alg:Critic-TD}
\begin{algorithmic}
\STATE {\bfseries Input:} {$s_{init}, \pi, \phi, \beta, T_c, H$}
\STATE {\bfseries Initialize:} $w_0$

\FOR{$k\leftarrow 0$ {\bfseries{ to }} $T_c-1$}
\STATE $s_{k,0} = s_{k-1,H}$ (when $k=0$, $s_{k,0}=s_{init}$)
\FOR{$j\gets 0$ {\bfseries{ to }} $H-1$}
\STATE $m_{k,j}\sim \pi(.)$, $a_{k,j}\sim K_{m_{k,j}}(s_{k,j},.)$
\STATE $s_{k,{j+1}}\sim \tilde{P}(.|s_{k,j}, m_{k,k})$
\STATE $\cE_{w_k}(s_{k,j}, m_{k,j}, s_{k,j+1})= \tr(s_{k,j}, m_{k,j}) + \left(\gamma\phi(s_{k,j+1})- \phi(s_{k,j})\right)^\top w_k    $
\ENDFOR
\STATE $w_{k+1} = w_k + \frac{\beta}{H} \sum\limits_{i=0}^{H-1} \cE_{w_k}(s_{k,i}, m_{k,i}, s_{k,i+1})\phi(s_{k,i}) $
\ENDFOR
\STATE {\bfseries {Output:}}{$ w_{T_c}, s_{T_c-1,H} $}
\end{algorithmic}
\end{algorithm}

In order to provide guarantees of the convergence rates of Algorithm ACIL, we make the following assumptions, which are standard in RL literature \citep{Konda_NIPS_AC, Bhandari-TD, Improving-AC-NAC}.
\begin{assumption}[Uniform Ergodicity]\label{assumption:ergodicity}
For any $\theta\in \Real^M$, consider the Markov Chain induced by the policy $\pi_\theta$, and following the transition kernel $\bar{P}(.|s,m)$. Let $\xi_{\pi_\theta}$ be the stationary distribution of this Markov Chain. We assume that there exists constants $\kappa>0$ and $\xi\in (0,1)$ such that 
\vspace{-5pt}
\[\sup_{s\in \cS} \norm{\mathbb{P}\left(s_t\in\cdot|s_0=s, \pi_\theta \right)-\xi_{\pi_\theta}(\cdot) }_{TV} \leq \kappa\xi^t. \]
\end{assumption}
\vspace{-5pt}
Further, let $L_\pi:=\mathbb{E}_{\nu_\pi}[\phi(s)(\gamma\phi(s')-\phi(s))^\top]$ and $v_\pi:=\mathbb{E}_{\nu_\pi}[r(s,m,s')\phi(s)]$. The optimal solution to the critic's TD learning is now $w^*:=-L_\pi^{-1}v_\pi$.
\begin{assumption}\label{assumption:bound on w}
There exists a positive constant $\Gamma_L$ such that for all $w\in \Real^d$, we have $\innprod{w-w^*,L_\pi(w-w^*)}\leq -{\Gamma_L}\norm{w-w^*}_2^2.$
\end{assumption}
\vspace{-5pt}
Based on the above two assumptions, let $L_V:=\frac{2\sqrt{2}C_{\kappa \xi} +1}{1-\gamma}$, where $C_{\kappa \xi}=\left( 1+\left\lceil{\log_\xi\frac{1}{\kappa}}\right\rceil +\frac{1}{1-\xi}  \right)$.
\begin{theorem}\label{thm:AC main theorem}
Consider the Actor-Critic improper learning algorithm ACIL (Alg~\ref{alg:actor-critic improper RL alg}). Assume $\sup_{s\in \cS}\norm{\phi(s)}_2 \leq 1$. Under Assumptions~\ref{assumption:ergodicity} and \ref{assumption:bound on w} with step-sizes chosen as $\alpha=\left(\frac{1}{4L_V\sqrt{M}}\right)$, $\beta=\min\left\{\cO\left(\Gamma_L\right), \cO\left(1/\Gamma_L\right)\right\}$, batch-sizes $H=\mathcal{O}\left(\frac{1}{\epsilon}  \right)$,  $B=\cO\left(1/\epsilon \right)$, $T_c=\mathcal{O}\left(\frac{\sqrt{M}}{\Gamma_L}\log(1/\epsilon)\right)$, $T= \cO\left( \frac{\sqrt{M}}{(1-\gamma)^2\epsilon} \right)$, we have 
$\mathbb{E} [\norm{\nabla_\theta V(\theta_{\hat{T}})}^2_2] \leq \epsilon +  \cO(\Delta_{critic}).$
Hence, the total sample complexity is $\cO\left(M(1-\gamma)^{-2}\epsilon^{-2}\log(1/\epsilon) \right)$.
\end{theorem}
\vspace{-5pt}
Here, $\Delta_{critic}:=\max_{\theta\in \Real^M}\mathbb{E}_{\nu_{\pi_\theta}}\left[\abs{V^{\pi_\theta}(s)-V^{w^*_{\pi_\theta}}}^2\right]$, which equals zero, if the value function lies in the linear space spanned by the features.

Next we provide the {\color{blue}global optimality} guarantee for the Natural-Actor-Critic version of ACIL.
\begin{theorem}\label{thm:NAC main theorem}
Assume $\sup_{s\in \cS}\norm{\phi(s)}_2 \leq 1$. Under Assumptions~\ref{assumption:ergodicity} and \ref{assumption:bound on w} with step-sizes chosen as $\alpha=\left(\frac{\lambda^2}{2\sqrt{M}L_V(1+\lambda)}\right)$, $\beta=\min\left\{\cO\left(\Gamma_L\right), \cO\left(1/\Gamma_L\right)\right\}$, batch-sizes $H=\mathcal{O}\left(\frac{1}{\Gamma_L\epsilon^2}  \right)$,  $B=\cO\left(\frac{1}{(1-\gamma)^2\epsilon^2}  \right)$, $T_c=\mathcal{O}\left(\frac{\sqrt{M}}{\Gamma_L}\log(1/\epsilon)\right)$, $T= \cO\left( \frac{\sqrt{M}}{(1-\gamma)^2\epsilon} \right)$ and $\lambda=\cO(\Delta_{critic})$ we have 
$V({\pi^*})-\frac{1}{T}\sum\limits_{t=0}^{T-1}\expect{V(\pi_{\theta_t})} \leq \epsilon + \cO\left(\sqrt{\frac{\Delta_{actor}}{(1-\gamma)^3} }\right)  + \cO(\Delta_{critic}).$
Hence, the total sample complexity is $\cO\left(\frac{M}{(1-\gamma)^4\epsilon^3} \log\frac{1}{\epsilon}\right)$.
\end{theorem}
\vspace{-10pt}
 where $\Delta_{actor}:=\max_{\theta\in \Real^M}\min_{w\in \Real^d}\mathbb{E}_{\nu_{\pi_\theta}}[[\psi_\theta^\top w - A_{\pi_\theta}(s,m)]^2]$ and $\Delta_{critic}$ is same as before.
\vspace{-5pt}
\section{Numerical  Results}\label{sec:simulation-results}
\subsection{Simulations with Softmax PG}\label{subsec:softmaxPG simulations}
We now discuss the results of implementing Softmax PG (Alg~\ref{alg:mainPolicyGradMDP}) on the cartpole system and on the constrained queueing examples described in Sec.~\ref{sec:motivating-examples}. 
Since neither value functions nor value gradients for these problems are available in closed-form, we modify SoftMax PG (Algorithm~\ref{alg:mainPolicyGradMDP}) to make it generally implementable using a combination of (1) \emph{rollouts} to estimate the value function of the current (improper) policy and (2) \emph{simultaneous perturbation stochastic approximation} (SPSA) to estimate its value gradient. 
Specifically, we use the approach in \cite{Flaxman05}, noting that for a function $V:\Real^M\to\Real$, the gradient, $\nabla V$, $\nabla V(\theta) \approx \expect{\left(V(\theta+\alpha. u)-V(\theta) \right)u}.\frac{M}{\alpha},$ 
where the perturbation parameter $\alpha \in (0,1)$ and $u$ is sampled uniformly randomly from the unit sphere.
%
%
This expression requires evaluation of the value function at the point $(\theta+\alpha.u)$. Since the value function may not be explicitly computable, we employ rollouts, for its evaluation. The full algorithm, \emph{\color{blue}GradEst}, can be  found in the appendix (Alg.~\ref{alg:gradEst}).  



\begin{figure*}[t]
\centering     
\subfigure[Cartpole with $\{K_1=K_{opt}, K_2=K_{opt}+\Delta\}$. ]{\label{subfig:cartpole--asym}\includegraphics[ height=3.5cm, width=4.2cm]{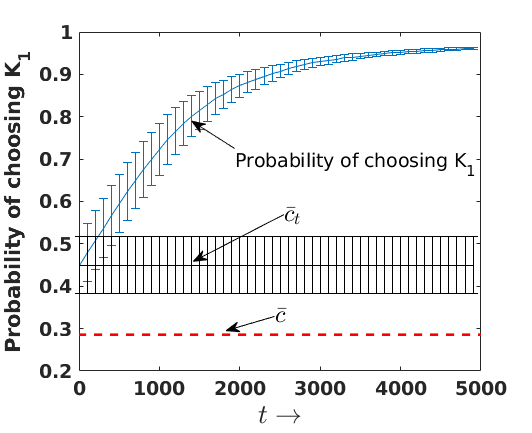}}
\hfill
\subfigure[Cartpole with $\{K_1=K_{opt}-\Delta, K_2=K_{opt}+\Delta\}$. ]{\label{subfig:cartpole--symm}\includegraphics[height=3.5cm, width=4.2cm]{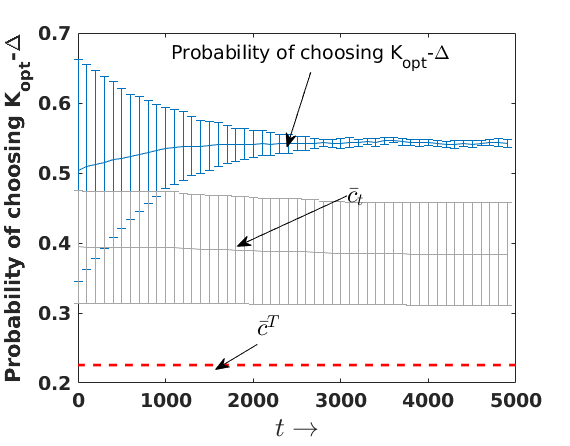}}
\hfill
\subfigure[Softmax PG applied to a Path Graph Network.]{\label{subfig:BQ5}\includegraphics[height=3.5cm, width=4.2cm]{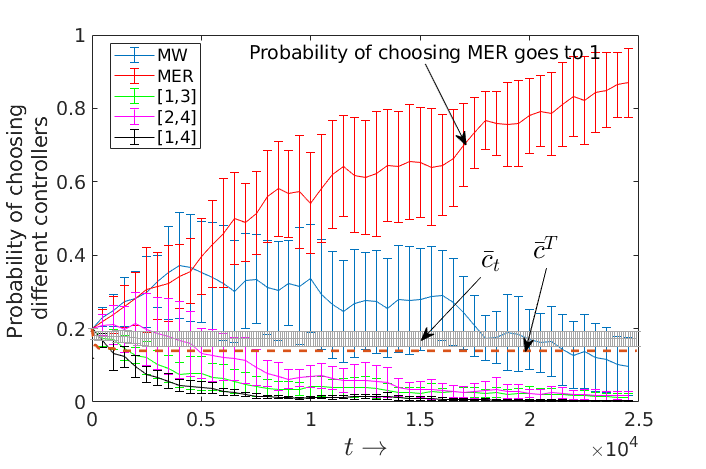}}
\hfill
\subfigure[2 queue system with time-varying arrival rates]{\label{subfig:nonstatbernoulliqueues}\includegraphics[height=3.5cm, width=4.2cm]{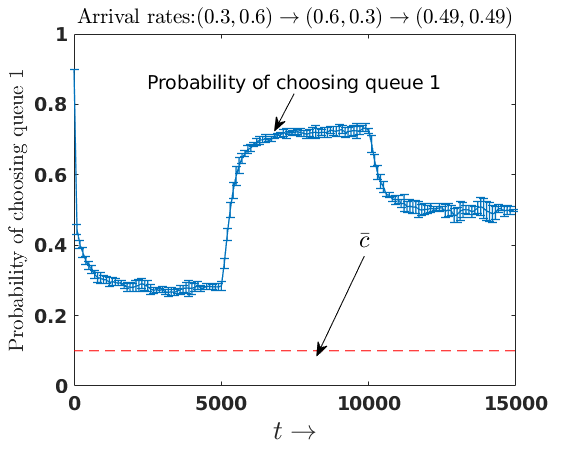}}
\caption{Softmax PG algorithm applied to the cartpole control and path graph scheduling tasks. Each plot shows (a) the  learnt probabilities of various base controllers over time, and (b) the minimum probability $\bar{c_t}$ and $\bar{c}^T$ as described in text.}
\vspace{-10pt}
\label{fig:simulations2}
\end{figure*}

\begin{figure*}[t]
 \centering
\subfigure[Arrival rate:$(\lambda_1,\lambda_2)=(0.4,0.4)$]{\label{subfig:mixture optimal queue NAC}\includegraphics[height=3.5cm, width=5.3cm]{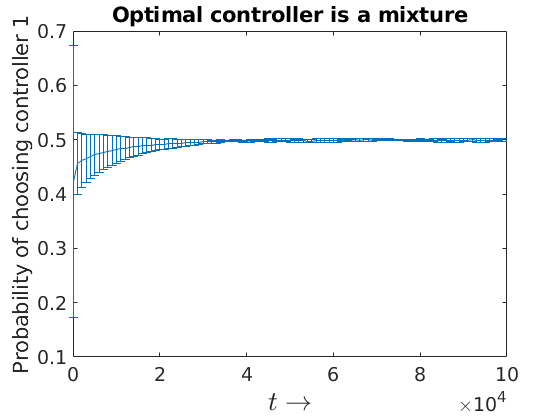}}
\quad
\subfigure[Arrival rate:$(\lambda_1,\lambda_2)=(0.35,0.35)$]{\label{subfig:corner point optimal NAC}\includegraphics[height=3.5cm, width=5.3cm]{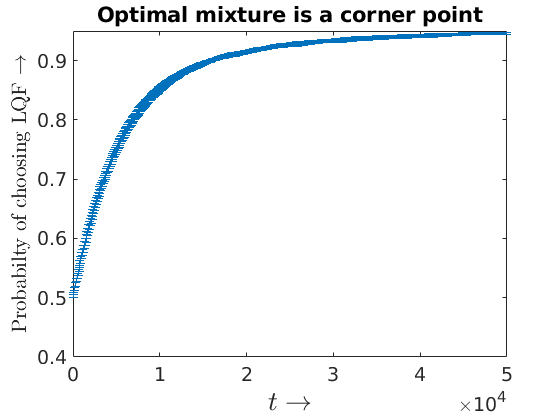}}
\quad
\subfigure[(Estimated) 2 queue system with time-varying arrival rates ]{\label{subfig:transition NAC}       
\includegraphics[height=3.5cm, width=5.3cm]{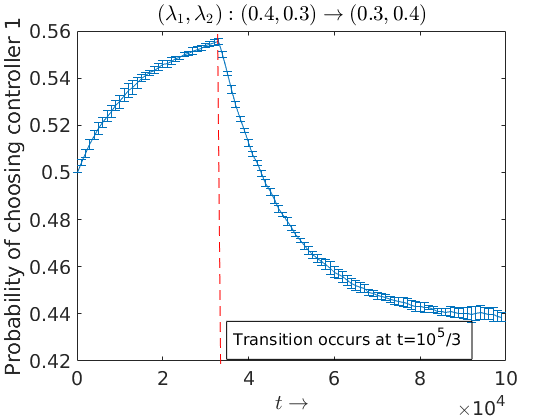}}
\vspace{-0.5cm}
\caption{Natural-actor-critic based improper learning algorithm applied to various queuing networks show convergence to the best mixture policy.\vspace{-20pt}}
\label{fig:simulations}
\end{figure*}
\vspace{-5pt}

\setlength{\textfloatsep}{10pt}
 Note that all of the simulations shown have been averaged over $\#\text{\tt L}=20$ trials, and the mean and standard deviations plotted. We also show empirically that $c_t$ in Theorem \ref{thm:convergence of policy gradient} is indeed strictly positive. In the sequel, for every trial $l\in[\#\text{\tt L}]$, let $\bar{c}_t^l:=\inf\limits_{1\leq s\leq t} \min\limits_{m\in \{m'\in [M]: \pi^*(m')>0\}}\pi_{\theta_s}(m),$ and  $\bar{c}_t:= \frac{1}{\#\text{\tt L}} \sum_{l=1}^{\#\text{\tt L}}\bar{c}_t^l.$ Also let $\bar{c}^T := \min\limits_{l\in[\#\text{\tt L}]}\min\limits_{1\leq t\leq T}\bar{c}_t^l$. That is the sequences $\{\bar{c}_t^l\}_{t=1, l=1}^{T,\#\text{\tt L}}$ define the minimum probabilities that the algorithm puts, over rounds $1:t$ in trial $l$, on controllers with $\pi^*(\cdot)>0$. $\{\bar{c}_t\}_{t=1}^T$ represents its average across the different trials, and $\bar{c}^T$ is the minimum such probability that the algorithm learns across all rounds $1\leq t\leq T$ and across trials.

\noindent{\bf Simulations for the Cartpole.} We study two different settings for the Cartpole example. Let $K_{opt}$ be the optimal controller for the given system, computed via standard procedures (details can be found in \cite{bertsekas11dynamic}). We set $M=2$ and consider two scenarios: (i) the two base controllers are $\cC\equiv\{K_{opt}, K_{opt}+\Delta\}$, where $\Delta$ is a random matrix, each entry of which is drawn IID $\cN(0,0.1)$, (ii) $\cC\equiv\{K_{opt}-\Delta, K_{opt}+\Delta\}$. In the first case a corner point of the simplex is optimal. In the second case a strict improper mixture of the available controllers is optimum. As we can see in Fig. 2(a) and 2(b) our policy gradient algorithm converges to the best controller/mixture in both the cases. The details of all the hyperparameters for this setting are provided in the appendix.
We note here that in the second setting even though none of the controllers, applied individually, stabilizes the system, our Softmax PG algorithm finds and follows a improper mixture of the controllers which stabilizes the given Cartpole. 
\newline{\bfseries{ Constrained Queueing Networks.}} We present simulation results for the following networks.\\
{\bfseries{(i) Path Graph Networks.}} The scheduling constraints in the first network we study dictate that Queues $i$ and $i+1$ cannot be served simultaneously for $i\in[N-1]$ in any round $t\geq0$. Such queueing systems are called \emph{path graph} networks \cite{Mohan2020ThroughputOD}. We work with $N=4$. Therefore, sets of queues which can be served simultaneously are $\cA=\{\emptyset,\{1\}, \{2\}, \{3\}, \{4\}, \{1,3\}, \{2,4\}, \{1,4\} \}$. The constituents of $\cA$ are called \emph{independent sets} in the literature. In each round $t$, the scheduler selects an independent set to serve the queues therein. Let $Q_j(t)$ be the backlog of Queue~$j$ at time $t$. We use the following base controllers: (i) $K_1:$ Max Weight (MW) controller \cite{tassiulas-ephremides92stability-queueing} chooses a set $s_t:=\argmax_{\underbar{s}\in \mathcal{A}}\sum_{j\in \underbar{s}} Q_j(t)$, i.e, the set with the largest backlog, (ii) $K_2:$ Maximum Egress Rate (MER) controller chooses a set $s_t:=\argmax_{\underbar{s}\in \mathcal{A}}\sum_{j\in \underbar{s}} \ind\{Q_j(t)>0\}$, i.e, the set which has the maximum number of non-empty queues.
We also choose $K_3,K_4$ and $K_5$ which serve the sets $\{1,3\}, \{2,4\}, \{1,4\}$ respectively with probability 1. We fix the arrival rates to the queues $(0.495,0.495,0.495,0.495)$. It is well known that the MER rule is mean-delay optimal in this case \cite{Mohan2020ThroughputOD}. In Fig. 2(c), we plot the probability of choosing $K_i, i\in [5]$, learnt by our algorithm. The probability of choosing MER indeed converges to 1.\\
{\bfseries{(ii) Non-stationary arrival rates.}} Recall the example discussed in Sec.~\ref{sec:motivating-queueing-example} of two queues. The scheduler there is now given two base/atomic controllers $\mathcal{C}:=\{K_1,K_2\}$, i.e. $M=2$. Controller $K_i$ serves Queue~$i$ with probability $1$, $i=1,2$. As can be seen in Fig. 2(d), the arrival rates $\boldsymbol{\lambda}$ to the two queues vary over time (adversarially) during the learning. In particular, $\boldsymbol{\lambda}$ varies from $(0.3,0.6)\to (0.6,0.3)\to(0.49,0.49)$. Our PG algorithm successfully \emph{tracks} this change and \emph{adapts} to the optimal improper stationary policies in each case. 
\\In all the simulations shown above we note that the empirical trajectories of $\bar{c}_t$ and $\bar{c}^T$ become flat after some initial rounds and are bounded away from zero. This supports our conjecture that $\lim_{t\to\infty}c_t$ in Theorem \ref{thm:convergence of policy gradient} is bounded away from zero, rendering the theorem statement non-vacuous. 
Note that Alg.~\ref{alg:mainPolicyGradMDP} performs well in challenging scenarios, \emph{even} with estimates of the value function and its gradient. 
\vspace{-10pt}
\subsection{Simulations with ACIL}
\vspace{-5pt}
We perform some queueing theoretic simulations on the natural actor critic version of ACIL, which we will call NACIL in this section.
Unlike Softmax PG, ACIL estimates gradients using temporal difference instead of SPSA. We study three different settings (1) where in the first case the optimal policy is a strict improper combination of the available controllers and (2) where it is at a corner point, i.e., one of the available controllers itself is optimal (3) arrival rates are time-varying as in the previous section. Our simulations show that in all the cases, ACIL converges to the correct controller mixture. 

Recall the example that we discussed in Sec.~\ref{sec:motivating-queueing-example}. We consider the case with Bernoulli arrivals with rates $\boldsymbol{\lambda}=[\lambda_1,\lambda_2]$ and are given two base/atomic controllers $\{K_1,K_2\}$, where controller $K_i$ serves Queue~$i$ with probability $1$, $i=1,2$. As can be seen in Fig.~\ref{subfig:mixture optimal queue NAC} when $\boldsymbol{\lambda}=[0.4,0.4]$ (equal arrival rates), NACIL converges to an improper mixture policy that serves each queue with probability $[0.5,0.5]$.  
Next in Fig~\ref{subfig:corner point optimal NAC} shows a situation where one of the base controllers, i.e., the ``Longest-Queue-First" (LQF) is the optimal controller. NACIL converges correctly to the corner point.

Lastly, Fig.~\ref{subfig:transition NAC} shows a setting similar to (ii) Sec.~\ref{subsec:softmaxPG simulations} above. Here there is a single transition of $(\lambda_1,\lambda_2)$ from $(0.4,0.3)\to(0.3,0.4)$ which occurs at $t=\left\lceil 10^5/3\right\rceil$, which is unknown to the learner. We show the probability of choosing controller 1. NACIL tracks the changing arrival rates over time.
We supply some more simulations with NACIL in the appendix due to space limitations.

\section*{Acknowledgment}

This work was partially supported by the Israel Science Foundation under contract 2199/20. Mohammadi Zaki was supported by the Aerospace Network Research Consortium (ANRC) Grant on Airplane IOT Data Management.
\newpage
\nocite{langley00}

\bibliography{MDPbib}
\bibliographystyle{icml2022}

\newpage
\appendix
\onecolumn
\onecolumn
\appendix
\section{Glossary of Symbols}
\begin{enumerate}
    \item $\cS$: State space
    \item $\cA:$ Action space
    \item $S:$ Cardinality of $\cS$
    \item $A:$ Cardinality of $\cA$
    \item $M:$ Number of controllers
    \item $K_i$ Controller $i$, $i=1,\cdots,M$. For finite SA space MDP, $K_i$ is a matrix of size $S\times A$, where each row is a probability distribution over the actions. 
    \item $\cC:$ Given collection of $M$ controllers.
    \item $\cI_{soft}(\cC):$ Improper policy class setup by the learner.
    \item  $\theta\in \Real^M:$ Parameter assigned to the controllers to controllers, representing weights, updated each round by the learner. 
    \item $\pi(.):$ Probability of choosing controllers
    \item $\pi(.\given s)$ Probability of choosing action given state $s$. Note that in our setting, given $\pi(.)$ over controllers (see previous item) and the set of controllers, $\pi(.\given s)$ is completely defined, i.e., $\pi(a\given s)=\sm \pi(m)K_m(s,a)$. Hence we use simply $\pi$ to denote the policy followed, whenever the context is clear.
    \item $r(s,a):$ Immediate (one-step) reward obtained if action $a$ is played in state $s$.
    \item $\tP(s'\given s,a)$ Probability of transitioning to state $s'$ from state $s$ having taken action $a$.
    \item $V^{\pi}(\rho):=\mathbb{E}_{s_0\sim \rho}\left[V^\pi(s_0) \right]  = \mathbb{E}^{\pi}_{\rho}\sum_{t=0}^{\infty}\gamma^tr(s_t,a_t)$ Value function starting with initial distribution $\rho$ over states, and following policy $\pi$. 
    \item $Q^\pi(s,a):= \expect{r(s,a) + \gamma\sum\limits_{s'\in \cS}\tP(s'\given s,a)V^\pi(s')}$. 
    \item ${\Q}^\pi(s,m):= \expect{\sum\limits_{a\in \cA}K_m(s,a)r(s,a) + \gamma\sum\limits_{s'\in \cS}\tP(s'\given s,a)V^\pi(s')}$.
    \item $A^\pi(s,a):=Q^\pi(s,a)-V^\pi(s)$
    \item $\tilde{A}(s,m):=\Q^\pi(s,m)-V^\pi(s)$.
    \item $d_\nu^\pi:=\mathbb{E}_{s_0\sim \nu}\left[(1-\gamma)\sum\limits_{t=0}^\infty \prob{s_t=s\given s_o,\pi,\tP} \right]$. Denotes a distribution over the states, is called the ``discounted state visitation measure"
    \item $c : \inf_{t\geq 1}\min\limits_{m\in \{m'\in[M]:\pi^*(m') >0\}}\pi_{\theta_t}(m)$.
    \item $\norm{\frac{d_\mu^{\pi^*}}{\mu}}_\infty = \max_s \frac{d_\mu^{\pi^*}(s)}{\mu(s)}$.
    \item $\norm{\frac{1}{\mu}}_\infty = \max_s \frac{1}{\mu(s)}$.
\end{enumerate}

\section{Expanded Survey of Related Work}
\label{sec:relatedworkfull}
In this section, we provide a detailed survey of related works. It is vital to distinguish the approach investigated in the present paper from the plethora of existing algorithms based on 'proper learning'. Essentially, these algorithms try to find an (approximately) optimal policy for the MDP under investigation. 
These approaches can broadly be classified in two groups: \emph{\color{blue}model-based} and \emph{\color{blue}model-free}. 

The former is based on first learning the dynamics of the unknown MDP followed by planning for this learnt model. Algorithms in this class  include Thompson Sampling-based approaches \cite{OsbandTS_NIPS2013,Ouyang2017LearningUM,pmlr-v40-Gopalan15}, Optimism-based approaches such as the UCRL algorithm \cite{UCRL}, both achieving order-wise optimal $\cO(\sqrt{T})$ regret bound.

A particular class of MDPs which has been studied extensively is the Linear Quadratic Regulator (LQR) which is a continuous state-action MDP with linear state dynamics and quadratic cost \cite{SarahDean}. Let $x_t\in \Real^m$ be the current state and let $u_t\in \Real^n$ be the action applied at time $t$. The infinite horizon average cost minimization problem for LQR is to find a policy to choose actions $\{u_t\}_{t\geq 1}$ so as to minimize
\[\lim\limits_{T\to \infty}\expect{\frac{1}{T}\sum\limits_{t=1}^T x_t\transpose Qx_t +u_t\transpose Ru_t}\]
such that $x_{t+1}=Ax_t+Bu_t+n(t)$, $n(t)$ is iid zero-mean noise. Here the matrices $A$ and $B$ are unknown to the learner.
Earlier works like \cite{abbasi-yadkori, Ibrahimi}  proposed algorithms based on the well-known optimism principle (with confidence ellipsoids around estimates of $A$ and $B$). These show regret bounds of $\cO(\sqrt{T})$. 

However, these approaches do not focus on the stability of the closed-loop system. \cite{SarahDean} describes a robust controller design which seeks  to  minimize  the  worst-case  performance  of  the  system  given  the error in the estimation process. They show a sample complexity analysis guaranteeing convergence rate of $\cO(1/\sqrt{N})$ to the optimal policy for the given LQR, $N$ being the number of rollouts. More recently, certainity equivalence \cite{mania2019certainty} was shown to achieve $\cO(\sqrt{T})$ regret for LQRs. Further, \cite{pmlr-v119-cassel20a} show  that it is possible to achieve $\cO(\log T)$ regret if either one of the matrices $A$ or $B$ are known to the learner, and also provided a lower bound showing that $\Omega(\sqrt{T})$ regret is unavoidable when both are unknown.

The \textit{model-free} approach on the other hand, bypasses model estimation and directly learns the value function of the unknown MDP. While the most popular among these have historically been Q-learning, TD-learning \cite{Sutton1998} and SARSA \cite{Rummery94on-lineq-learning-SARSA}, algorithms based on gradient-based policy optimization have been gaining considerable attention of late, following their stunning success with playing the game of Go which has long been viewed as the most challenging of classic games for artificial intelligence owing to its enormous search space and the difficulty of evaluating board positions and moves. \cite{AlphaGo} and more recently \cite{AlphaGozero} use policy gradient method combined with a neural network representation to beat human experts. Indeed, the Policy Gradient method has become the cornerstone of modern RL and given birth to an entire class of highly efficient policy search algorithms such as TRPO \cite{TRPO}, PPO\cite{PPO}, and MADDPG \cite{MADDPG}. 

Despite its excellent empirical performance, 
not much was known about theoretical guarantees for this approach until recently. There is now a growing body of promising results showing convergence rates for PG algorithms over finite state-action MDPs \cite{Agarwal2020,Shani2020AdaptiveTR,Bhandari2019GlobalOG,Mei2020}, where the parameterization is over the entire space of state -action pairs, i.e., $\Real^{S\times A}$. In particular, \cite{Bhandari2019GlobalOG} show that projected gradient descent does not suffer from spurious local optima on the simplex, \cite{Agarwal2020} show that the with softmax parameterization PG converges to the global optima asymptotically. \cite{Shani2020AdaptiveTR} show a $\cO(1/\sqrt{t})$ convergence rate for mirror descent. \cite{Mei2020} show that with softmax policy gradient convergence to the global optima occurs at a rate $\cO(1/t)$ and at $\cO(e^{-t})$ with entropy regularization.



 We end this section noting once again that all of the above works concern \emph{proper} learning.
Improper learning, on the other hand, has been separately studied in statistical learning theory in the IID setting \cite{Improper1, Improper2}.  In this framework, which is also called \emph{Representation Independent} learning, the learning algorithm is  not restricted to output a hypothesis from a given set of hypotheses. 
 We note that improper learning has not been studied in RL literature to the best of our knowledge.
 
To our knowledge, \cite{hazan-etal20boosting-control-dynamical} is the only existing work that attempts to frame and solve policy optimization over an improper class via boosting a given class of controllers. However, the paper is situated in the rather different context of non-stochastic control and assumes perfect knowledge of (i) the memory-boundedness of the MDP, and (ii) the state noise vector in every round, which amounts to essentially knowing the MDP transition dynamics. We work in the stochastic MDP setting and moreover assume no access to the MDP's transition kernel. Further, \cite{hazan-etal20boosting-control-dynamical} also assumes that all the atomic controllers available to them are \emph{stabilizing} which, when working with an unknown MDP, is a very strong assumption to make. We make no such assumptions on our atomic controller class and, as we show in Sec.~\ref{sec:motivating-examples} and Sec.~\ref{sec:simulation-results}, our algorithms even begin with provably unstable controllers and yet succeed in stabilizing the system.

In summary, the problem that we address concerns finding the best among a \emph{given} class of controllers. None of these need be optimal for the MDP at hand. Moreover, our PG algorithm could very well converge to an improper mixture of these controllers meaning that the output of our algorithms need not be any of the atomic controllers we are provided with. This setting, to the best of our knowledge has not been investigated in the RL literature hitherto.
\section{Details of Setup and Modelling of the Cartpole}\label{appendix:Details of Setup and Modelling of the Inverted Pendulum}

    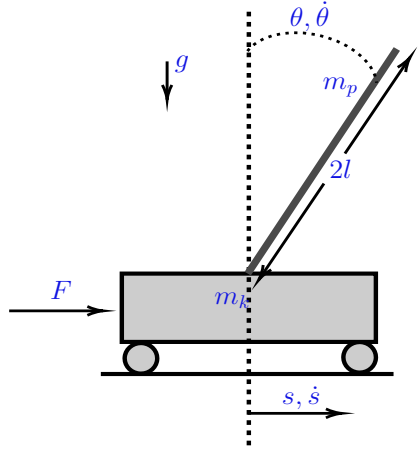
\begin{figure}[H]
\centering
\tikzset{every picture/.style={line width=0.75pt}} 

\resizebox{5.50cm}{6.0cm}{
\begin{tikzpicture}[x=0.75pt,y=0.75pt,yscale=-1,xscale=1, scale=0.75]

\draw  [fill={rgb, 255:red, 155; green, 155; blue, 155 }  ,fill opacity=0.5 ][line width=2.25]  (209,268) -- (439,268) -- (439,334) -- (209,334) -- cycle ;
\draw  [fill={rgb, 255:red, 155; green, 155; blue, 155 }  ,fill opacity=0.5 ][line width=2.25]  (409.68,347.3) .. controls (410.55,339.22) and (417.8,333.37) .. (425.89,334.24) .. controls (433.97,335.11) and (439.82,342.36) .. (438.95,350.44) .. controls (438.09,358.53) and (430.83,364.38) .. (422.75,363.51) .. controls (414.67,362.64) and (408.82,355.39) .. (409.68,347.3) -- cycle ;
\draw  [fill={rgb, 255:red, 155; green, 155; blue, 155 }  ,fill opacity=0.5 ][line width=2.25]  (211.68,348.3) .. controls (212.55,340.22) and (219.8,334.37) .. (227.89,335.24) .. controls (235.97,336.11) and (241.82,343.36) .. (240.95,351.44) .. controls (240.09,359.53) and (232.83,365.38) .. (224.75,364.51) .. controls (216.67,363.64) and (210.82,356.39) .. (211.68,348.3) -- cycle ;
\draw [line width=2.25]    (190,365) -- (455,365) ;
\draw [line width=2.25]  [dash pattern={on 2.53pt off 3.02pt}]  (324,15) -- (324,436) ;
\draw [color={rgb, 255:red, 74; green, 74; blue, 74 }  ,draw opacity=1 ][fill={rgb, 255:red, 208; green, 2; blue, 27 }  ,fill opacity=1 ][line width=3.75]    (323,268) -- (458,51) ;
\draw  [draw opacity=0][dash pattern={on 1.69pt off 2.76pt}][line width=1.5]  (322.77,53.04) .. controls (333.78,37.52) and (360.79,30.92) .. (389.01,38.26) .. controls (417.32,45.63) and (437.72,64.66) .. (439.66,83.65) -- (378.7,77.86) -- cycle ; \draw  [dash pattern={on 1.69pt off 2.76pt}][line width=1.5]  (322.77,53.04) .. controls (333.78,37.52) and (360.79,30.92) .. (389.01,38.26) .. controls (417.32,45.63) and (437.72,64.66) .. (439.66,83.65) ;
\draw [line width=1.5]    (250,63) -- (250,97) ;
\draw [shift={(250,100)}, rotate = 270] [color={rgb, 255:red, 0; green, 0; blue, 0 }  ][line width=1.5]    (14.21,-4.28) .. controls (9.04,-1.82) and (4.3,-0.39) .. (0,0) .. controls (4.3,0.39) and (9.04,1.82) .. (14.21,4.28)   ;
\draw [line width=1.5]    (107,304) -- (189,304) ;
\draw [shift={(192,304)}, rotate = 180] [color={rgb, 255:red, 0; green, 0; blue, 0 }  ][line width=1.5]    (14.21,-4.28) .. controls (9.04,-1.82) and (4.3,-0.39) .. (0,0) .. controls (4.3,0.39) and (9.04,1.82) .. (14.21,4.28)   ;
\draw [line width=1.5]    (323,403) -- (405,403) ;
\draw [shift={(408,403)}, rotate = 180] [color={rgb, 255:red, 0; green, 0; blue, 0 }  ][line width=1.5]    (14.21,-4.28) .. controls (9.04,-1.82) and (4.3,-0.39) .. (0,0) .. controls (4.3,0.39) and (9.04,1.82) .. (14.21,4.28)   ;
\draw [line width=1.5]    (392,181) -- (334.61,271.47) ;
\draw [shift={(333,274)}, rotate = 302.39] [color={rgb, 255:red, 0; green, 0; blue, 0 }  ][line width=1.5]    (14.21,-4.28) .. controls (9.04,-1.82) and (4.3,-0.39) .. (0,0) .. controls (4.3,0.39) and (9.04,1.82) .. (14.21,4.28)   ;
\draw [line width=1.5]    (468.39,61.53) -- (411,152) ;
\draw [shift={(470,59)}, rotate = 122.39] [color={rgb, 255:red, 0; green, 0; blue, 0 }  ][line width=1.5]    (14.21,-4.28) .. controls (9.04,-1.82) and (4.3,-0.39) .. (0,0) .. controls (4.3,0.39) and (9.04,1.82) .. (14.21,4.28)   ;

\draw (359,3) node [anchor=north west][inner sep=0.75pt]  [font=\Large,color={rgb, 255:red, 16; green, 19; blue, 226 }  ,opacity=1 ] [align=left] {$\displaystyle \theta ,\dot{\theta }$};
\draw (255,55) node [anchor=north west][inner sep=0.75pt]  [font=\Large,color={rgb, 255:red, 16; green, 19; blue, 226 }  ,opacity=1 ] [align=left] {$\displaystyle g$};
\draw (395,156) node [anchor=north west][inner sep=0.75pt]  [font=\Large,color={rgb, 255:red, 16; green, 19; blue, 226 }  ,opacity=1 ] [align=left] {$\displaystyle 2l$};
\draw (389,77) node [anchor=north west][inner sep=0.75pt]  [font=\Large,color={rgb, 255:red, 16; green, 19; blue, 226 }  ,opacity=1 ] [align=left] {$\displaystyle m_{p}$};
\draw (142,273) node [anchor=north west][inner sep=0.75pt]  [font=\Large,color={rgb, 255:red, 16; green, 19; blue, 226 }  ,opacity=1 ] [align=left] {$\displaystyle F$};
\draw (352,375) node [anchor=north west][inner sep=0.75pt]  [font=\Large,color={rgb, 255:red, 16; green, 19; blue, 226 }  ,opacity=1 ] [align=left] {$\displaystyle s,\dot{s}$};
\draw (291,284) node [anchor=north west][inner sep=0.75pt]  [font=\Large,color={rgb, 255:red, 16; green, 19; blue, 226 }  ,opacity=1 ] [align=left] {$\displaystyle m_{k}$};

\end{tikzpicture}
}
\caption{The Cartpole system. The mass of the pendulum is denoted by $m_p$, that of the cart by $m_K,$ the force used to drive the cart by $F$, and the distance of the center of mass of the cart from its starting position by $s.$ $\theta$ denotes the angle the pendulum makes with the normal and its length is denoted by $2l.$ Gravity is denoted by $g.$ }
\label{fig:inverted-pendulum}
\end{figure}

 As shown in Fig.~\ref{fig:inverted-pendulum}, it comprises a pendulum whose pivot is mounted on a cart which can be moved in the horizontal direction by applying a force. The objective is to modulate the direction and magnitude of this force $F$ to keep the pendulum from keeling over under the influence of gravity. The state of the system at time $t,$ is given by the 4-tuple $\mathbf{x}(t):=[s,\Dot{s},\theta,\Dot{\theta}]$, with $\mathbf{x}(\cdot)=\mathbf{0}$ corresponding to the pendulum being upright and stationary. One of the strategies used to design control policies for this system is by first approximating the dynamics around $\mathbf{x}(\cdot)=\mathbf{0}$ with a linear, quadratic cost model and designing a linear controller for these approximate dynamics. This, after time discretization, 
The objective now reduces to finding a (potentially randomized) control policy $u\equiv \{u(t),t\geq0\}$ that solves: 

\begin{eqnarray}\label{eqn:cartpole-LQR-approx-WITHOUT-Details}
\inf_{u} J(\mathbf{x}(0)) &=& \mathbb{E}_{u}\sum_{t=0}^\infty \mathbf{x}^\intercal(t) Q \mathbf{x}(t) + R u^2(t), \nonumber \\
s.t.~{\mathbf{x}}(t+1) &=& \underbrace{\begin{pmatrix}
0 & 1 & 0                                                       & 0\\
0 & 0 & \frac{g}{l\left(\frac{4}{3}-\frac{m_p}{m_p+m_k}\right)} & 0 \\
0 & 0 &  0                                                      & 1 \\
0 & 0 & \frac{g}{l\left(\frac{4}{3}-\frac{m_p}{m_p+m_k}\right)} & 0 \\
\end{pmatrix}}_{\color{blue}A_{open}} \mathbf{x}(t)
+ \underbrace{\begin{pmatrix}
0 \\
\frac{1}{m_p+m_k}\\
0 \\
\frac{1}{l\left(\frac{4}{3}-\frac{m_p}{m_p+m_k}\right)}  \\
\end{pmatrix}}_{\color{blue}\mathbf{b}} u(t).
\end{eqnarray}
\looseness=-1 Under standard assumptions of controllability and observability, this optimization has a stationary, linear solution $u^*(t) = -\mathbf{K}^\intercal\mathbf{x}(t)$ (details are available in \cite[Chap.~3]{bertsekas11dynamic}). Moreover, setting $A:=A_{open}-\mathbf{b}\mathbf{K}^\intercal$, it is well know that the dynamics ${\mathbf{x}}(t+1) = A \mathbf{x}(t),~t\geq0,$
are stable. 
\subsection{Details of simulations settings for the cartpole system}
In this section we supply the adjustments we made for specifically for the cartpole experiments. We first mention that we scale down the estimated gradient of the value function returned by the GradEst subroutine (Algorithm ~\ref{alg:gradEst}) (in the cartpole simulation only). The scaling that worked for us is $\frac{10}{\norm{\widehat{{\nabla}V^{\pi}(\mu)}}}$.

Next, we provide the values of the constants that were described in Sec. ~\ref{appendix:Details of Setup and Modelling of the Inverted Pendulum} in Table ~\ref{table:hyperparameters of inverted pend}.
\begin{table}[H]
\centering
\begin{tabular}{c c}
\hline\hline
Parameter & Value \\
\hline
 Gravity $g$  &    9.8\\
 Mass of pole $m_p$  &  0.1\\
 Length of pole $l$  & 1\\
 Mass of cart $m_k$  &  1\\
 Total mass $m_t$  & 1.1\\
 \hline
\end{tabular}
\caption{Values of the hyperparameters used for the cartpole simulation
}
\label{table:hyperparameters of inverted pend}
\end{table}
\section{Stability for Ergodic Parameter Linear Systems (EPLS)}
For simplicity and ease of understanding, we connect our current discussion to the cartpole example discussed in Sec. ~\ref{sec:motivating-cartpole-example}. Consider a generic (ergodic) control policy that switches across a menu of controllers $\{K_1,\cdots,K_N\}$. That is, at any time $t$, it chooses controller $K_i,~i\in[N],$ w.p. $p_i$, so that the control input at time $t$ is $u(t)=-\mathbf{K}_i^\intercal\mathbf{x}(t)$ w.p. $p_i.$ Let $A(i) := A_{open}-\mathbf{b}\mathbf{K}_i^\intercal$. The resulting controlled dynamics are given by
\begin{eqnarray}\label{eqn:EPLS-definition}
    {\mathbf{x}}(t+1) &=& A(r(t)) \mathbf{x}(t)\nonumber\\
    \mathbf{x}(0) &=& \mathbf{0},
\end{eqnarray}
where $r(t)=i$ w.p. $p_i,$ IID across time. In the literature, this belongs to a class of systems known as \emph{Ergodic Parameter Linear Systems} (EPLS) \cite{bolzern-etal08almost-sure-stability-ergodic-linear}, which are said to be \emph{Exponentially Almost Surely Stable} (EAS) if there exists $\rho>0$ such that for any $\mathbf{x}(0),$
\begin{eqnarray}\label{eqn:exp-almost-sure-definition}
    \mathbb{P}\left\lbrace\omega\in\Omega\bigg|\limsup_{t\rightarrow\infty}\frac{1}{t}\log{\norm{\mathbf{x}(t,\omega)}}\leq-\rho\right\rbrace = 1.
\end{eqnarray}
In other words, w.p.~1, the trajectories of the system decay to the origin exponentially fast. The random variable $\lambda(\omega):=\limsup_{t\rightarrow\infty}\frac{1}{t}\log{\norm{\mathbf{x}(t,\omega)}}$ in \eqref{eqn:exp-almost-sure-definition} is called the \emph{Lyapunov Exponent} of the system. For our EPLS,
\begin{eqnarray}
    \lambda(\omega) &=& \limsup_{t\rightarrow\infty}\frac{1}{t}\log{\norm{\mathbf{x}(t,\omega)}}= \limsup_{t\rightarrow\infty}\frac{1}{t}\log{\norm{\prod_{s=1}^tA(r(s,\omega))\mathbf{x}(0)}}\nonumber\\
                    &\leq& \limsup_{t\rightarrow\infty}\cancelto{0}{\frac{1}{t}\log{\norm{\mathbf{x}(0)}}} + \limsup_{t\rightarrow\infty}\frac{1}{t}\log{\norm{\prod_{s=1}^tA(r(s,\omega))}} \nonumber\\
                    &\leq& \limsup_{t\rightarrow\infty}\frac{1}{t}\sum_{s=1}^t\log{\norm{A(r(s,\omega))}} \stackrel{(\ast)}{=} \lim_{t\rightarrow\infty}\frac{1}{t}\sum_{s=1}^t\log{\norm{A(r(s,\omega))}}\nonumber\\
                    &\stackrel{(\dagger)}{=}& \mathbb{E}\log{\norm{A(r)}}= \sum_{i=1}^N p_i \log{\norm{A(i)}},\label{eqn:lyapExponentLQR}
\end{eqnarray}
where the equalities $(\ast)$ and $(\dagger)$ are due to the ergodic law of large numbers. The control policy can now be designed by choosing $\{p_1,\cdots,p_N\}$ such that $\lambda(\omega)<-\rho$ for some $\rho>0$, ensuring exponentially almost sure stability.

\section{The Constrained Queuing Example}

The system, shown in Fig.~\ref{subfig:the-two-queues}, comprises two queues fed by independent, stochastic arrival processes $A_i(t),i\in\{1,2\},t\in\mathbb{N}.$ The length of Queue~$i$, measured at the beginning of time slot $t,$ is denoted by $Q_i(t)\in\mathbb{Z}_+$. A common server serves both queues and can drain at most one packet from the system in a time slot\footnote{Hence, a \emph{constrained} queueing system.}. The server, therefore, needs to decide which of the two queues it intends to serve in a given slot (we assume that once the server chooses to serve a packet, service succeeds with probability 1). The server's decision is denoted by the vector $\mathbf{D}(t)\in\mathcal{A}:=\left\lbrace[0,0],[1,0],[0,1]\right\rbrace,$ where a \enquote{$1$} denotes service and a \enquote{$0$} denotes lack thereof.

\begin{figure*}[tbh]
\centering
        \tikzset{every picture/.style={line width=0.75pt}} 

\resizebox{5.50cm}{3.5cm}{

\begin{tikzpicture}[x=0.5pt,y=0.5pt,yscale=-1,xscale=1]

\draw [line width=1.5]    (91.11,161.23) -- (292.11,161.23) ;
\draw [line width=1.5]    (89.99,227.37) -- (292.11,228.47) ;
\draw    (68.66,194.3) -- (100.34,194.3) ;
\draw [shift={(102.34,194.3)}, rotate = 180] [color={rgb, 255:red, 0; green, 0; blue, 0 }  ][line width=0.75]    (10.93,-3.29) .. controls (6.95,-1.4) and (3.31,-0.3) .. (0,0) .. controls (3.31,0.3) and (6.95,1.4) .. (10.93,3.29)   ;
\draw [line width=1.5]    (292.11,161.23) -- (292.11,228.47) ;
\draw    (201.01,39.58) -- (289.11,39.58) ;
\draw [shift={(292.11,39.58)}, rotate = 180] [fill={rgb, 255:red, 0; green, 0; blue, 0 }  ][line width=0.08]  [draw opacity=0] (8.93,-4.29) -- (0,0) -- (8.93,4.29) -- cycle    ;
\draw [shift={(198.01,39.58)}, rotate = 0] [fill={rgb, 255:red, 0; green, 0; blue, 0 }  ][line width=0.08]  [draw opacity=0] (8.93,-4.29) -- (0,0) -- (8.93,4.29) -- cycle    ;
\draw    (136.78,237.99) -- (289.11,237.99) ;
\draw [shift={(292.11,237.99)}, rotate = 180] [fill={rgb, 255:red, 0; green, 0; blue, 0 }  ][line width=0.08]  [draw opacity=0] (8.93,-4.29) -- (0,0) -- (8.93,4.29) -- cycle    ;
\draw [shift={(133.78,237.99)}, rotate = 0] [fill={rgb, 255:red, 0; green, 0; blue, 0 }  ][line width=0.08]  [draw opacity=0] (8.93,-4.29) -- (0,0) -- (8.93,4.29) -- cycle    ;
\draw  [fill={rgb, 255:red, 155; green, 155; blue, 155 }  ,fill opacity=0.67 ] (409.11,146.7) .. controls (409.11,137.68) and (416.56,130.37) .. (425.75,130.37) .. controls (434.94,130.37) and (442.39,137.68) .. (442.39,146.7) .. controls (442.39,155.72) and (434.94,163.04) .. (425.75,163.04) .. controls (416.56,163.04) and (409.11,155.72) .. (409.11,146.7) -- cycle ;
\draw    (385.31,75.25) .. controls (313.05,125.45) and (338.74,181.05) .. (383.93,215.31) ;
\draw [shift={(385.31,216.35)}, rotate = 216.57999999999998] [color={rgb, 255:red, 0; green, 0; blue, 0 }  ][line width=0.75]    (10.93,-3.29) .. controls (6.95,-1.4) and (3.31,-0.3) .. (0,0) .. controls (3.31,0.3) and (6.95,1.4) .. (10.93,3.29)   ;
\draw    (409.11,146.7) -- (322.96,88.5) ;
\draw [shift={(321.3,87.38)}, rotate = 394.03999999999996] [color={rgb, 255:red, 0; green, 0; blue, 0 }  ][line width=0.75]    (10.93,-3.29) .. controls (6.95,-1.4) and (3.31,-0.3) .. (0,0) .. controls (3.31,0.3) and (6.95,1.4) .. (10.93,3.29)   ;
\draw  [color={rgb, 255:red, 95; green, 19; blue, 254 }  ,draw opacity=0.23 ][fill={rgb, 255:red, 95; green, 19; blue, 254 }  ,fill opacity=0.23 ] (139.8,161.01) -- (292.11,161.01) -- (292.11,228.47) -- (139.8,228.47) -- cycle ;
\draw [line width=1.5]    (91.11,51) -- (292.11,51) ;
\draw [line width=1.5]    (89.99,117.14) -- (292.11,118.24) ;
\draw    (68.66,84.07) -- (100.34,84.07) ;
\draw [shift={(102.34,84.07)}, rotate = 180] [color={rgb, 255:red, 0; green, 0; blue, 0 }  ][line width=0.75]    (10.93,-3.29) .. controls (6.95,-1.4) and (3.31,-0.3) .. (0,0) .. controls (3.31,0.3) and (6.95,1.4) .. (10.93,3.29)   ;
\draw [line width=1.5]    (292.11,51) -- (292.11,118.24) ;
\draw  [color={rgb, 255:red, 95; green, 19; blue, 254 }  ,draw opacity=0.23 ][fill={rgb, 255:red, 95; green, 19; blue, 254 }  ,fill opacity=0.23 ] (209.8,50.78) -- (292.11,50.78) -- (292.11,118.24) -- (209.8,118.24) -- cycle ;

\draw (10.64,184.95) node [anchor=north west][inner sep=0.75pt]   [align=left] {$\displaystyle A_{2}( t)$};
\draw (9.52,74.72) node [anchor=north west][inner sep=0.75pt]   [align=left] {$\displaystyle A_{1}( t)$};
\draw (233.03,4.18) node [anchor=north west][inner sep=0.75pt]   [align=left] {$\displaystyle Q_{1}( t)$};
\draw (195.98,249.99) node [anchor=north west][inner sep=0.75pt]   [align=left] {$\displaystyle Q_{2}( t)$};
\draw (406.65,101.42) node [anchor=north west][inner sep=0.75pt]    {$\mathbf{D}( t)$};

\end{tikzpicture}
}
        \caption{$Q_i(t)$ is the length of Queue~$i$ ($i\in\{1,2\}$) at the beginning of time slot $t$, $A_i(t)$ is its packet arrival process and $\mathbf{D}(t)\in\left\lbrace[0,0],[1,0],[0,1]\right\rbrace.$}
        \label{subfig:the-two-queues}
\end{figure*}
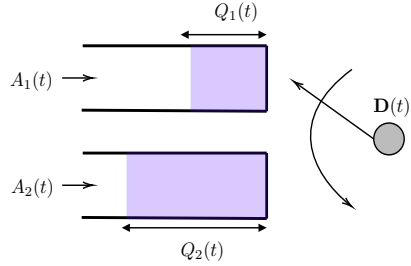
For simplicity, we assume that the processes $\left(A_i(t)\right)_{t=0}^\infty$ are both IID Bernoulli, with $\mathbb{E}A_i(t)=\lambda_i.$ Note that the arrival rate $\boldsymbol{\lambda}=[\lambda_1,\lambda_2]$ is {\color{blue}unknown} to the learner.  Defining $(x)^+:=\max\{0,x\},~\forall~x\in\mathbb{R},$ queue length evolution is given by the equations
\begin{equation}
    Q_i(t+1) = \left(Q_i(t)-D_i(t)\right)^+ + A_i(t+1),~i\in\{1,2\}.
\end{equation}

\section{Non-concavity of the Value function}\label{appendix:nonconcavity of V}
We show here that the value function $V^{\pi}(\rho)$ is in general non-concave, and hence standard convex optimization techniques for maximization may get stuck in local optima. We note once again that this is \emph{different} from the non-concavity of $V^\pi$ when the parameterization is over the entire state-action space, i.e., $\Real^{S\times A}$. 

We show here that for both SoftMax and direct parameterization, the value function is non-concave where, by \enquote{direct} parameterization we mean that the controllers $K_m$ are parameterized by weights $\theta_m\in \Real$, where $\theta_i\geq 0,~\forall i\in[M]$ and $\sum\limits_{i=1}^M\theta_i=1$.
A similar argument holds for softmax parameterization, which we outline in Note \ref{remark:nonconcavity for softmax}.
\begin{lemma}(Non-concavity of Value function)\label{lemma:nonconcavity of V}
There is an MDP and a set of controllers, for which the maximization problem  of the value function (i.e. \eqref{eq:main optimization problem}) is non-concave for SoftMax parameterization, i.e., $\theta\mapsto V^{\pi_{\theta}}$ is non-concave.
\end{lemma}
\begin{proof}
\begin{figure}
    \centering
\includegraphics[scale=0.3]{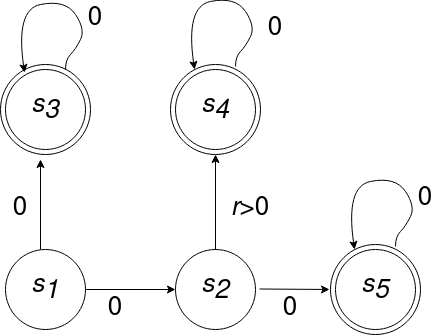}    
\caption{An example of an MDP with controllers as defined in \eqref{eqn:controllers-for-non-concavity-counterexample} having a non-concave value function. The MDP has $S=5$ states and $A=2$ actions. States $s_3,s_4\text{ and }s_5$ are terminal states. The only transition with nonzero reward is $s_2\rightarrow s_4.$}
    \label{fig:example showing non-concavity of the value function}
\end{figure}

Consider the MDP shown in Figure \ref{fig:example showing non-concavity of the value function} with 5 states, $s_1,\ldots,s_5$ . States $s_3,s_4$ and $s_5$ are terminal states. In the figure we also show the allowed transitions and the rewards obtained by those transitions. Let the action set $\cA$ consists of only three actions $\{a_1,a_2,a_3\} \equiv \{\tt{right}, \tt{up}, \tt{null}\}$, where 'null' is a dummy action included to accommodate the three terminal states.  Let us consider the case when $M=2$. The two controllers $K_i\in \Real^{S\times A}$, $i=1,2$ (where each row is probability distribution over $\cA$) are shown below.
\begin{equation}
    K_1 = \begin{bmatrix}1/4 & 3/4 & 0\\ 3/4 & 1/4 & 0\\ 0 & 0 & 1\\ 0 & 0 & 1\\ 0 & 0 & 1 \end{bmatrix}, K_2 = \begin{bmatrix}3/4 & 1/4 & 0\\ 1/4 & 3/4 & 0\\ 0 & 0 & 1\\ 0 & 0 & 1\\ 0 & 0 & 1 \end{bmatrix}.    
    \label{eqn:controllers-for-non-concavity-counterexample}
\end{equation}
Let $\theta^{(1)} = (1,0)\transpose$ and $\theta^{(2)} = (0,1)\transpose$. Let us fix the initial state to be $s_1$. Since a nonzero reward is only earned during a $s_2\rightarrow s_4$ transition, we note for any policy $\pi:\cA\to\cS$ that $V^\pi(s_1)=\pi(a_1|s_1)\pi(a_2|s_2)r$.
We also have,
\[(K_1 + K_2)/2 = \begin{bmatrix}1/2 & 1/2 & 0\\ 1/2 & 1/2 & 0\\ 0 & 0 & 1\\ 0 & 0 & 1\\ 0 & 0 & 1 \end{bmatrix}.\]

We will show that $\frac{1}{2}V^{\pi_{\theta^{(1)}}}+\frac{1}{2}V^{\pi_{\theta^{(2)}}} > V^{\pi_{\left(\theta^{(1)}+\theta^{(2)}\right)/2}}$.
\newline We observe the following.
\begin{align*}
    V^{\pi_{\theta^{(1)}}}(s_1) &= V^{K_1}(s_1) = (1/4).(1/4).r=r/16.\\
    V^{\pi_{\theta^{(2)}}}(s_1) &= V^{K_2}(s_1) = (3/4).(3/4).r=9r/16.
\end{align*}
where $V^{K}(s)$ denotes the value obtained by starting from state $s$ and following a controller matrix $K$ for all time.

Also, on the other hand we have, 
\[V^{\pi_{\left(\theta^{(1)}+\theta^{(2)}\right)/2}} = V^{\left(K_1+K_2\right)/2}(s_1) = (1/2).(1/2).r=r/4. \]
Hence we see that, \[\frac{1}{2}V^{\pi_{\theta^{(1)}}}+\frac{1}{2}V^{\pi_{\theta^{(2)}}} = r/32+9r/32 = 10r/32 =1.25r/4 > r/4 = V^{\pi_{\left(\theta^{(1)}+\theta^{(2)}\right)/2}}. \]
This shows that $\theta\mapsto V^{\pi_{\theta}}$ is non-concave, which concludes the proof for direct parameterization.
\begin{remark}\label{remark:nonconcavity for softmax}
For softmax parametrization, we choose the same 2 controllers $K_1,K_2$ as above. Fix some $\epsilon\in (0,1)$ and set $\theta^{(1)}= \left(\log(1-\epsilon), \log\epsilon \right)\transpose$ and $\theta^{(2)}= \left(\log\epsilon, \log(1-\epsilon) \right)\transpose$. A similar calculation using softmax projection, and using the fact that $\pi_\theta(a|s) = \sum\limits_{m=1}^M\pi_\theta(m)K_m(s,a)$, shows that under $\theta^{(1)}$ we follow matrix $(1-\epsilon)K_1+\epsilon K_2$, which yields a Value of $\left(1/4+\epsilon/2\right)^2r$. Under $\theta^{(2)}$ we follow matrix $\epsilon K_1+(1-\epsilon) K_2$, which yields a Value of $\left(3/4-\epsilon/2\right)^2r$. On the other hand,  $(\theta^{(1)}+\theta^{(2)})/2$ amounts to playing the matrix $(K_1+K_2)/2$, yielding the a value of $r/4$, as above.  One can verify easily that $\left(1/4+\epsilon/2\right)^2r + \left(3/4-\epsilon/2\right)^2r > 2.r/4$. This shows the non-concavity of $\theta\mapsto V^{\pi_{\theta}}$ under softmax parameterization.
\end{remark}
\end{proof}
\section{Example showing that the value function need not be\\  pointwise (over states) monotone over the improper class}\label{appendix:monotonicity does not hold}

Consider the same MDP as in Sec~\ref{appendix:nonconcavity of V}, however with different base controllers. Let the initial state be $s_1$. 

The two base controllers $K_i\in \Real^{S\times A}$, $i=1,2$ (where each row is probability distribution over $\cA$) are shown below.
\begin{equation}
    K_1 = \begin{bmatrix}1/4 & 3/4 & 0\\ 1/4 & 3/4 & 0\\ 0 & 0 & 1\\ 0 & 0 & 1\\ 0 & 0 & 1 \end{bmatrix}, K_2 = \begin{bmatrix}3/4 & 1/4 & 0\\ 3/4 & 1/4 & 0\\ 0 & 0 & 1\\ 0 & 0 & 1\\ 0 & 0 & 1 \end{bmatrix}.    
    \label{eqn:controllers-for-non-concavity-counterexample1}
\end{equation}
Let $\theta^{(1)} = (1,0)\transpose$ and $\theta^{(2)} = (0,1)\transpose$. Let us fix the initial state to be $s_1$. Since a nonzero reward is only earned during a $s_2\rightarrow s_4$ transition, we note for any policy $\pi$, that $V^\pi(s_1)=\pi(a_1|s_1)\pi(a_2|s_2)r$ and $V^\pi(s_2)=\pi(a_2|s_2)r$.
Note here that the optimal policy of this MDP is \emph{deterministic} with $\pi^*(a_1|s_1)=1$ and $\pi^*(a_2|s_2)=1$. The transitions are all deterministic. 

However, notice that the optimal policy (with initial state $s_1$) given $K_1$ and $K_2$ is \emph{strict mixture}, because, given any $\boldsymbol{\theta}=[\theta, 1-\theta],~\theta\in[0,1],$ the value of the policy $\pi_{\boldsymbol{\theta}}$ is 
\begin{equation}
    v^{\pi_{\boldsymbol{\theta}}}=\frac{1}{4}(3-2\theta)(1+2\theta)r,
\end{equation}
which is maximized at $\theta=1/2.$ This means that the optimal \emph{non deterministic} policy chooses $K_1$ and $K_2$ with probabilites $(1/2,1/2)$, i.e.,
\[K^*=(K_1 + K_2)/2 = \begin{bmatrix}1/2 & 1/2 & 0\\ 1/2 & 1/2 & 0\\ 0 & 0 & 1\\ 0 & 0 & 1\\ 0 & 0 & 1 \end{bmatrix}.\]
\newline We observe the following.
\begin{align*}
    V^{\pi_{\theta^{(1)}}}(s_1) &= V^{K_1}(s_1) = (1/4).(3/4).r=3r/16.\\
    V^{\pi_{\theta^{(2)}}}(s_1) &= V^{K_2}(s_1) = (3/4).(1/4).r=3r/16.\\
    V^{\pi_{\theta^{(1)}}}(s_2) &= V^{K_1}(s_2) = (3/4).r=3r/4.\\
\end{align*}

On the other hand we have, 
\begin{align*}
V^{\pi_{\left(\theta^{(1)}+\theta^{(2)}\right)/2}}(s_1) = V^{K^*}(s_1) = (1/2).(1/2).r=r/4.\\
V^{\pi_{\left(\theta^{(1)}+\theta^{(2)}\right)/2}}(s_2) = V^{K^*}(s_2) = (1/2).r=r/2.\\
\end{align*}
We see that $V^{K^*}(s_1)> \max\{V^{K_1}(s_1), V^{K_2}(s_1)\}$. However, $V^{K^*}(s_2)< V^{K_1}(s_2)$. This implies that playing according to an improved mixture policy (here the optimal given the initial state is $s_1$) does not necessarily improve the value across \textit{all} states.

\section{Proof details for Bandit-over-bandits}\label{appendix:proofs for MABs}

In this section we consider the instructive sub-case when $S=1$, which is also called the Multiarmed Bandit. We provide regret bounds for two cases (1) when the value gradient $\frac{dV^{\pi_{\theta_t}}(\mu)}{d\theta^t}$ (in the gradient update) is available in each round, and (2) when it needs to be estimated.

Note that each controller in this case, is a probability distribution over the $A$ arms of the bandit. We consider the scenario where the agent at each time $t\geq 1$, has to choose a probability distribution $K_{m_t}$ from a set of $M$ probability distributions over actions $\cA$. She then plays an action $a_t\sim K_{m_t}$. This is different from the standard MABs because the learner cannot choose the actions directly, instead chooses from a \emph{given} set of controllers, to play actions. 
Note the $V$ function has no argument as $S=1$.
Let $\mu\in [0,1]^A$ be the mean vector of the arms $\cA$. The value function for any given mixture $\pi\in \cP([M])$, 
\begin{align}\label{eq:bandits-value function}
    V^\pi &\bydef \expect{\sum\limits_{t=0}^\infty \gamma^tr_t\given \pi } = \sum\limits_{t=0}^\infty \gamma^t\expect{r_t\given \pi}\nonumber\\
    &=\sum\limits_{t=0}^\infty \gamma^t\sum\limits_{a\in \cA}\sm  \pi(m)K_m(a)\mu_a. \nonumber\\
    &= \frac{1}{1-\gamma} \sm \pi_m \mu\transpose K_m = \frac{1}{1-\gamma} \sm \pi_m \mathfrak{r}^\mu_m.
\end{align}
where the interpretation of $\kr_m^\mu$ is that it is the mean reward one obtains if the controller $m$ is chosen at any round $t$. 
 Since $V^{\pi}$ is linear in $\pi$, the maximum is attained at one of the base controllers $\pi^*$ puts mass 1 on $m^*$ where
 $m^*:= \argmax\limits_{m\in [M]} V^{K_m},$
and $V^{K_m}$ is the value obtained using $K_m$ for all time. In the sequel, we assume $\Delta_i\bydef \mathfrak{r}^\mu_{m^*}-\mathfrak{r}^\mu_i>0$.

\subsection{Proofs for MABs with perfect gradient knowledge}

With access to the exact value gradient at each step, we have the following result, when Softmax PG (Algorithm ~\ref{alg:mainPolicyGradMDP}) is applied for the bandits-over-bandits case.
\begin{restatable}{theorem}{convergencebandits}\label{thm:convergence for bandit}
With $\eta=\frac{2(1-\gamma)}{5}$ and with $\theta^{(1)}_m=1/M$ for all $m\in [M]$, with the availability for true gradient, we have $\forall t\geq 1$,
\[V^{\pi^*}-V^{\pi_{\theta_t}} \leq \frac{5}{1-\gamma} \frac{M^2}{t}.\]
\end{restatable}
Also, defining regret for a time horizon of $T$ rounds as 
\begin{equation}\label{eq:regret definition MAB}
    \mathcal{R}(T):= \sum\limits_{t=1}^T V^{\pi^*} - V^{\pi_{\theta_t}},
\end{equation}
we show as a corollary to Thm.~\ref{thm:convergence for bandit} that,
\begin{corollary}\label{cor:regret true gradient MAB}
\[\mathcal{R}(T)\leq \min\left\{ \frac{5M^2}{1-\gamma}{\color{blue}\log T}, \sqrt{\frac{5}{1-\gamma}}M{\color{blue}\sqrt{T}} \right\}. \]
\end{corollary}

\begin{proof}
Recall from eq (\ref{eq:bandits-value function}), that the value function for any given policy $\pi \in \cP([M])$, that is a distribution over the given $M$ controllers (which are itself distributions over actions $\cA$) can be simplified as:
\[V^\pi = \frac{1}{1-\gamma} \sm \pi_m \mu\transpose K_m = \frac{1}{1-\gamma} \sm \pi_m \mathfrak{r}^\mu_m \]
where $\mu$ here is the (unknown) vector of mean rewards of the arms $\cA$. Here, $\kr^\mu_m:= \mu\transpose K_m$, $i=1,\cdots, M$, represents the mean reward obtained by choosing to play controller $K_m, m\in {M}$. For ease of notation, we will drop the superscript $\mu$ in the proofs of this section. We first show a simplification of the gradient of the value function w.r.t. the parameter $\theta$. Fix a $m\in [M]$, 
\begin{equation}\label{eq:grad simplification for MAB}
    \frac{\partial}{\partial \theta_{m'}} V^{\pi_\theta} =  \frac{1}{1-\gamma} \sm \frac{\partial}{\partial \theta_m} \pi_\theta(m) \mathfrak{r}_m = \frac{1}{1-\gamma} \sm \pi_\theta(m') \left\{\ind_{mm'}-\pi_\theta(m) \right\} \kr_m. 
\end{equation}
Next we show that $V^\pi$ is $\beta-$ smooth. A function $f:\Real^M\to \Real$ is $\beta-$ smooth, if $\forall \theta',\theta \in \Real^M$
\[\abs{f(\theta') - f(\theta) -\innprod{\frac{d}{d\theta}f(\theta), \theta'-\theta}} \leq \frac{\beta}{2} \norm{\theta'-\theta}_2^2.\]

Let $S:=\frac{d^2}{d\theta^2} V^{\pi_{\theta}} $. This is a matrix of size $M\times M$. Let $1\leq i,j\leq M$.
\begin{align}
    S_{i,j}&=\left(\del\left(\del V^{\pi_{\theta}} \right)\right)_{i,j}\\
    &=\frac{1}{1-\gamma} \frac{d(\pi_\theta(i)(\kr(i)-\pi_\theta\transpose \kr))}{d\theta_j}\\
    &=\frac{1}{1-\gamma} \left(\frac{d\pi_\theta(i)}{d\theta_j}(\kr(i)-\pi_\theta\transpose \kr) + \pi_\theta(i)\frac{d(\kr(i)-\pi_\theta\transpose \kr)}{d\theta_j}      \right)\\
    &=\frac{1}{1-\gamma} \left(\pi_\theta(j)(\kr(i)-\pi_\theta\transpose\kr) - \pi_\theta(i)\pi_\theta(j)(\kr(i)-\pi_\theta\transpose\kr) - \pi_\theta(i)\pi_\theta(j)(\kr(j)-\pi_\theta\transpose\kr) \right).
\end{align}
Next, let $y\in \Real^M$, 
\begin{align*}
    \abs{y\transpose Sy} &= \abs{\si\sj S_{ij}y(i)y(j)}\\
    &=\frac{1}{1-\gamma} \abs{\si\sj  \left(\pi_\theta(j)(\kr(i)-\pi_\theta\transpose\kr) - \pi_\theta(i)\pi_\theta(j)(\kr(i)-\pi_\theta\transpose\kr) - \pi_\theta(i)\pi_\theta(j)(\kr(j)-\pi_\theta\transpose\kr) \right) y(i)y(j) }\\
    &=\frac{1}{1-\gamma} \abs{\si  \pi_\theta(i)(\kr(i)-\pi_\theta\transpose\kr)y(i)^2 - 2\si\sj \pi_\theta(i)\pi_\theta(j)(\kr(i)-\pi_\theta\transpose\kr) y(i)y(j) }\\
    &= \frac{1}{1-\gamma} \abs{ \si  \pi_\theta(i)(\kr(i)-\pi_\theta\transpose\kr)y(i)^2 -2 \si \pi_\theta(i)(\kr(i)-\pi_\theta\transpose\kr) y(i)\sj \pi_\theta(j)y(j)  }\\
    &\leq \frac{1}{1-\gamma} \abs{ \si  \pi_\theta(i)(\kr(i)-\pi_\theta\transpose\kr)y(i)^2} +\frac{2}{1-\gamma}\abs{ \si \pi_\theta(i)(\kr(i)-\pi_\theta\transpose\kr) y(i)\sj \pi_\theta(j)y(j)  } \\
    &\leq  \frac{1}{1-\gamma} \norm{\pi_\theta\odot (\kr-\pi_\theta\transpose\kr)}_\infty \norm{y\odot y }_1 + \frac{2}{1-\gamma} \norm{\pi_\theta\odot (\kr-\pi_\theta\transpose\kr)}_1.\norm{y}_\infty. \norm{\pi_\theta}_1\norm{y}_\infty.
\end{align*}
The last equality is by the assumption that reward are bounded in [0,1].
We observe that,
\begin{align*}
        \norm{\pi_\theta\odot (\kr-\pi_\theta\transpose\kr)}_1 &=\sm \abs{\pi_\theta(i)(\kr(i) -\pi_\theta\transpose\kr)  }\\
        &= \sm \pi_\theta(i)\abs{\kr(i) -\pi_\theta\transpose\kr }\\
        &=\max\limits_{i=1,\ldots, M} \abs{\kr(i) -\pi_\theta\transpose\kr }\leq 1. 
\end{align*}
Next, for any $i\in [M]$,
\begin{align*}
    \abs{\pi_\theta(i) (\kr(i)-\pi_\theta\transpose\kr) } &=\abs{\pi_\theta(i)\kr(i) -\pi_\theta(i)^2r(i)-\sum\limits_{j\neq i}\pi_\theta(i)\pi_\theta(j)\kr(j) }\\
    &=\pi_\theta(i)(1-\pi_\theta(i)) + \pi_\theta(i) (1-\pi_\theta(i)) \leq 2.1/4 =1/2.
\end{align*}
Combining the above two inequalities with the fact that $\norm{\pi_\theta}_1=1$ and $\norm{y}_\infty \leq \norm{y}_2$, we get,
\[\abs{y\transpose Sy} \leq \frac{1}{1-\gamma} \norm{\pi_\theta\odot (\kr-\pi_\theta\transpose\kr)}_\infty \norm{y\odot y }_1 + \frac{2}{1-\gamma} \norm{\pi_\theta\odot (\kr-\pi_\theta\transpose\kr)}_1.\norm{y}_\infty. \norm{\pi_\theta}_1\norm{y}_\infty \leq \frac{1}{1-\gamma}(1/2+2)\norm{y}_2^2. \]
Hence $V^{\pi_\theta}$ is $\beta-$smooth with $\beta = \frac{5}{2(1-\gamma)}$.

We establish a lower bound on the norm of the gradient of the value function at every step $t$ as below (these type of inequalities are called {\L}ojaseiwicz inequalities \cite{lojasiewicz1963equations})
\begin{lemma}\label{lemma:nonuniform lojaseiwicz MAB}[Lower bound on norm of gradient]
\[\norm{\frac{\partial V^{\pi_\theta}}{\partial \theta}}_2 \geq \pi_{\theta_{m^*}} \left(V^{\pi^*}-V^{\pi_{\theta}}\right). \]
\end{lemma}

\underline{Proof of Lemma \ref{lemma:nonuniform lojaseiwicz MAB}.}

\begin{proof}
Recall from the simplification of gradient of $V^\pi$, i.e., eq (\ref{eq:grad simplification for MAB}):
\begin{align*}
    \frac{\partial}{\partial \theta_{m}} V^{\pi_\theta} &=   \frac{1}{1-\gamma} \sum\limits_{m'=1}^M \pi_\theta(m) \left\{\ind_{mm'}-\pi_\theta(m') \right\} \kr_m'\\
    &=\frac{1}{1-\gamma} \pi(m)\left(\kr(m)-\pi\transpose \kr \right).
\end{align*}
Taking norm both sides,
\begin{align*}
    \norm{\frac{\partial}{\partial \theta} V^{\pi_\theta}} &=\frac{1}{1-\gamma} \sqrt{\sm (\pi(m))^2\left(\kr(m)-\pi\transpose \kr \right)^2  }\\
    &\geq \frac{1}{1-\gamma} \sqrt{ (\pi(m^*))^2\left(\kr(m^*)-\pi\transpose \kr \right)^2  }\\
    &=\frac{1}{1-\gamma}  (\pi(m^*))\left(\kr(m^*)-\pi\transpose \kr \right)  \\
    &=\frac{1}{1-\gamma}  (\pi(m^*))\left(\pi^*-\pi\right)\transpose \kr \\
    &=(\pi(m^*)) \left[ V^{\pi^*}-V^{\pi_\theta}\right].
\end{align*}
where $\pi^*=e_{m^*}$.
\end{proof}

We will now prove Theorem \ref{thm:convergence for bandit} and corollary \ref{cor:regret true gradient MAB}. We restate the result here.
\begin{restatable}{theorem}{convergencebandits}\label{thm:convergence for bandit}
With $\eta=\frac{2(1-\gamma)}{5}$ and with $\theta^{(1)}_m=1/M$ for all $m\in [M]$, with the availability for true gradient, we have $\forall t\geq 1$,
\[V^{\pi^*}-V^{\pi_{\theta_t}} \leq \frac{5}{1-\gamma} \frac{M^2}{t}.\]
\end{restatable}
\begin{proof}
First, note that since $V^\pi$ is smooth we have:
\begin{align*}
    V^{\pi_{\theta_t}}- V^{\pi_{\theta_{t+1}}} &\leq -\innprod{\frac{d}{d\theta_t}V^{\pi_{\theta_t}}, \theta_{t+1}-\theta_t} +\frac{5}{2(1-\gamma)} \norm{\theta_{t+1}-\theta_t}_2^2\\
    &=-\eta\norm{ \frac{d}{d\theta_t}V^{\pi_{\theta_t}}}_2^2 + \frac{5}{4(1-\gamma)}\eta^2 \norm{ \frac{d}{d\theta_t}V^{\pi_{\theta_t}}}_2^2\\
    &=\norm{ \frac{d}{d\theta_t}V^{\pi_{\theta_t}}}_2^2\left(\frac{5\eta^2}{4(1-\gamma)} -\eta \right)\\
    &= -\left(\frac{1-\gamma}{5}\right) \norm{ \frac{d}{d\theta_t}V^{\pi_{\theta_t}}}_2^2.\\
    &\leq -\left(\frac{1-\gamma}{5}\right) (\pi_{\theta_t}(m^*))^2 \left[ V^{\pi^*}-V^{\pi_\theta}\right]^2 \qquad \text{Lemma \ref{lemma:nonuniform lojaseiwicz MAB}}\\
    &\leq -\left(\frac{1-\gamma}{5}\right) (\underbrace{\inf\limits_{1\leq s\leq t}\pi_{\theta_t}(m^*)}_{=:c_t})^2 \left[ V^{\pi^*}-V^{\pi_\theta}\right]^2.
\end{align*}
The first equality is by smoothness, second inequality is by the update equation in algorithm \ref{alg:mainPolicyGradMDP}.

Next, let $\delta_t:= V^{\pi^*}-V^{\pi_{\theta_t}}$. We have,
\begin{equation}\label{eq:induction hyp MAB}
    \delta_{t+1}-\delta_t \leq -\frac{(1-\gamma)}{5}c_t^2 \delta_t^2.
\end{equation}
\textbf{Claim:}
$\forall t\geq1, \delta_t\leq \frac{5}{c_t^2(1-\gamma)} \frac{1}{t}.$ 
\newline
We prove the claim by using induction on $t\geq 1$.
\newline \underline{Base case.} Since $\delta_t\leq \frac{1}{1-\gamma}$, the claim is true for all $t\leq 5$.
\newline \underline{Induction step:} Let $\phi_t:=\frac{5}{c_t^2(1-\gamma)}$. Fix a $t\geq 2$, assume $\delta_t \leq \frac{\phi_t}{t}$.

Let $g:\Real\to \Real$ be a function defined as $g(x) = x-\frac{1}{\phi_t}x^2$. One can verify easily that $g$ is monotonically increasing in $\left[ 0, \frac{\phi_t}{2}\right]$. Next with equation \ref{eq:induction step}, we have
\begin{align*}
    \delta_{t+1} &\leq \delta_t -\frac{1}{\phi_t} \delta_t^2\\
    &= g(\delta_t)\\
    &\leq g(\frac{\phi_t}{ t})\\
    &\leq \frac{\phi_t}{t} - \frac{\phi_t}{t^2}\\
    &= \phi_t\left(\frac{1}{t}-\frac{1}{t^2} \right)\\
    &\leq \phi_t\left(\frac{1}{t+1}\right).
\end{align*}
This completes the proof of the claim. We will show that $c_t\geq 1/M$ in the next lemma. We first complete the proof of the corollary assuming this.

We fix a $T\geq 1$. Observe that, $\delta_t\leq \frac{5}{(1-\gamma)c_t^2}\frac{1}{t}\leq \frac{5}{(1-\gamma)c_T^2}\frac{1}{t}$.
\[\sum\limits_{t=1}^T V^{\pi^*}-V^{\pi_{\theta_t}} = \frac{1}{1-\gamma} \sum\limits_{t=1}^T (\pi^*-\pi_{\theta_t})\transpose \kr\leq \frac{5\log T}{(1-\gamma)c_T^2}+1. \]
Also we have that, 
\[\sum\limits_{t=1}^T V^{\pi^*}-V^{\pi_{\theta_t}} = \sum\limits_{t=1}^T \delta_t \leq \sqrt{T} \sqrt{\sum\limits_{t=1}^T\delta_t^2}\leq \sqrt{T} \sqrt{\sum\limits_{t=1}^T \frac{5}{(1-\gamma)c_T^2}(\delta_t-\delta_{t+1}) }\leq \frac{1}{c_T}\sqrt{\frac{5T}{(1-\gamma)}}. \]
We next show that with $\theta_m^{(1)}=1/M,\forall m$, i.e., uniform initialization, $\inf_{t\geq 1}c_t=1/M$, which will then complete the proof of Theorem \ref{thm:convergence for bandit} and of corollary \ref{cor:regret true gradient MAB}.

\begin{lemma}\label{lemma:c bounded from zero :MAB}
We have $\inf_{t\geq 1} \pi_{\theta_t}(m^*) >0$. Furthermore, with uniform initialization of the parameters $\theta_m^{(1)}$, i.e., $1/M, \forall m\in [M]$, we have $\inf_{t\geq 1} \pi_{\theta_t}(m^*) = \frac{1}{M}$.
\end{lemma}
\begin{proof}
We will show that there exists $t_0$ such that $\inf_{t\geq 1} \pi_{\theta_t}(m^*) = \min\limits_{1\leq t\leq t_0}\pi_{\theta_t}(m^*)$, where $t_0=\min\left\{t: \pi_{\theta_t}(m^*)\geq C \right\}$. We define the following sets.
\begin{align*}
    \cS_1 &= \left\{\theta: \frac{dV^{\pi_\theta}}{d\theta_{m^*}} \geq \frac{dV^{\pi_\theta}}{d\theta_{m}}, \forall m\neq m^*  \right\}\\
    \cS_2 &= \left\{\theta: \pi_\theta(m^*) \geq \pi_\theta(m), \forall m\neq m^*  \right\} \\
    \cS_3 &= \left\{\theta: \pi_\theta(m^*) \geq C \right\}
\end{align*}
Note that $\cS_3$ depends on the choice of $C$.  Let $C:=\frac{M-\Delta}{M+\Delta}$. We claim the following:
\newline \textbf{Claim 2.} $(i) \theta_t\in \cS_1\implies \theta_{t+1}\in \cS_1$ and $(ii) \theta_t\in \cS_1 \implies \pi_{\theta_{t+1}}(m^*) \geq \pi_{\theta_{t}}(m^*)$.
\begin{proof}[Proof of Claim 2]
$(i)$ Fix a $m\neq m^*$. We will show that if $\frac{dV^{\pi_\theta}}{d\theta_t(m^*)} \geq \frac{dV^{\pi_\theta}}{d\theta_t(m)} $, then $\frac{dV^{\pi_\theta}}{d\theta_{t+1}(m^*)} \geq \frac{dV^{\pi_\theta}}{d\theta_{t+1}(m)} $. This will prove the first part. 
\newline \underline{Case (a): $\pi_{\theta_t}(m^*)\geq \pi_{\theta_t}(m)$}. This implies, by the softmax property, that $\theta_t(m^*) \geq \theta_t(m)$. After gradient ascent update step we have:
\begin{align*}
    \theta_{t+1}(m^*) &= \theta_{t}(m^*) +\eta \frac{dV^{\pi_{\theta_t}}}{d\theta_t(m^*)}\\
    &\geq \theta_t(m) + \eta \frac{dV^{\pi_{\theta_t}}}{d\theta_t(m)}\\
    &= \theta_{t+1}(m).
\end{align*}
This again implies that $\theta_{t+1}(m^*) \geq \theta_{t+1}(m)$. By the definition of derivative of $V^{\pi_{\theta}}$ w.r.t $\theta_t$ (see eq (\ref{eq:grad simplification for MAB})), 
\begin{align*}
    \frac{dV^{\pi_\theta}}{d\theta_{t+1}(m^*)} &= \frac{1}{1-\gamma} \pi_{\theta_{t+1}(m^*)}(\kr(m^*)-\pi_{\theta_{t+1}}\transpose \kr)\\
    &=\frac{1}{1-\gamma} \pi_{\theta_{t+1}(m)}(\kr(m)-\pi_{\theta_{t+1}}\transpose \kr)\\
    &= \frac{dV^{\pi_\theta}}{d\theta_{t+1}(m)}.
\end{align*}
This implies $\theta_{t+1}\in \cS_1$.
\newline \underline{Case (b): $\pi_{\theta_t}(m^*)< \pi_{\theta_t}(m)$}. We first note the following equivalence:
\[\frac{dV^{\pi_{\theta}}}{d\theta(m^*)}\geq \frac{dV^{\pi_{\theta}}}{d\theta(m)} \longleftrightarrow (\kr(m^*)-\kr(m))\left(1-\frac{\pi_\theta(m^*)}{\pi_\theta(m^*)}\right)(\kr(m^*)-\pi_\theta\transpose\kr).\]
which can be simplified as:
\begin{align*}
    (\kr(m^*)-\kr(m))\left(1-\frac{\pi_\theta(m^*)}{\pi_\theta(m^*)}\right)(\kr(m^*)-\pi_\theta\transpose\kr) & = (\kr(m^*)-\kr(m))\left(1-\exp\left(\theta_t(m^*)-\theta_t(m)  \right) \right)(\kr(m^*)-\pi_\theta\transpose\kr).
\end{align*}
The above condition can be rearranged as:
\[\kr(m^*)-\kr(m)\geq \left(1-\exp\left(\theta_t(m^*)-\theta_t(m) \right)\right)\left(\kr(m^*)-\pi_{\theta_{t}}\transpose\kr\right). \]
By lemma \ref{lemma:gradient ascent lemma}, we have that $V^{\pi_{\theta_{t+1}}}\geq V^{\pi_{\theta_{t}}}\implies \pi_{\theta_{t+1}}\transpose \kr \geq \pi_{\theta_{t}}\transpose \kr$. Hence,
\[0<\kr(m^*)-\pi_{\theta_{t+1}}\transpose\kr \leq \pi_{\theta_{t}}\transpose\kr. \]
Also, we note:
\[\theta_{t+1}(m^*) -\theta_{t+1}(m) = \theta_t(m^*) + \eta \frac{dV^{\pi_t}}{d\theta_t(m^*)} - \theta_{t+1}(m) -\eta \frac{dV^{\pi_t}}{d\theta_t(m)} \geq \theta_t(m^*)-\theta_t(m).\]
This implies, $1-\exp\left(\theta_{t+1}(m^*) -\theta_{t+1}(m) \right)\leq 1-\exp\left(\theta_t(m^*)-\theta_t(m) \right)$.

Next, we observe that by the assumption $\pi_t(m^*)< \pi_t(m)$, we have \[1-\exp\left(\theta_t(m^*)-\theta_t(m) \right) =1-\frac{\pi_t(m^*)}{\pi_t(m)}>0. \]
Hence we have,
\begin{align*}
    \left(1-\exp\left(\theta_{t+1}(m^*)-\theta_{t+1}(m) \right)\right)\left(\kr(m^*)-\pi_{\theta_{t+1}}\transpose\kr\right) &\leq \left(1-\exp\left(\theta_t(m^*)-\theta_t(m) \right)\right)\left(\kr(m^*)-\pi_{\theta_{t}}\transpose\kr\right)\\
    &\leq \kr(m^*)-\kr(m). 
\end{align*}
Equivalently,
\[\left(1-\frac{\pi_{t+1}(m^*)}{\pi_{t+1}(m)} \right)(\kr(m^*)-\pi_{t+1}\transpose\kr )\leq \kr(m^*)-\kr(m). \]
Finishing the proof of the claim 2(i).
\newline (ii) Let $\theta_t\in \cS_1$. We observe that:
\begin{align*}
    \pi_{t+1}(m^*) &= \frac{\exp(\theta_{t+1}(m^*))}{\sm \exp(\theta_{t+1}(m))}\\
    &=\frac{\exp(\theta_{t}(m^*)+\eta \frac{dV^{\pi_t}}{d\theta_t(m^*)})}{\sm \exp(\theta_{t}(m)+\eta \frac{dV^{\pi_t}}{d\theta_t(m)})}\\
    &\geq \frac{\exp(\theta_{t}(m^*)+\eta \frac{dV^{\pi_t}}{d\theta_t(m^*)})}{\sm \exp(\theta_{t}(m)+\eta \frac{dV^{\pi_t}}{d\theta_t(m^*)})}\\
    &=\frac{\exp(\theta_{t}(m^*))}{\sm \exp(\theta_{t}(m))} = \pi_t(m^*)
\end{align*}
This completes the proof of Claim 2(ii).
\end{proof}
\textbf{Claim 3.} $\cS_2\subset \cS_1$ and $\cS_3\subset \cS_1$.
\begin{proof}
To show that $\cS_2\subset \cS_1$, let $\theta\in cS_2$. 
We have $\pi_\theta(m^*) \geq \pi_\theta(m), \forall m\neq m^*$. 
\begin{align*}
    \frac{dV^{\pi_\theta}}{d\theta(m^*)} &=\frac{1}{1-\gamma} \pi_\theta(m^*)(\kr(m^*)-\pi_\theta\transpose \kr)\\
    &>\frac{1}{1-\gamma} \pi_\theta(m)(\kr(m)-\pi_\theta\transpose \kr)\\
    &=\frac{dV^{\pi_\theta}}{d\theta(m)}.
\end{align*}
This shows that $\theta\in \cS_1$. For showing the second part of the claim, we assume $\theta\in \cS_3\cap \cS_2^c$, because if $\theta\in \cS_2$, we are done. Let $m\neq m^*$. 
We have,
\begin{align*}
    \frac{dV^{\pi_\theta}}{d\theta(m^*)} - \frac{dV^{\pi_\theta}}{d\theta(m)} &= \frac{1}{1-\gamma} \left(\pi_\theta(m^*)(\kr(m^*)-\pi_\theta\transpose \kr) -\pi_\theta(m)(\kr(m)-\pi_\theta\transpose \kr)   \right)\\
    &= \frac{1}{1-\gamma} \left(2\pi_\theta(m^*)(\kr(m^*)-\pi_\theta\transpose \kr) + \sum\limits_{i\neq m^*,m}^M \pi_\theta(i)(\kr(i)-\pi_\theta\transpose \kr)  \right)\\
    &= \frac{1}{1-\gamma} \left( \left(2\pi_\theta(m^*) + \sum\limits_{i\neq m^*,m}^M \pi_\theta(i)\right)(\kr(m^*)-\pi_\theta\transpose \kr) - \sum\limits_{i\neq m^*,m}^M \pi_\theta(i)(\kr(m^*)-\kr(i)) \right)\\
    &\geq \frac{1}{1-\gamma} \left( \left(2\pi_\theta(m^*) + \sum\limits_{i\neq m^*,m}^M \pi_\theta(i)\right)(\kr(m^*)-\pi_\theta\transpose \kr) - \sum\limits_{i\neq m^*,m}^M \pi_\theta(i) \right)\\
    &\geq  \frac{1}{1-\gamma} \left( \left(2\pi_\theta(m^*) + \sum\limits_{i\neq m^*,m}^M \pi_\theta(i)\right)\frac{\Delta}{M}- \sum\limits_{i\neq m^*,m}^M \pi_\theta(i) \right).
\end{align*}
Observe that, $\sum\limits_{i\neq m^*,m}^M \pi_\theta(i) = 1-\pi(m^*)-\pi(m)$. Using this and rearranging we get,
\[\frac{dV^{\pi_\theta}}{d\theta(m^*)} - \frac{dV^{\pi_\theta}}{d\theta(m)} \geq \frac{1}{1-\gamma}\left(\pi(m^*)\left(1+\frac{\Delta}{M}\right) - \left(1-\frac{\Delta}{M}\right) + \pi(m) \left(1-\frac{\Delta}{M}\right)\right) \geq \frac{1}{1-\gamma}\pi(m) \left(1-\frac{\Delta}{M}\right)\geq 0. \]
The last inequality follows because $\theta\in \cS_3$ and the choice of $C$. This completes the proof of Claim 3.
\end{proof}
\textbf{Claim 4.} There exists a finite $t_0$, such that $\theta_{t_0}\in \cS_3$.
\begin{proof}
The proof of this claim relies on the asymptotic convergence result of \cite{Agarwal2020}. We note that their convergence result hold for our choice of $\eta=\frac{2(1-\gamma)}{5}$. As noted in \cite{Mei2020}, the choice of $\eta$ is used to justify the gradient ascent lemma \ref{lemma:gradient ascent lemma}. Hence we have $\pi_{\theta_t}{\to} 1$ as ${t\to \infty}$. Therefore, there exists a finite $t_0$ such that $\pi_{\theta_{t_0}}(m^*)\geq C$ and hence $\theta_{t_0}\in \cS_3$.
\end{proof}
This completes the proof that there exists a $t_0$ such that $\inf\limits_{t\geq 1}\pi_{\theta_t}(m^*) = \inf\limits_{1\leq t\leq t_0}\pi_{\theta_t}(m^*)$, since once the $\theta_t\in \cS_3$, by Claim 3, $\theta_t\in \cS_1$. Further, by Claim 2, $\forall t\geq t_0$, $\theta_t\in \cS_1$ and $\pi_{\theta_t}(m^*)$ is non-decreasing after $t_0$.
\end{proof}
With uniform initialization $\theta_1(m^*)=\frac{1}{M}\geq \theta_1(m)$, for all $m\neq m^*$. Hence, $\pi_{\theta_1}(m^*)\geq \pi_{\theta_1}(m) $ for all $m\neq m^*$. This implies $\theta_1\in \cS_2$, which implies $\theta_1\in \cS_1$. As established in Claim 2, $\cS_1$ remains invariant under gradient ascent updates, implying $t_0=1$. Hence we have that $\inf\limits_{t\geq 1}\pi_{\theta_t}(m^*) = \pi_{\theta_1}(m^*)=1/M$, completing the proof of Theorem \ref{thm:convergence for bandit} and corollary \ref{cor:regret true gradient MAB}.
\end{proof}

\end{proof}
\subsection{Proofs for MABs with noisy gradients}
When value gradients are unavailable, we follow a direct policy gradient algorithm instead of softmax projection. The full pseudo-code is provided here in Algorithm \ref{alg:ProjectionFreePolicyGradient}. At each round $t\geq 1$, the learning rate for $\eta$ is chosen asynchronously for each controller $m$, to be $\alpha \pi_t(m)^2$, to ensure that we remain inside the simplex, for some $\alpha \in (0,1)$. To justify its name as a policy gradient algorithm, observe that in order to minimize regret, we need to solve the following optimization problem:
\[\min\limits_{\pi\in \cP([M])} \sm \pi(m)(\kr_\mu(m^*)- \kr_\mu(m)). \]
A direct gradient with respect to the parameters $\pi(m)$ gives us a rule for the policy gradient algorithm. The other changes in the update step (eq \ref{algstep:update noisy gradien--repeated}), stem from the fact that true means of the arms are unavailable and importance sampling.
\begin{algorithm}[tb]
   \caption{Projection-free Policy Gradient (for MABs)}
   \label{alg:ProjectionFreePolicyGradient}
\begin{algorithmic}
   \STATE {\bfseries Input:} learning rate $\eta\in (0,1)$
   \STATE Initialize each $\pi_1(m)=\frac{1}{M}$, for all $m\in [M]$.
   \FOR{$t=1$ {\bfseries to} $T$}
   \STATE $m_*(t)\leftarrow \argmax\limits_{m\in [M]} \pi_t(m)$
   \STATE Choose controller $m_t\sim \pi_t$.
   \STATE Play action $a_t \sim K_{m_t}$.
   \STATE Receive reward  $R_{m_t}$ by pulling arm $a_t$.
   \STATE Update $\forall m\in [M], m\neq m_*(t): $ 
   \begin{equation}\label{algstep:update noisy gradien--repeated}
   \pi_{t+1}(m) = \pi_t(m) + \eta\left(\frac{R_{m}\ind_{m}}{\pi_t(m)} - \frac{R_{m_*(t)}\ind_{m_*(t)}}{\pi_t(m_*(t))}\right)    
   \end{equation}
   \STATE Set $\pi_{t+1}(m_*(t)) = 1- \sum\limits_{m\neq m_*(t)}\pi_{t+1}(m)$.
   \ENDFOR
\end{algorithmic}
\end{algorithm}

We have the following result.
\begin{restatable}{theorem}{regretnoisyMAB}\label{thm:regret noisy gradient}
With value of $\alpha$ chosen to be less than $\frac{\Delta_{min}}{\mathfrak{r}^\mu_{m^*}-\Delta_{min}}$, $\left(\pi_t\right)$ is a Markov process, with $\pi_t(m^*)\to 1$ as $t\to \infty, a.s.$ Further the regret till any time $T$ is bounded as
\[\cR(T) \leq \frac{1}{1-\gamma}\sum\limits_{m\neq m^*} \frac{\Delta_m}{\alpha\Delta_{min}^2} \log T + C,\]
\end{restatable}
where $C\bydef \frac{1}{1-\gamma}\sum\limits_{t\geq 1}\mathbb{P}\left\{\pi_{t}(m^*(t)) \leq \frac{1}{2}\right\}<\infty$.

We make couple of remarks before providing the full proof of Theorem \ref{thm:regret noisy gradient}.

\begin{remark}\label{remark:relation between true and noisy grad}
The ``cost" of not knowing the true gradient seems to cause the dependence on $\Delta_{min}$ in the regret, as is not the case when true gradient is available (see Theorem \ref{thm:convergence for bandit} and Corollary \ref{cor:regret true gradient MAB}). The dependence on $\Delta_{min}$ as is well known from the work of \cite{LAI19854}, is unavoidable.
\end{remark}

\begin{remark}\label{remark:dependence on delta}
The dependence of $\alpha$  on $\Delta_{min}$ can be removed by a more sophisticated choice of learning rate, at the cost of an extra $\log T$ dependence on regret \cite{Denisov2020RegretAO}.  
\end{remark}

\begin{proof}
The proof is an extension of that of Theorem 1 of \cite{Denisov2020RegretAO} for the setting that we have. The proof is divided into three main parts. In the first part we show that the recurrence time of the process $\{\pi_t(m^*)\}_{t\geq 1}$ is almost surely finite. Next we bound the expected value of the time taken by the process $\pi_t(m^*)$ to reach 1. Finally we show that almost surely, $\lim\limits_{t\to\infty}\pi_t(m^*) \to 1$, in other words the process  $\{\pi_t(m^*)\}_{t\geq 1}$ is transient. We use all these facts to show a regret bound.
\newline
Recall $m_*(t):=\argmax\limits_{m\in [M]} \pi_t(m)$.
We start by defining the following quantity which will be useful for the analysis of algorithm \ref{alg:ProjectionFreePolicyGradient}. 

Let $\tau\bydef \min \left\{t\geq 1: \pi_t(m^*) > \frac{1}{2}  \right\}$. 

Next, let $\cS:=\left\{\pi\in \cP([M]): \frac{1-\alpha}{2} \leq \pi(m^*) < \frac{1}{2}\right\}$. 

In addition, we define for any $a \in \Real$,
$\cS_a:=\left\{\pi\in \cP([M]): \frac{1-\alpha}{a} \leq \pi(m^*) < \frac{1}{x}\right\}$.
Observe  that if $\pi_1(m^*)\geq 1/a$ and $\pi_2(m^*) < 1/a$ then $\pi_1\in \cS_a$.  This fact follows just by the update step of the algorithm \ref{alg:ProjectionFreePolicyGradient}, and choosing $\eta=\alpha\pi_t(m)$ for every $m\neq m^*$.
\begin{lemma}\label{lemma:NoisyMAB:finite rec time}
For $\alpha>0$ such that $\alpha<\frac{\Delta_{min}}{\kr(m^*)-\Delta_{min}}$, we have that 
\[\sup\limits_{\pi\in \cS}\expect{\tau\given \pi_1=\pi}<\infty. \]
\end{lemma}
\begin{proof}
The proof here is for completeness. We first make note of the following useful result: For a sequence of positive real numbers $\{a_n \}_{n\geq 1}$ such that the following condition is met:
\[a(n+1)\leq a(n) - b.a(n)^2, \]
for some $b>0$, the following is always true:
\begin{equation*}\label{eq:useful derivative like inequality MAB}
    a_n\leq \frac{a_1}{1+bt}.
\end{equation*}
This inequality follows by rearranging and observing the $a_n$ is a non-increasing sequence. A complete proof can be found in eg. (\cite{Denisov2020RegretAO}, Appendix A.1). 
Returning to the proof of lemma, we proceed by showing that the sequence $1/{\pi_t(m^*)}-ct$ is a supermartingale for some $c>0$. Let $\Delta_{min}:=\Delta$ for ease of notation. Note that if the condition on $\alpha$ holds then there exists an $\epsilon>0$, such that $(1+\epsilon)(1+\alpha)< \kr^*/(\kr^*-\Delta)$, where $\kr^* :=\kr(m^*)$. We choose $c$ to be 
\[c:= \alpha. \frac{\kr^*}{1+\alpha} - \alpha(\kr^*-\Delta)(1+\epsilon) >0. \]
Next, let $x$ to be greater than $M$ and satisfying:
\[\frac{x}{x-\alpha M}\leq 1+\epsilon. \]
Let $\xi_x:=\min\{t\geq 1: \pi_t(m^*) > 1/x \}$. Since for $t=1,\ldots, \xi_x-1$, $m_*(t)\neq m^*$, we have $\pi_{t+1}(m^*) = (1+\alpha)\pi_t(m^*)$ w.p. $\pi_t(m^*)\kr^*$ and $\pi_{t+1}(m^*) = \pi_t(m^*)+ \alpha\pi_t(m^*)^2/\pi_t(m_*)^2$ w.p. $\pi_t(m_*)\kr_*(t)$, where $\kr_*(t):=\kr(m_*(t))$.  

Let $y(t):=1/{\pi_t(m^*)}$, then we observe by a short calculation that,
\begin{align*}
y(t+1)&=
    \begin{cases}
    y(t) - \frac{\alpha}{1+\alpha}y(t), & w.p. \frac{\kr^*}{y(t)}\\
    y(t) + \alpha \frac{y(t)}{\pi_t(m_*(t))y(t)-\alpha}. & w.p. \pi_t(m_*)\kr_*(t)\\
    y(t) & otherwise.
    \end{cases}
\end{align*}
We see that,
\begin{align*}
&\expect{y(t+1)\given H(t)}-y(t)\\
&= \frac{\kr^*}{y(t)}. (y(t) - \frac{\alpha}{1+\alpha}y(t)) + \pi_t(m_*)\kr_*(t).(y(t) + \alpha \frac{y(t)}{\pi_t(m_*(t))y(t)-\alpha}) - y(t)(\frac{\kr^*}{y(t)}+ \pi_t(m_*)\kr_*(t)) \\
&\leq \alpha(\kr^*-\Delta)(1+\epsilon)-\frac{\alpha \kr^*}{1+\alpha}=-c. 
\end{align*}
The inequality holds because $\kr_*(t)\leq \kr^*\Delta$ and that $\pi_t(m_*)>1/M$. By the Optional Stopping Theorem \cite{Durrett11probability:theory},
\[-c\expect{\xi_x\land t }\geq \expect{y(\xi_x \land t)-\expect{y(1)}} \geq -\frac{x}{1-\alpha}.  \]
The final inequality holds because $\pi_1(m^*)\geq \frac{1-\alpha}{x}$. 

Next, applying the monotone convergence theorem gives theta $\expect{\xi_x}\leq \frac{x}{c(1-\alpha)}$. Finally to show the result of lemma \ref{lemma:NoisyMAB:finite rec time}, we refer the reader to (Appendix A.2, \cite{Denisov2020RegretAO}), which follow from standard Markov chain arguments.
\end{proof}

Next we define an embedded Markov Chain $\{p(s), s\in \mathbb{Z}_+ \}$ as follows. First let $\sigma(k):= \min \left\{t\geq \tau(k): \pi_t(m^*) <\frac{1}{2} \right\}$ and $\tau(k):= \min \left\{t\geq \sigma(k-1): \pi_t (m^*) \geq \frac{1}{2} \right\}$.  Note that within the region $[\tau(k), \sigma(k))$, $\pi_t(m^*)\geq 1/2$ and in $[\sigma(k), \tau(k+1))$, $\pi_t(m^*)< 1/2$. We next analyze the rate at which $\pi_t(m^*)$ approaches 1. Define 
\begin{align*}
    p(s):= \pi_{t_s}(m^*) & \text{ where } & t_s=s+ \sum\limits_{i=0}^k(\tau(i+1)-\sigma(i))\\
    & \text{ for } & s\in \left[\sum\limits_{i=0}^k(\sigma(i)-\tau(i)),\sum\limits_{i=0}^{k+1}(\sigma(i)-\tau(i))  \right)
\end{align*}
Also let,
\[\sigma_s := \min\left\{t>0: \pi_{t+t_s}(m^*)>1/2     \right\} \]
and, \[\tau_s:=\min\left\{t>\sigma_s: \pi_{t+t_s}(m^*)\leq 1/2 \right\} \]
\begin{lemma}\label{lemma:noisy MAB expected bound on process }
The process $\{p(s)\}_{s\geq 1}$, is a submartingale. Further, $p(s)\to 1$, as $s\to \infty$. Finally,
\[\expect{p(s)} \geq 1- \frac{1}{1+ \alpha\frac{\Delta^2}{\left(\sum\limits_{m'\neq m^*}\Delta_{m'} \right)} s}. \]
\end{lemma}
\begin{proof}
We first observe that,
\begin{align*}
    p(s+1) &=
    \begin{cases}
    \pi_{t_s+1}(m^*) & if\, \pi_{t_s+1}(m^*) \geq 1/2\\
    \pi_{t_s+\tau+s}(m^*) & if\, \pi_{t_s+1}(m^*) < 1/2
    \end{cases}
\end{align*}
Since $\pi_{t_s+\tau_s}(m^*)\geq 1/2$, we have that,
\[p(s+1)\geq \pi_{t_s+1}(m^*) \text{ and } p(s)= \pi_{t_s}(m^*). \]
Since at times $t_s$, $\pi_{t_s}(m^*)>1/2$, we know that $m^*$ is the leading arm. Thus by the update step, for all $m\neq m^*$,
\[\pi_{t_s+1}(m) = \pi_{t_s}(m) +\alpha \pi_{t_s}(m)^2\left[\frac{\ind_mR_m(t_s)}{\pi_{t_s}(m)} - \frac{\ind_{m^*}R_{m^*}(t_s)}{\pi_{t_s}(m^*)} \right]. \]
Taking expectations both sides,
\[\expect{\pi_{t_s+1}(m)\given H(t_s)} - {\pi_{t_s}(m)} = \alpha\pi_{t_s}(m)^2(\kr_m-\kr_{m^*}) = -\alpha \Delta_m\pi_{t_s}(m)^2. \]
Summing over all $m\neq m^*$:
\[-\expect{\pi_{t_s+1}(m^*)\given H(t_s)} + {\pi_{t_s}(m^*)} = -\alpha\sum\limits_{m\neq m^*}\Delta_m\pi_{t_s}(m)^2. \]
By Jensen's inequality, 
\begin{align*}
\sum\limits_{m\neq m^*}\Delta_m\pi_{t_s}(m)^2 &= \left(\sum\limits_{m'\neq m^*}\Delta_{m'} \right) \sum\limits_{m\neq m^*}\frac{\Delta_m}{\left(\sum\limits_{m'\neq m^*}\Delta_{m'} \right)}\pi_{t_s}(m)^2 \\
&\geq \left(\sum\limits_{m'\neq m^*}\Delta_{m'} \right) \left(\sum\limits_{m\neq m^*}\frac{\Delta_m\pi_{t_s}(m)}{\left(\sum\limits_{m'\neq m^*}\Delta_{m'} \right)} \right)^2\\
&\geq \left(\sum\limits_{m'\neq m^*}\Delta_{m'} \right) \frac{\Delta^2\left(\sum\limits_{m\neq m^*}\pi_{t_s}(m)\right)^2}{\left(\sum\limits_{m'\neq m^*}\Delta_{m'} \right)^2}\\
&= \frac{\Delta^2\left(1-\pi_{t_s}(m^*)\right)^2}{\left(\sum\limits_{m'\neq m^*}\Delta_{m'} \right)}.
\end{align*}
Hence we get,
\[p(s) - \expect{p(s+1)\given H(t_s)} \leq  -\alpha \frac{\Delta^2\left(1-p(s)\right)^2}{\left(\sum\limits_{m'\neq m^*}\Delta_{m'} \right)} \implies \expect{p(s+1)\given H(t_s)}\geq p(s) +  \alpha \frac{\Delta^2\left(1-p(s)\right)^2}{\left(\sum\limits_{m'\neq m^*}\Delta_{m'} \right)}. \]
This implies immediately that $\{p(s)\}_{s\geq 1}$ is a submartingale.

Since, $\{p(s)\}$ is non-negative and bounded by 1, by Martingale Convergence Theorem, $\lim_{s\to \infty} p(s)$ exists. We will now show that the limit is 1.  Clearly, it is sufficient to show that $\limsup\limits_{s\to \infty} p(s)=1$. For $a>2$, let 
\[\phi_a:= \min\left\{s\geq 1: p(s)\geq \frac{a-1}{a} \right\}.  \]
As is shown in \cite{Denisov2020RegretAO}, it is sufficient to show $\phi_a<\infty$, with probability 1, because then one can define a sequence of stopping times for increasing $a$, each finite w.p. 1. which implies that $p(s)\to 1$. By the previous display, we have 
\[\expect{p(s+1)\given H(t_s)}- p(s) \geq  \alpha \frac{\Delta^2}{\left(\sum\limits_{m'\neq m^*}\Delta_{m'} \right)a^2} \]
as long as $p(s)\leq \frac{a-1}{a}$. Hence by applying Optional Stopping Theorem and rearranging we get,
\[\expect{\phi_a}\leq \lim_{s\to \infty} \expect{\phi_a \land s} \leq \frac{\left(\sum\limits_{m'\neq m^*}\Delta_{m'} \right)a^2}{\alpha \Delta}(1-\expect{p(1)}) <\infty. \]
Since $\phi_a$ is a non-negative random variable with finite expectation, $\phi_a<\infty a.s.$
Let $q(s)=1-p(s)$. We have :
\[\expect{q(s+1)} -\expect{q(s)} \leq -\alpha \frac{\Delta^2\left(q(s)\right)^2}{\left(\sum\limits_{m'\neq m^*}\Delta_{m'} \right)}. \]
By the useful result \ref{eq:useful derivative like inequality MAB}, we get,
\[\expect{q(s)} \leq \frac{\expect{q(1)}}{1+ \alpha\frac{\Delta^2\expect{q(1)}}{\left(\sum\limits_{m'\neq m^*}\Delta_{m'} \right)}s } \leq  \frac{1}{1+ \alpha\frac{\Delta^2}{\left(\sum\limits_{m'\neq m^*}\Delta_{m'} \right)} s}.\]
This completes the proof of the lemma.
\end{proof}
Finally we provide a lemma to tie the results above. We refer (Appendix A.5 \cite{Denisov2020RegretAO}) for the proof of this lemma. 
\begin{lemma}\label{lemma:technical lemm noisy MAB}
\[\sum\limits_{t\geq 1}\prob{\pi_t(m^*)<1/2} <\infty. \]
Also, with probability 1, $\pi_t(m^*)\to 1$, as $t\to \infty$.
\end{lemma}

\underline{Proof of regret bound:}
Since $\kr^*-\kr(m)\leq 1$, we have by the definition of regret (see eq \ref{eq:regret definition MAB})
\[\cR(T) = \expect{\frac{1}{1-\gamma}\sum\limits_{t=1}^T\left(\sm \pi^*(m)\kr_m - \pi_t(m)\kr_m \right)}.\]
Here we recall that $\pi^*=e_{m^*}$, we have:
\begin{align*}
\cR(T) &= \frac{1}{1-\gamma}\expect{\sum\limits_{t=1}^T\left(\sm (\pi^*(m)\kr_m - \pi_t(m)\kr_m) \right)}\\
&=\frac{1}{1-\gamma}\expect{\sum\limits_{m=1}^M\left(\sum\limits_{t=1}^T (\pi^*(m)\kr_m - \pi_t(m)\kr_m) \right)}\\
&=\frac{1}{1-\gamma}\expect{\sum\limits_{t=1}^T\left(\kr^* -\sm \pi_t(m)\kr_m \right)}\\
&=\frac{1}{1-\gamma}\expect{\left(\sum\limits_{t=1}^T\kr^* -\sum\limits_{t=1}^T\sm \pi_t(m)\kr_m \right)}\\
&=\frac{1}{1-\gamma}\expect{\left(\sum\limits_{t=1}^T\kr^*(1-\pi_t(m^*)) -\sum\limits_{t=1}^T\sum\limits_{m\neq m^*} \pi_t(m)\kr_m \right)}\\
&=\frac{1}{1-\gamma}\expect{\left(\sum\limits_{t=1}^T\sum\limits_{m\neq m^*}\kr^*\pi_t(m) -\sum\limits_{t=1}^T\sum\limits_{m\neq m^*} \pi_t(m)\kr_m \right)}\\
&= \frac{1}{1-\gamma} \sum\limits_{m\neq m^*} (\kr^*-\kr_m)\expect{\sum\limits_{t=1}^T \pi_t(m)}.
\end{align*}
Hence we have,
\begin{align*}
\cR(T) &= \frac{1}{1-\gamma} \sum\limits_{m\neq m^*} (\kr^*-\kr_m)\expect{\sum\limits_{t=1}^T \pi_t(m)}\\
&\leq \frac{1}{1-\gamma} \sum\limits_{m\neq m^*}\expect{\sum\limits_{t=1}^T \pi_t(m)}\\
&=\frac{1}{1-\gamma} \expect{\sum\limits_{t=1}^T (1-\pi_t(m^*))}\\
\end{align*}
We analyze the following  term:
\[\expect{\sum\limits_{t=1}^T  (1-\pi_t(m^*))} =\expect{\sum\limits_{t=1}^T  (1-\pi_t(m^*))\ind\{\pi_t(m^*)\geq 1/2 \}}+\expect{\sum\limits_{t=1}^T  (1-\pi_t(m^*))\ind\{\pi_t(m^*)<1/2 \}} \]
\[=\expect{\sum\limits_{t=1}^T  (1-\pi_t(m^*))\ind\{\pi_t(m^*)\geq 1/2 \}} +C_1.  \]
where, $C_1:=\sum\limits_{t=1}^\infty \prob{\pi_t(m^*)<1/2}<\infty$ by Lemma \ref{lemma:technical lemm noisy MAB}. Next we observe that,
\[\expect{\sum\limits_{t=1}^T  (1-\pi_t(m^*))\ind\{\pi_t(m^*)\geq 1/2 \}} = \expect{\sum\limits_{s=1}^T  q(s)\ind\{\pi_t(m^*)\geq 1/2 \}} \leq \expect{\sum\limits_{s=1}^T  q(s)} \]
\[= \sum\limits_{t=1}^T \frac{1}{1+ \alpha\frac{\Delta^2}{\left(\sum\limits_{m'\neq m^*}\Delta_{m'} \right)} s} \leq \sum\limits_{t=1}^T \frac{{\left(\sum\limits_{m'\neq m^*}\Delta_{m'} \right)}}{ \alpha{\Delta^2} s} \]
\[ \leq \frac{{\left(\sum\limits_{m'\neq m^*}\Delta_{m'} \right)}}{ \alpha{\Delta^2} }\log T. \]

Putting things together, we get,
\begin{align*}
\cR(T) &\leq \frac{1}{1-\gamma} \left( \frac{{\left(\sum\limits_{m'\neq m^*}\Delta_{m'} \right)}}{ \alpha{\Delta^2} }\log T + C_1 \right)\\
&=\frac{1}{1-\gamma} \left( \frac{{\left(\sum\limits_{m'\neq m^*}\Delta_{m'} \right)}}{ \alpha{\Delta^2} }\log T \right) + C .
\end{align*}
This completes the proof of Theorem \ref{thm:regret noisy gradient}.

\end{proof}

\section{Proofs for MDPs}

First we recall the policy gradient theorem.
\begin{theorem}[Policy Gradient Theorem \cite{Sutton2000}]\label{thm:policy gradient theorem}
\[\frac{\partial}{\partial \theta}V^{\pi_\theta} (\mu) = \frac{1}{1-\gamma}\sum\limits_{s\in \cS} d_\mu^{\pi_\theta}(s) \sa \frac{\partial \pi_\theta(a|s)}{\partial \theta}Q^{\pi_\theta}(s,a). \]
\end{theorem}
Let $s\in \cS$ and $m\in [m]$.
Let $\Q^{\pi_\theta}(s,m)\bydef \sum\limits_{a\in \cA}K_m(s,a)Q^{\pi_\theta}(s,a) $. Also let $\A(s,m) \bydef \Q(s,m)-V(s)$.
\begin{restatable}[Gradient Simplification]{lemma}{gradientofV}\label{lemma:Gradient simplification}
The softmax policy gradient with respect to the parameter $\theta\in \Real^M$ is $\frac{\partial}{\partial \theta_m}V^{\pi_{\theta}}(\mu) = \frac{1}{1-\gamma}\sum\limits_{s\in \cS}d_\mu^{\pi_\theta}(s) \pi_\theta(m)\tilde{A}(s,m)$,
where $\tilde{A}(s,m):=\Q(s,m)-V(s)$ and $\Q(s,m):= \sum\limits_{a\in \cA}K_m(s,a)Q^{\pi_\theta}(s,a) $, and $d_\mu^{\pi_\theta}(.)$ is the \textit{discounted state visitation measure} starting with an initial distribution $\mu$ and following policy $\pi_\theta$.
\end{restatable}
The interpretation of $\tilde{A}(s,m)$ is the advantage of following controller $m$ at state $s$ and then following the policy $\pi_\theta$ for all time versus following $\pi_\theta$ always.
As mentioned in section \ref{sec:PG theory}, we proceed by proving smoothness of the $V^{\pi}$ function over the space $\Real^M$.
\begin{proof}
From the policy gradient theorem \ref{thm:policy gradient theorem}, we have:
\begin{align*}
    \frac{\partial}{\partial \theta_{m'}} V^{\pi_\theta}(\mu) &= \frac{1}{1-\gamma}\sum\limits_{s\in \cS} d_\mu^{\pi_\theta}(s) \sa \frac{\partial \pi_{\theta_{m'}}(a|s)}{\partial \theta}Q^{\pi_\theta}(s,a)\\
    &=\frac{1}{1-\gamma}\sum\limits_{s\in \cS} d_\mu^{\pi_\theta}(s) \sa \frac{\partial }{\partial {\theta_{m'}}}\left(\sm \pi_\theta(m)K_m(s,a)\right)Q^{\pi_\theta}(s,a)\\
    &=\frac{1}{1-\gamma}\sum\limits_{s\in \cS} d_\mu^{\pi_\theta}(s) \sm \sa \left(\frac{\partial }{\partial {\theta_{m'}}} \pi_\theta(m)\right)K_m(s,a)Q(s,a)\\
    &=\frac{1}{1-\gamma}\sum\limits_{s\in \cS} d_\mu^{\pi_\theta}(s) \sa \pi_{m'}\left(K_{m'}(s,a) - \sm \pi_mK_m(s,a) \right)Q(s,a)\\
    &=\frac{1}{1-\gamma}\sum\limits_{s\in \cS} d_\mu^{\pi_\theta}(s)\pi_{m'} \sa \left(K_{m'}(s,a) - \sm \pi_mK_m(s,a) \right)Q(s,a)\\
    &=\frac{1}{1-\gamma}\sum\limits_{s\in \cS} d_\mu^{\pi_\theta}(s)\pi_{m'} \left[\sa K_{m'}(s,a)Q(s,a) -\sa\sm \pi_mK_m(s,a)Q(s,a) \right]\\
    &=\frac{1}{1-\gamma}\sum\limits_{s\in \cS} d_\mu^{\pi_\theta}(s)\pi_{m'} \left[\Q(s,m') -V(s) \right]\\
    &=\frac{1}{1-\gamma}\sum\limits_{s\in \cS} d_\mu^{\pi_\theta}(s)\pi_{m'} \A^{\pi_\theta}(s,m').
\end{align*}
\end{proof}
\allowdisplaybreaks
\begin{restatable} {lemma}{smoothnesslemmaMDP}\label{lemma:smoothness of V}
${V}^{\pi_{\theta}}\left(\mu\right)$ is $\frac{7\gamma^2+4\gamma+5}{2\left(1-\gamma\right)^2}$-smooth.
\end{restatable}

\begin{proof}
\allowdisplaybreaks
The proof uses ideas from \cite{Agarwal2020} and \cite{Mei2020}. Let $\theta_{\alpha} = \theta+\alpha u$, where $u\in \Real^M$, $\alpha\in \Real$. For any $s\in \cS$, 
\begin{align*}
\allowdisplaybreaks
    \sum\limits_a \abs{\frac{\partial\pi_{\theta_{\alpha}}(a| s)}{\partial \alpha}\Big|_{\alpha=0}} &= \sum\limits_a \abs{\innprod{\frac{\partial\pi_{\theta_{\alpha}}(a| s)}{\partial \theta_{\alpha}}\Big|_{\alpha=0}, \frac{\partial\theta_\alpha}{\partial\alpha}}} = \sum\limits_a \abs{\innprod{\frac{\partial\pi_{\theta_{\alpha}}(a| s)}{\partial \theta_{\alpha}}\Big|_{\alpha=0}, u}} \\
    &=\sum\limits_a\abs{\sum\limits_{m''=1}^M\sum\limits_{m=1}^M\pi_{\theta_{m''}}\left(\ind_{mm''}-\pi_{\theta_m}\right)K_m(s,a)u(m'') }\\
    &= \sum\limits_a\abs{\sum\limits_{m''=1}^M\pi_{\theta_{m''}}\left(K_{m''}(s,a)u(m'')-\sum\limits_{m=1}^M K_m(s,a)u(m'')\right) }\\
    &\leq \sum\limits_a\sum\limits_{m''=1}^M\pi_{\theta_{m''}}K_{m''}(s,a)\abs{u(m'')} + \sum\limits_a\sum\limits_{m''=1}^M\sum\limits_{m=1}^M\pi_{\theta_{m''}}\pi_{\theta_{m}}K_{m}(s,a)\abs{u(m'')} \\
    &= \sum\limits_{m''=1}^M\pi_{\theta_{m''}}\abs{u(m'')}\underbrace{\sum\limits_a K_{m''}(s,a)}_{=1} + \sum\limits_{m''=1}^M\sum\limits_{m=1}^M\pi_{\theta_{m''}}\pi_{\theta_{m}}\abs{u(m'')}\underbrace{\sum\limits_a K_{m}(s,a) }_{=1}\\
    &=\sum\limits_{m''=1}^M\pi_{\theta_{m''}}\abs{u(m'')}+ \sum\limits_{m''=1}^M\sum\limits_{m=1}^M\pi_{\theta_{m''}}\pi_{\theta_{m}}\abs{u(m'')}\\
    &=2\sum\limits_{m''=1}^M \pi_{\theta_{m''}}\abs{u(m'')} \leq 2\norm{u}_2.
\end{align*}
Next we bound the second derivative.
\allowdisplaybreaks
\[\sum\limits_a \abs{\frac{\partial^2 \pi_{\theta_{\alpha}}(a\given s)}{\partial\alpha^2}\given_{\alpha=0}} = \sum
\limits_a \abs{\innprod{\frac{\partial}{\partial\theta_\alpha}\frac{\partial \pi_{\theta_{\alpha}}(a\given s)}{\partial\alpha}\given_{\alpha=0},u}}=\sum
\limits_a \abs{\innprod{\frac{\partial^2 \pi_{\theta_{\alpha}}(a\given s)}{\partial\alpha^2}\given_{\alpha=0}u,u}}.\]
Let $H^{a,\theta}\bydef \frac{\partial^2 \pi_{\theta_{\alpha}}(a\given s)}{\partial \theta^2} \in \Real^{M\times M}$. We have,
\begin{align*}
\allowdisplaybreaks
    H^{a,\theta}_{i,j} &= \frac{\partial}{\partial \theta_j}\left(\sum\limits_{m=1}^{M} \pi_{\theta_i}\left(\ind_{mi}-\pi_{\theta_m}\right)K_m(s,a) \right)\\
    &= \frac{\partial}{\partial \theta_j}\left(\pi_{\theta_i}K_i(s,a)-\sum\limits_{m=1}^{M} \pi_{\theta_i}\pi_{\theta_m}K_m(s,a) \right)\\
    &=\pi_{\theta_j}(\ind_{ij}-\pi_{\theta_i})K_i(s,a) -\sum\limits_{m=1}^M K_m(s,a) \frac{\partial \pi_{\theta_i}\pi_{\theta_m}}{\partial\theta_j}\\
    &= \pi_j(\ind_{ij}-\pi_i)K_i(s,a) - \sum\limits_{m=1}^M K_m(s,a)\left(\pi_j(\ind_{ij}-\pi_i)\pi_m + \pi_i\pi_j(\ind_{mj}-\pi_m)\right)\\
    &= \pi_j\left( (\ind_{ij}-\pi_i)K_i(s,a) - \sum\limits_{m=1}^M \pi_m(\ind_{ij}-\pi_i)K_m(s,a) -\sum\limits_{m=1}^M\pi_i(\ind_{mj}-\pi_m)K_m(s,a)   \right).
\end{align*}

Plugging this into the second derivative, we get,
\begin{align*}
\allowdisplaybreaks
\begin{split}
    & \abs{\innprod{\frac{\partial^2}{\partial\theta^2}\pi_\theta(a|s)u,u}}\\
    &= \abs{\sum\limits_{j=1}^M\sum\limits_{i=1}^MH_{i,j}^{a,\theta}u_iu_j}\\
    &=\abs{\sum\limits_{j=1}^M\sum\limits_{i=1}^M  \pi_j\left( (\ind_{ij}-\pi_i)K_i(s,a) - \sum\limits_{m=1}^M \pi_m(\ind_{ij}-\pi_i)K_m(s,a) -\sum\limits_{m=1}^M\pi_i(\ind_{mj}-\pi_m)K_m(s,a)   \right)u_iu_j  }\\
    &=\Bigg|\sum\limits_{i=1}^M \pi_iK_i(s,a)u_i^2 - \si\sj\pi_i\pi_jK_i(s,a)u_iu_j - \si\sm\pi_i\pi_mK_m(s,a)u_i^2\\ &\qquad +\si\sj\sm \pi_i\pi_j\pi_m K_m(s,a)u_iu_j - \si\sj\pi_i\pi_j K_j(s,a)u_iu_j  \\
    &\qquad+\si\sj\sm \pi_i\pi_j\pi_m K_m(s,a)u_iu_j \Bigg|\\
    &= \Bigg| \si\pi_iK_i(s,a) u_i^2 -2\si\sj\pi_i\pi_jK_i(s,a)u_iu_j\\
    &\qquad - \si\sm\pi_i\pi_mK_m(s,a) u_i^2 + 2\si\sj\sm\pi_i\pi_j\pi_mK_m(s,a)u_iu_j\Bigg|\\
    &= \Bigg|\si\pi_iu_i^2\left(K_i(s,a) -\sm\pi_mK_m(s,a) \right) -2\si\pi_iu_i\sj\pi_ju_j\left( K_i(s,a) -\sm\pi_mK_m(s,a) \right)\Bigg|\\
    &\leq \si\pi_iu_i^2\underbrace{\abs{K_i(s,a) -\sm\pi_mK_m(s,a)}}_{\leq 1} +2 \si\pi_i\abs{u_i} \sj\pi_j\abs{u_j} \underbrace{\abs{K_i(s,a) -\sm\pi_mK_m(s,a)}}_{\leq 1}\\
    &\leq \norm{u}_2^2 +2 \si\pi_i\abs{u_i}\sj\pi_j\abs{u_j}\leq 3\norm{u}^2_2.
 \end{split}
\end{align*}

The rest of the proof is similar to \cite{Mei2020} and we include this for completeness. Define $P(\alpha)\in \Real^{S\times S}$, where $\forall (s,s'),$
\[\left[P(\alpha) \right]_{(s,s')} = \sa \pi_{\theta_\alpha} (a\given s).\tP(s'|s,a). \]
The derivative w.r.t. $\alpha$ is,
\[ \left[\frac{\partial}{\partial \alpha} P(\alpha)\Big|_{\alpha=0} \right]_{(s,s') } = \sa \left[\frac{\partial}{\partial \alpha}\pi_{\theta_\alpha} (a\given s)\Big|_{\alpha=0}\right].\tP(s'|s,a).\]
For any vector $x\in \Real^S$, 
\[\left[\frac{\partial}{\partial \alpha} P(\alpha)\Big|_{\alpha=0}x \right]_{(s) } = \sum\limits_{s'\in \cS}\sa \left[\frac{\partial}{\partial \alpha}\pi_{\theta_\alpha} (a\given s)\Big|_{\alpha=0}\right].\tP(s'|s,a). x(s'). \]

The $l_\infty$ norm can be upper-bounded as,
\begin{align*}
    \norm{\frac{\partial}{\partial \alpha} P(\alpha)\Big|_{\alpha=0}x}_\infty &= \max\limits_{s\in \cS}\abs{ \sum\limits_{s'\in \cS}\sa \left[\frac{\partial}{\partial \alpha}\pi_{\theta_\alpha} (a\given s)\Big|_{\alpha=0}\right].\tP(s'|s,a). x(s') }\\
    &\leq \max\limits_{s\in \cS} \sum\limits_{s'\in \cS}\sa \abs{\frac{\partial}{\partial \alpha}\pi_{\theta_\alpha} (a\given s)\Big|_{\alpha=0}}.\tP(s'|s,a). \norm{x}_\infty\\
    &\leq 2\norm{u}_2\norm{x}_\infty.
\end{align*}
Now we find the second derivative,
\begin{align*}
    \left[\frac{\partial^2P(\alpha)}{\partial\alpha^2} \Big|_{\alpha=0}\right]_{(s,s')} = \sa \left[\frac{\partial^2\pi_{\theta_\alpha}(a|s)}{\partial \alpha^2}\Big|_{\alpha=0} \right]\tP(s'|s,a)
\end{align*}
taking the $l_\infty$ norm,
\begin{align*}
\norm{\left[\frac{\partial^2P(\alpha)}{\partial\alpha^2} \Big|_{\alpha=0}\right]x}_{\infty} &= \max_s \abs{\sum\limits_{s'\in \cS}\sa \left[\frac{\partial^2\pi_{\theta_\alpha}(a|s)}{\partial \alpha^2}\Big|_{\alpha=0} \right]\tP(s'|s,a)x(s')}\\
&\leq \max_s \sum\limits_{s'\in \cS}\left[\abs{\frac{\partial^2\pi_{\theta_\alpha}(a|s)}{\partial \alpha^2}\Big|_{\alpha=0}} \right]\tP(s'|s,a)\norm{x}_\infty \leq 3\norm{u}_2\norm{x}_\infty.
\end{align*}
Next we observe that the value function of $\pi_{\theta_\alpha}:$
\[V^{\pi_{\theta_\alpha}}(s) = \underbrace{\sa \pi_{\theta_{\alpha}}(a|s)r(s,a)}_{r_{\theta_\alpha}} + \gamma \sa \pi_{\theta_{\alpha}}(a|s) \sum\limits_{s'\in \cS} \tP(s'|s,a)V^{\pi_{\theta_{\alpha}}}(s').  \]
In matrix form,
\begin{align*}
    V^{\pi_{\theta_{\alpha}}} = r_{\theta_{\alpha}} + \gamma P(\alpha)V^{\pi_{\theta_{\alpha}}}\\
    \implies \left(Id-\gamma P(\alpha) \right)V^{\pi_{\theta_{\alpha}}} = r_{\theta_{\alpha}}\\
    V^{\pi_{\theta_{\alpha}}} = \left(Id-\gamma P(\alpha) \right)^{-1}r_{\theta_{\alpha}}.
\end{align*}
Let $M(\alpha):=\left(Id-\gamma P(\alpha) \right)^{-1}= \sum\limits_{t=0}^\infty \gamma^t [P(\alpha)]^t $.
Also, observe that 
\[\mathbf{1} = \frac{1}{1-\gamma} \left(Id-\gamma P(\alpha) \right) \mathbf{1}\implies M(\alpha)\mathbf{1}=\frac{1}{1-\gamma}\mathbf{1}.\]
\[\implies \forall i \norm{[M(\alpha)]_{i,:}}_1 = \frac{1}{1-\gamma}\]
where $[M(\alpha)]_{i,:}$ is the $i^{th}$ row of $M(\alpha)$.
Hence for any vector $x\in \Real^S$, $\norm{M(\alpha)x}_\infty\leq \frac{1}{1-\gamma} \norm{x}_\infty.$

By assumption \ref{assumption:bounded reward}, we have $\norm{r_{\theta_\alpha}}_{\infty} = \max_s \abs{r_{\theta_\alpha}(s)} \leq 1$. Next we find the derivative of $r_{\theta_\alpha}$ w.r.t $\alpha$.
\begin{align*}
    \abs{\frac{\partial r_{\theta_\alpha}(s)}{\partial \alpha}} &= \abs{\left(\frac{\partial r_{\theta_\alpha}(s)}{\partial \theta_\alpha}\right)\transpose \frac{\partial \theta_\alpha}{\partial \alpha}}\\
    &\leq \abs{\sum\limits_{m''=1}^M\sm \sa \pi_{\theta_\alpha}(m'') (\ind_{mm''}-\pi_{\theta_\alpha}(m))K_m(s,a) r(s,a)u(m'')  }\\
    &=\abs{\sum\limits_{m''=1}^M \sa \pi_{\theta_\alpha}(m'')K_{m''}(s,a) r(s,a)u(m'') - \sum\limits_{m''=1}^M \sm\sa \pi_{\theta_\alpha}(m'')\pi_{\theta_\alpha}(m)K_m(s,a) r(s,a)u(m'')}\\
    &\leq \abs{\sum\limits_{m''=1}^M \sa \pi_{\theta_\alpha}(m'')K_{m''}(s,a) r(s,a) - \sum\limits_{m''=1}^M \sm\sa \pi_{\theta_\alpha}(m'')\pi_{\theta_\alpha}(m)K_m(s,a) r(s,a)}\norm{u}_\infty \leq \norm{u}_2.\\
\end{align*}
Similarly, we can calculate the upper-bound on second derivative,
\begin{align*}
    \norm{\frac{\partial r_{\theta_\alpha}}{\partial \alpha^2}}_\infty &= \max_s \abs{ \frac{\partial r_{\theta_\alpha}(s)}{\partial \alpha^2} }\\
    &=\max_s \abs{\left( \frac{\partial}{\partial \alpha} \left\{ \frac{\partial r_{\theta_\alpha}(s)}{\partial \alpha}\right\} \right)\transpose \frac{\partial \theta_\alpha}{\partial \alpha}}\\
    &= \max_s \abs{\left( \frac{\partial^2 r_{\theta_\alpha}(s)}{\partial \alpha^2}   \frac{\partial \theta_\alpha}{\partial \alpha}  \right)\transpose   \frac{\partial \theta_\alpha}{\partial \alpha} }
    &\leq 5/2 \norm{u}_2^2.
\end{align*}
Next, the derivative of the value function w.r.t $\alpha$ is given by,
\[\frac{\partial V^{\pi_{\theta_\alpha}}(s)}{\partial \alpha} = \gamma e_s\transpose M(\alpha) \frac{\partial P(\alpha)}{\partial\alpha}M(\alpha)r_{\theta_\alpha} + e_s\transpose M(\alpha) \frac{\partial r_{\theta_\alpha}}{\partial \alpha}.\]
And the second derivative,
\begin{align*}
\begin{split}
    \frac{\partial^2 V^{\pi_{\theta_\alpha}}(s)}{\partial \alpha^2} &= \underbrace{2\gamma^2e_s\transpose M(\alpha)\frac{\partial P(\alpha)}{\partial\alpha} M(\alpha)\frac{\partial P(\alpha)}{\partial\alpha}M(\alpha)r_{\theta_\alpha}}_{T1} + \underbrace{\gamma e_s\transpose M(\alpha) \frac{\partial^2 P(\alpha)}{\partial\alpha^2} M(\alpha)r_{\theta_\alpha}}_{T2}\\& + \underbrace{2\gamma  e_s\transpose M(\alpha)\frac{\partial P(\alpha)}{\partial\alpha} M(\alpha)\frac{\partial r_{\theta_\alpha}}{\partial \alpha}}_{T3}+ \underbrace{e_s\transpose M(\alpha)\frac{\partial^2 r_{\theta_\alpha}}{\partial\alpha^2}}_{T4}.
\end{split}
\end{align*}
We use the above derived bounds to bound each of the term in the above display. The calculations here are same as shown for Lemma 7 in \cite{Mei2020}, except for the particular values of the bounds. Hence we directly, mention the final bounds that we obtain and refer to \cite{Mei2020} for the detailed but elementary calculations.
\begin{align*}
    \abs{T1} &\leq \frac{4}{(1-\gamma)^3} \norm{u}_2^2\\
    \abs{T2} &\leq \frac{3}{(1-\gamma)^2} \norm{u}_2^2\\
    \abs{T3} &\leq \frac{2}{(1-\gamma)^2} \norm{u}_2^2\\
    \abs{T4} &\leq \frac{5/2}{(1-\gamma)} \norm{u}_2^2.
\end{align*}
Combining the above bounds we get,
\[\abs{\frac{\partial^2 V^{\pi_{\theta_\alpha}}(s)}{\partial \alpha^2}\Big|_{\alpha=0}}\leq \left(\frac{8\gamma^2}{(1-\gamma)^3} + \frac{3\gamma}{(1-\gamma)^2} + \frac{4\gamma}{(1-\gamma)^2} + \frac{5/2}{(1-\gamma)} \right)\norm{u}_2^2 \]
\[=\frac{7\gamma^2+4\gamma+5}{2(1-\gamma)^3}\norm{u}_2. \]
Finally, let $y\in \Real^M$ and fix a $\theta \in \Real^M$:
\begin{align*}
    \abs{y\transpose\frac{\partial^2 V^{\pi_{\theta}}(s)}{\partial\theta^2}y} &=\abs{\frac{y}{\norm{y}_2}\transpose\frac{\partial^2 V^{\pi_{\theta}}(s)}{\partial\theta^2}\frac{y}{\norm{y}_2}}.\norm{y}_2^2\\
    &\leq \max\limits_{\norm{u}_2=1}\abs{\innprod{\frac{\partial^2 V^{\pi_{\theta}}(s)}{\partial\theta^2}u,u }}.\norm{y}_2^2\\
    &= \max\limits_{\norm{u}_2=1}\abs{\innprod{\frac{\partial^2 V^{\pi_{\theta_\alpha}}(s)}{\partial\theta_\alpha^2}\Big|_{\alpha=0}\frac{\partial\theta_\alpha}{\partial \alpha},\frac{\partial\theta_\alpha}{\partial \alpha} }}.\norm{y}_2^2\\
    &= \max\limits_{\norm{u}_2=1}\abs{\frac{\partial^2V^{\pi_{\theta_\alpha}}(s)}{\partial\alpha^2}\Big|_{\alpha=0}}.\norm{y}_2^2\\
    &\leq \frac{7\gamma^2+4\gamma+5}{2(1-\gamma)^3}\norm{y}_2^2.
\end{align*}
Let $\theta_\xi:= \theta + \xi(\theta'-\theta)$ where $\xi\in [0,1]$. By Taylor's theorem $\forall s,\theta,\theta'$,
\[\abs{V^{\pi_{\theta'}}(s) -V^{\pi_{\theta}}(s) -\innprod{\frac{\partial V^{\pi_{\theta}}(s)}{\partial \theta} } } = \frac{1}{2}.\abs{(\theta'-\theta)\transpose \frac{\partial^2V^{\pi_{\theta_\xi}}(s) }{\partial \theta_\xi^2}(\theta'-\theta) } \]
\[\leq \frac{7\gamma^2+4\gamma+5}{4(1-\gamma)^3}\norm{\theta'-\theta}_2^2. \]
Since $V^{\pi_\theta}(s)$ is $\frac{7\gamma^2+4\gamma+5}{2(1-\gamma)^3}$ smooth for every $s$, $V^{\pi_\theta}(\mu)$ is also $\frac{7\gamma^2+4\gamma+5}{2(1-\gamma)^3}-$ smooth.
\end{proof}

\allowdisplaybreaks
\begin{lemma}[Value Difference Lemma-1]\label{lemma:value diffence lemma}
For any two policies $\pi$ and $\pi'$, and for any state $s\in \cS$, the following is true.
\[V^{\pi'}(s)-V^{\pi}(s) = \frac{1}{1-\gamma} \sum\limits_{s'\in \cS}d_s^{\pi'}(s')\sm \pi'_m \tilde{A}(s',m).\]
\end{lemma}
\begin{proof}
\begin{align*}
\allowdisplaybreaks
   V^{\pi'}(s) - V^{\pi}(s) &= \sm\pi_m'\Q'(s,m) - \sm \pi_m \Q(s,m)\\
   &=\sm \pi_m'\left(\Q'(s,m)-\Q(s,m) \right) +\sm (\pi_m'-\pi_m)\Q(s,m)\\
   &=\sm(\pi_m'-\pi_m)\Q(s,m) +\underbrace{\sm\pi_m'\sum\limits_{a\in \cA} K_m(s,a)}_{=\sum\limits_{a\in \cA}\pi_\theta(a|s)}\sum\limits_{s'\in \cS} \tP(s'|s,a)\left[V^{\pi'}(s')-V^{\pi}(s') \right]\\
   &=\frac{1}{1-\gamma}\sum\limits_{s'\in \cS} d_s^{\pi'}(s') \sum\limits_{m'=1}^M (\pi'_{m'}-\pi_{m'})\Q(s',m')\\
   &=\frac{1}{1-\gamma}\sum\limits_{s'\in \cS} d_s^{\pi'}(s') \sum\limits_{m'=1}^M \pi'_{m'}(\Q{s',m'}-V(s'))\\
   &=\frac{1}{1-\gamma}\sum\limits_{s'\in \cS} d_s^{\pi'}(s') \sum\limits_{m'=1}^M \pi'_{m'}\tilde{A}(s',m').
\end{align*}

\end{proof}
\begin{lemma}(Value Difference Lemma-2)\label{lemma:value diffrence lemma-2}
For any two policies $\pi$ and $\pi'$ and state $s\in \cS$, the following is true.
\[V^{\pi'}(s)-V^{\pi}(s) = \frac{1}{1-\gamma} \sum\limits_{s'\in \cS} d_s^{\pi}(s')\sm (\pi_m'-\pi_m)\Q^{\pi'}(s',m). \]
\end{lemma}

\begin{proof}
We will use $\Q$ for $\Q^{\pi}$ and $\Q'$ for $\Q^{\pi'}$ as a shorthand.
\begin{align*}
\begin{split}
    V^{\pi'}(s)-V^{\pi}(s) &= \sm \pi_m'\Q'(s,m) - \sm \pi_m\Q(s,m)\\
    &=\sm(\pi_m'-\pi_m)\Q'(s,m) + \sm \pi_m(\Q'(s,m)-\Q(s,m))\\
    &= \sm(\pi_m'-\pi_m)\Q'(s,m) + \\ & \gamma \sm \pi_m\left(\sum\limits_{a\in \cA} K_m(s,a)\sum\limits_{s'\in \cS} \tP(s'|s,a)V'(s') - \sum\limits_{a\in \cA} K_m(s,a)\sum\limits_{s'\in \cS} \tP(s'|s,a)V(s') \right)\\
    &=\sm(\pi_m'-\pi_m)\Q'(s,m) + \gamma\sum\limits_{a\in \cA} \pi_\theta(a|s) \sum\limits_{s'\in \cS} \tP(s'|s,a) \left[V'(s)-V(s') \right]\\
    &=\frac{1}{1-\gamma} \sum\limits_{s'\in \cS} d_s^{\pi}(s') \sm (\pi'_m-\pi_m) \Q'(s',m).
\end{split}
\end{align*}

\end{proof}

\begin{assumption}\label{assumption:bounded reward}
The reward $r(s,a)\in [0,1]$, for all pairs $(s,a)\in \cS\times \cA$.
\end{assumption}
\begin{assumption}\label{assumption:positivity of advantage}
Let $\pi^*\bydef \argmax\limits_{\pi\in \cP_M} V^\pi(s_0)$. We make the following assumption.
\[\mathbb{E}_{m\sim \pi^*} \left[Q^{\pi_\theta}(s,m) \right] -V^{\pi_\theta}(s) \geq 0, \forall s\in \cS,\forall \pi_\theta \in \Pi. \]
\end{assumption}

Let the best controller be a point in the $M-simplex$, i.e., $K^* \bydef \sm\pi^*_mK_m$.
\nonuniformLE*
\begin{proof}
\allowdisplaybreaks
\begin{align*}
    \norm{\frac{\partial}{\partial\theta}V^{\pi_\theta}(\mu)}_2 &=\left(\sm \left(\frac{\partial V^{\pi_\theta}(\mu)}{\partial\theta_m}\right)^2  \right)^{1/2}\\
    &\geq \frac{1}{\sqrt{M}} \sm \abs{\frac{\partial V^{\pi_\theta}(\mu)}{\partial\theta_m}} \text{      \quad\quad (Cauchy-Schwarz)}\\
    &=\frac{1}{\sqrt{M}} \sm \frac{1}{1-\gamma} \abs{\sum\limits_{s\in \cS} d_\mu^{\pi_\theta}(s)\pi_m\tilde{A}(s,m)} \text{\quad \quad Lemma \ref{lemma:Gradient simplification}}\\
    &\geq\frac{1}{\sqrt{M}} \sm \frac{\pi_m^*\pi_m}{1-\gamma} \abs{\sum\limits_{s\in \cS} d_\mu^{\pi_\theta}(s)\tilde{A}(s,m)}\\
    &\geq \left(\min\limits_{m:\pi^*_{\theta_{m}}>0} \pi_{\theta_m} \right) \frac{1}{\sqrt{M}} \sm \frac{\pi_m^*}{1-\gamma} \abs{\sum\limits_{s\in \cS} d_\mu^{\pi_\theta}(s)\tilde{A}(s,m)}\\
    &\geq \left(\min\limits_{m:\pi^*_{\theta_{m}}>0} \pi_{\theta_m} \right) \frac{1}{\sqrt{M}} \abs{\sm \frac{\pi_m^*}{1-\gamma} \sum\limits_{s\in \cS} d_\mu^{\pi_\theta}(s)\tilde{A}(s,m)}\\
    &= \left(\min\limits_{m:\pi^*_{\theta_{m}}>0} \pi_{\theta_m} \right) \abs{\frac{1}{\sqrt{M}} \sum\limits_{s\in \cS} d_\mu^{\pi_\theta}(s){\sm \frac{\pi_m^*}{1-\gamma} \tilde{A}(s,m)}}\\
    &= \left(\min\limits_{m:\pi^*_{\theta_{m}}>0} \pi_{\theta_m} \right) \frac{1}{\sqrt{M}} \sum\limits_{s\in \cS} d_\mu^{\pi_\theta}(s){\sm \frac{\pi_m^*}{1-\gamma} \tilde{A}(s,m)} \text{\quad \quad Assumption \ref{assumption:positivity of advantage}}\\
    &\geq \frac{1}{\sqrt{M}}\frac{1}{1-\gamma}\left(\min\limits_{m:\pi^*_{\theta_{m}}>0} \pi_{\theta_m} \right) \norm{\frac{d_{\rho}^{\pi^*}}{d_{\mu}^{\pi_\theta}}}_{\infty}^{-1} \sum\limits_{s\in \cS} d_{\rho}^*(s) \sm \pi_m^*\tilde{A}(s,m)\\
    &=\frac{1}{\sqrt{M}}\left(\min\limits_{m:\pi^*_{\theta_{m}}>0} \pi_{\theta_m} \right) \norm{\frac{d_{\rho}^{\pi^*}}{d_{\mu}^{\pi_\theta}}}_{\infty}^{-1} \left[V^*(\rho) -V^{\pi_\theta}(\rho) \right] \text{\quad \quad Lemma \ref{lemma:value diffence lemma}}. 
\end{align*}
\end{proof}
\subsection{Proof of the Theorem \ref{thm:convergence of policy gradient}}
\maintheorem*
Let $\beta:=\frac{7\gamma^2+4\gamma+5}{\left(1-\gamma\right)^2}$. We have that,
\begin{align*}
    V^*(\rho) -V^{\pi_\theta}(\rho) &= \frac{1}{1-\gamma} \sum\limits_{s\in \cS} d_\rho^{\pi_\theta}(s) \sm (\pi^*_m-\pi_m)\Q^{\pi^*}(s,m) \text{$\qquad$ (Lemma \ref{lemma:value diffrence lemma-2})}\\
    &= \frac{1}{1-\gamma} \sum\limits_{s\in \cS} \frac{d_\rho^{\pi_\theta}(s)}{d_\mu^{\pi_\theta}(s)}d_\mu^{\pi_\theta}(s) \sm (\pi^*_m-\pi_m)\Q^{\pi^*}(s,m) \\
    &\leq \frac{1}{1-\gamma} \norm{\frac{1}{d_\mu^{\pi_\theta}}}_{\infty} \sum\limits_{s\in \cS} \sm (\pi^*_m-\pi_m)\Q^{\pi^*}(s,m) \\
    &\leq \frac{1}{(1-\gamma)^2} \norm{\frac{1}{\mu}}_\infty \sum\limits_{s\in \cS} \sm (\pi^*_m-\pi_m)\Q^{\pi^*}(s,m) \\
    &=\frac{1}{(1-\gamma)} \norm{\frac{1}{\mu}}_\infty\left[V^*(\mu) -V^{\pi_\theta}(\mu) \right] \text{$\qquad$ (Lemma \ref{lemma:value diffrence lemma-2})}.
\end{align*}
Let $\delta_t\bydef V^*(\mu) -V^{\pi_{\theta_t}}(\mu)$.
\begin{align*}
    \delta_{t+1}-\delta_t &= V^{\pi_{\theta_{t}}}(\mu)-V^{\pi_{\theta_{t+1}}}(\mu) \text{\qquad (Lemma \ref{lemma:smoothness of V})}\\
    &\leq -\frac{1}{2\beta} \norm{\frac{\partial}{\partial\theta}V^{\pi_{\theta_{t}}}(\mu)}^2_2 \text{\qquad (Lemma \ref{lemma:gradient ascent lemma}  )}\\
    &\leq -\frac{1}{2\beta} \frac{1}{{M}}\left(\min\limits_{m:\pi^*_{\theta_{m}}>0} \pi_{\theta_m} \right)^2 \norm{\frac{d_{\rho}^{\pi^*}}{d_{\mu}^{\pi_\theta}}}_{\infty}^{-2}  \delta_t^2 \text{\qquad (Lemma \ref{lemma:nonuniform lojaseiwicz inequality})}\\
    &\leq -\frac{1}{2\beta} \left(1-\gamma\right)^2 \frac{1}{{M}}\left(\min\limits_{m:\pi^*_{\theta_{m}}>0} \pi_{\theta_m} \right)^2 \norm{\frac{d_{\rho}^{\pi^*}}{d_{\mu}^{\pi_\theta}}}_{\infty}^{-2}  \delta_t^2\\
    &\leq -\frac{1}{2\beta} \left(1-\gamma\right)^2 \frac{1}{{M}}\left(\min\limits_{1\leq s\leq t}\min\limits_{m:\pi^*_{\theta_{m}}>0} \pi_{\theta_m} \right)^2 \norm{\frac{d_{\rho}^{\pi^*}}{d_{\mu}^{\pi_\theta}}}_{\infty}^{-2}  \delta_t^2\\\\
    &= -\frac{1}{2\beta} \frac{1}{M}\left(1-\gamma\right)^2 \norm{\frac{d_\mu^{\pi^*}}{\mu}}_{\infty}^{-2}c_t^2 \delta_t^2,\\
\end{align*}   
where $c_t\bydef \min\limits_{1\leq s\leq t} \min\limits_{m:\pi^*_m>0}\pi_{\theta_s}(m)$. 
Hence we have that,
\begin{equation}\label{eq:induction step}
\delta_{t+1} \leq \delta_t - \frac{1}{2\beta} \frac{\left(1-\gamma\right)^2}{M} \norm{\frac{d_\mu^{\pi^*}}{\mu}}_{\infty}^{-2}c_t^2 \delta_t^2.
\end{equation}
The rest of the proof follows from a induction argument over $t\geq1$.

\underline{Base case:} Since $\delta_t\leq \frac{1}{1-\gamma}$, and $c_t \in (0,1)$, the result holds for all $t\leq \frac{2\beta M}{(1-\gamma)}\norm{\frac{d_\mu^{\pi^*}}{\mu}}_{\infty}^2.$ 

For ease of notation, let $\phi_t\bydef \frac{2\beta M}{c_t^2(1-\gamma)^2}\norm{\frac{d_\mu^{\pi^*}}{\mu}}_{\infty}^2$. We need to show that $\delta_t\leq \frac{\phi_t}{ t}$, for all $t\geq 1$.

\underline{Induction step:} Fix a $t\geq 2$, assume $\delta_t \leq \frac{\phi_t}{t}$.

Let $g:\Real\to \Real$ be a function defined as $g(x) = x-\frac{1}{\phi_t}x^2$. One can verify easily that $g$ is monotonically increasing in $\left[ 0, \frac{\phi_t}{2}\right]$. Next with equation \ref{eq:induction step}, we have
\begin{align*}
    \delta_{t+1} &\leq \delta_t -\frac{1}{\phi_t} \delta_t^2\\
    &= g(\delta_t)\\
    &\leq g(\frac{\phi_t}{ t})\\
    &\leq \frac{\phi_t}{t} - \frac{\phi_t}{t^2}\\
    &= \phi_t\left(\frac{1}{t}-\frac{1}{t^2} \right)\\
    &\leq \phi_t\left(\frac{1}{t+1}\right)\\
    &\leq \phi_{t+1}\left(\frac{1}{t+1}\right).
\end{align*}
where the last step follows from the fact that $c_{t+1}\leq c_t$ (infimum over a larger set does not increase the value). This completes the proof. 

\begin{lemma}\label{lemma:gradient ascent lemma}
Let $f:\Real^M\to \Real$ be $\beta-$smooth. Then gradient ascent with learning rate $\frac{1}{\beta}$ guarantees, for all $x,x'\in \Real^M$:
\[f(x)-f(x') \leq -\frac{1}{2\beta}\norm{\frac{df(x)}{dx}}_2^2. \]
\end{lemma}
\begin{proof}
\begin{align*}
    f(x) -f(x') &\leq -\innprod{\frac{\partial f(x)}{\partial x}} + \frac{\beta}{2}.\norm{x'-x}_2^2\\
    &= \frac{1}{\beta} \norm{\frac{df(x)}{dx}}_2^2 + \frac{\beta}{2}\frac{1}{\beta^2}\norm{\frac{df(x)}{dx}}_2^2\\
    &= -\frac{1}{2\beta}\norm{\frac{df(x)}{dx}}_2^2.
\end{align*}
\end{proof}

\section{Proofs for (Natural) Actor-critic based improper learning}
We will begin with some useful lemmas.
\begin{lemma}\label{lemma:value of Lpsi}
For any $\theta,\theta\in \Real^M$, we have $\norm{\psi_\theta(m)-\psi_{\theta'}(m)}_2\leq \norm{\theta-\theta'}_2$.
\end{lemma}
\begin{proof}
\allowdisplaybreaks
Recall, $\psi_{\theta}(m):= \nabla_\theta \log \pi_\theta(m)$. Fix $m'\in [M]$, 
\begin{align*}
    \frac{\partial\log\pi_\theta(m)}{\partial \theta_{m'}} &= \frac{\partial\log\left( \frac{e^{\theta_m}}{\sum\limits_{j=1}^M e^{\theta_j}}  \right)}{\partial \theta_{m'}}\\
    &= \frac{\partial}{\partial\theta_{m'}} \left(\theta_m-\log\left(\sum\limits_{j=1}^M e^{\theta_j}  \right) \right)\\
    &= \mathbbm{1}\{m'=m\} - \frac{e^{\theta_{m'}}}{\sum\limits_{j=1}^M e^{\theta_j} } \\
    &= \mathbbm{1}\{m'=m\} - \pi_{\theta}(m').
\end{align*}
\begin{align*}
    \norm{\psi_\theta(m)-\psi_{\theta'}(m)}_2\leq \norm{\theta-\theta'}_2 &=\norm{\nabla_\theta \log \pi_{\theta}(m)-\nabla_\theta \log \pi_{\theta'}(m)}_2\\
    &=\norm{\pi_\theta(.) - \pi_{\theta'}(.) }_2\\
    &\leq^{(*)} \norm{\theta-\theta'}_2. 
\end{align*}
Here (*) follows from the fact that the softmax function is 1-Lipschitz \cite{softmax}.
\end{proof}
\begin{lemma}\label{lemma:value of Cpsi}
For all $m\in [M]$ and $\theta\in \Real^M$,  $\norm{\psi_\theta(m)}_2 \leq \sqrt{2} $.
\end{lemma}
\begin{proof}
Proof follows by noticing that $\norm{\psi_\theta(m)}_2=\norm{\nabla_\theta \log \pi_\theta(m)}_2\leq \sqrt{2}$, where the last inequality follows because the 2-norm of a probability vector is bounded by 1. 
\end{proof}

\begin{lemma}\label{lemma:value of Cpi}
For all $\theta,\theta'\in \Real^M$, $\norm{\pi_\theta(.)-\pi_{\theta'}(.)}_{TV}\leq \frac{\sqrt{M}}{2}\norm{\theta-\theta'}_2$.
\end{lemma}
\begin{proof}
\begin{align*}
    \norm{\pi_\theta(.)-\pi_{\theta'}(.)}_{TV} &= \frac{1}{2}\norm{\pi_\theta(.)-\pi_{\theta'}(.)}_{1}\\
    &\leq \frac{\sqrt{M}}{2}\norm{\pi_\theta(.)-\pi_{\theta'}(.)}_{2}.
\end{align*}
The inequality follows from relation between 1-norm and 2-norm.
\end{proof}
\begin{proposition}\label{prop:restatement of Prop 1 in AC paper}
For any $\theta,\theta'\in \Real^M$, 
\[ \nabla V(\theta) - \nabla V(\theta')\leq \sqrt{M}L_V \norm{\theta-\theta'}_2  \]
where $L_V=\frac{2\sqrt{2}C_{\kappa \xi} +1}{1-\gamma}$, and $C_{\kappa \xi}=\left( 1+\left\lceil{\log_\xi\frac{1}{\kappa}}\right\rceil +\frac{1}{1-\xi}  \right)$.
\end{proposition}

\begin{proof}
We follow the same steps as in Proposition 1 in \cite{Improving-AC-NAC} along with Lemmas~\ref{lemma:value of Cpi},\ref{lemma:value of Cpsi},\ref{lemma:value of Lpsi} and that the maximum reward is bounded by 1.
\end{proof}

We will now restate a useful result from \cite{Improving-AC-NAC}, about the convergence of the critic parameter $w_t$ to the equilibrium point $w^*$ of the underlying ODE, applied to our setting.
\begin{proposition}\label{prop: restatement of Lemma 2 in ac paper}
Suppose assumptions`\ref{assumption:bound on w} and ~\ref{assumption:ergodicity} hold. Then is $\beta\leq \min\left\{\frac{\Gamma_L}{16},\frac{8}{\Gamma_L}  \right\}$ and $H\geq \left( \frac{4}{\Gamma_L}+2\alpha \right)\left[\frac{1536[1+(\kappa-1)\xi]}{(1-\xi)\Gamma_L}  \right]$. We have
\[ \expect{\norm{w_{T_c}-w^*}_2^2}\leq \left(1-\frac{\Gamma_L}{16}\alpha  \right)^{T_c}\norm{w_0-w^*}_2^2 + \left(\frac{4}{\Gamma_L}+2\alpha \right)\frac{1536(1+R_w^2)[1+(\kappa-1)\xi]}{(1-\xi)H}.   \]

If we further let $T_c\geq \frac{16}{\Gamma_L\alpha}\log \frac{2\norm{w_0-w^*}^2_2}{\epsilon}$ and $H\geq \left(\frac{4}{\Gamma_L}+2\alpha \right)\frac{3072(R_w^2+1)[1+(\kappa-1)\xi]}{(1-\xi)\Gamma_L\epsilon}$, then we have $\expect{\norm{w_{T_c}-w^*}_2^2}\leq \epsilon$ with total sample complexity given by $T_cH=\mathcal{O}\left(\frac{1}{\alpha\epsilon}\log\frac{1}{\epsilon} \right)$.
\end{proposition}
\begin{proof}
Proof follows along the similar lines as in Thm. 4 in \citet{Improving-AC-NAC} and by using $\norm{\phi(s)(\gamma\phi(s')-\phi(s))^\top}_F\leq (1+\gamma)\leq 2$ and assuming $\norm{\phi(s)}_2\leq 1$ for all $s,s'\in \cS$.
\end{proof}

\subsection{Actor-critic based improper learning}
\begin{proof}[Proof of Theorem~\ref{thm:AC main theorem}]
Let $v_t(w):=\frac{1}{B}\sum\limits_{i=0}^{B-1}\cE(s_{t,i},m_{t,i},s_{t,i+1})\psi_{\theta_t}(m_{t,i})$ and $A_w(s,m):=\mathbb{E}_{\bar{P}}\left[\cE(s,m,s')|(s,m) \right]$ and $g(w,\theta):=\mathbb{E}_{\nu_\theta}[A_w(s,m)\psi_\theta(m)]$ for all $\theta\in \Real^M,w\in\Real^d,s\in \cS, m\in [M]$. 
Using Prop~\ref{prop:restatement of Prop 1 in AC paper} we get,
\begin{align*}
\begin{split}
    V(\theta_{t+1}) &\geq V(\theta_t) + \innprod{\nabla_{\theta}V(\theta_t), \theta_{t+1}-\theta_t} - \frac{\sqrt{M}L_V}{2}\norm{\theta_{t+1}-\theta_t}_2^2\\
    &=V(\theta_t) + \alpha\innprod{\nabla_{\theta}V(\theta_t), v_t(w_t)-\nabla_{\theta}V(\theta_t)+\nabla_{\theta}V(\theta_t)} - \frac{\sqrt{M}L_V\alpha^2}{2}\norm{v_t(w_t)}_2^2\\
    &=V(\theta_t) +\alpha\norm{\nabla_{\theta}V(\theta_t)}_2^2 \\ &+\alpha\innprod{\nabla_{\theta}V(\theta_t), v_t(w_t)-\nabla_{\theta}V(\theta_t)}- \frac{\sqrt{M}L_V\alpha^2}{2}\norm{v_t(w_t)}_2^2\\
    &\geq V(\theta_t) +\left(\frac{1}{2}\alpha- {\sqrt{M}L_V\alpha^2} \right)\norm{\nabla_{\theta}V(\theta_t)}_2^2  -\left(\frac{1}{2}\alpha+ {\sqrt{M}L_V\alpha^2} \right)\norm{v_t(w_t)-\nabla_{\theta}V(\theta_t)}_2^2
    \end{split}
\end{align*}
Taking expectations and rearranging, we have
\begin{align*}
    &\left(\frac{1}{2}\alpha- {\sqrt{M}L_V\alpha^2} \right)\expect{\norm{\nabla_{\theta}V(\theta_t)}_2^2|\cF_t}\\
    &\leq \expect{V(\theta_{t+1})|\cF_t} -V(\theta_t) + \left(\frac{1}{2}\alpha+ {\sqrt{M}L_V\alpha^2} \right)\expect{\norm{v_t(w_t)-\nabla_{\theta}V(\theta_t)}_2^2|\cF_t}.
\end{align*}
Next we will upperbound $\expect{\norm{v_t(w_t)-\nabla_{\theta}V(\theta_t)}_2^2|\cF_t}$.
\begin{align*}
    & {\norm{v_t(w_t)-\nabla_{\theta}V(\theta_t)}_2^2}\\
    &\leq 3\norm{v_t(w_t) -v_t(w^*_{\theta_t}) }_2^2 + 3\norm{v_t(w^*_{\theta_t})-g(w^*_{\theta_t})}_2^2 + 3\norm{g(w^*_{\theta_t})-\nabla_{\theta}V(\theta_t)}_2^2.
\end{align*}
\begin{align*}
    \norm{v_t(w_t) -v_t(w^*_{\theta_t}) }_2^2 &= \norm{\frac{1}{B} \sum\limits_{i=0}^{B-1} [\cE_{w_t}(s_{t,i}, m_{t,i}, s_{t, i+1})-\cE_{w^*_{\theta_t}}(s_{t,i}, m_{t,i}, s_{t, i+1})]\psi(m_{t,i})}_2^2 \\
    &\leq \frac{1}{B}\sum\limits_{i=0}^{B-1}\norm{[\cE_{w_t}(s_{t,i}, m_{t,i}, s_{t, i+1})-\cE_{w^*_{\theta_t}}(s_{t,i}, m_{t,i}, s_{t, i+1})]\psi(m_{t,i})}_2^2\\
    &\leq \frac{2}{B}\sum\limits_{i=0}^{B-1}\norm{[\cE_{w_t}(s_{t,i}, m_{t,i}, s_{t, i+1})-\cE_{w^*_{\theta_t}}(s_{t,i}, m_{t,i}, s_{t, i+1})]}_2^2\\
    &=\frac{2}{B}\sum\limits_{i=0}^{B-1} \norm{(\gamma \phi(s_{t,i+1})-\phi(s_{t,i}))^\top(w_t-w^*_{\theta_t})}_2^2\\
    &\leq \frac{8}{B}\sum\limits_{i=0}^{B-1} \norm{(w_t-w^*_{\theta_t})}_2^2=8 \norm{(w_t-w^*_{\theta_t})}_2^2.
\end{align*}
Next we have,
\begin{align*}
    \norm{g(w^*_{\theta_t})-\nabla_{\theta}V(\theta_t)}_2^2 &=\norm{\mathbb{E}_{\nu_{\theta_t}}[A_{w^*_{\theta_t}}(s,m)\psi_{\theta_t}(m)]-\mathbb{E}_{\nu_{\theta_t}}[A_{\pi_{\theta_t}}(s,m)\psi_{\theta_t}(m)]  }_2^2\\
    &\leq 2 \mathbb{E}_{\nu_{\theta_t}}\norm{A_{w^*_{\theta_t}}(s,m)-A_{\pi_{\theta_t}}(s,m)}_2^2\\
    &=2\mathbb{E}_{\nu_{\theta_t}}\left[ | \gamma\expect{V_{w^*_{\theta_t}}(s')-V_{\pi_{{\theta_t}}}(s')|s,m}+ +V_{\pi_{{\theta_t}}}(s)-V_{w^*_{\theta_t}}(s) |^2 \right]\\
    &\leq  8 \Delta_{critic}.
\end{align*}

Finally we bound the last term $\norm{v_t(w^*_{\theta_t})-g(w^*_{\theta_t})}_2^2$ by using Assumption~\ref{assumption:ergodicity} we have,
\begin{align*}
    \norm{v_t(w^*_{\theta_t})-g(w^*_{\theta_t})}_2^2 &\leq \expect{\norm{\frac{1}{B}\sum\limits_{i=0}^{B-1}\cE_{w^*_{\theta_t}}(s_{t,i}, m_{t,i}, s_{t,i+1})\psi_{\theta_t}(m_{t,i}) - \mathbb{E}_{\nu_{\theta_t}}[A_{w^*_{\theta_t}}(s,m)\psi_{\theta_t}(m)]}_2^2 |\cF_t}. 
\end{align*}
We will now proceed in the similar manner as in \cite{Improving-AC-NAC} (eq 24 to eq 26), and using Lemma~\ref{lemma:value of Cpsi}, we have
\begin{align*}
    \expect{\norm{v_t(w^*_{\theta_t})-g(w^*_{\theta_t})}_2^2|\cF_t} &\leq \frac{32(1+R_w)^2[1+(\kappa-1)\xi]}{B(1-\xi)}.
\end{align*}
Putting things back we have,
\begin{align*}
    \expect{{\norm{v_t(w_t)-\nabla_{\theta}V(\theta_t)}_2^2}\Big|\cF_t } &\leq \frac{96(1+R_w)^2[1+(\kappa-1)\xi]}{B(1-\xi)} + {24} \expect{\norm{(w_t-w^*_{\theta_t})}_2^2} + 24 \Delta_{critic}.
\end{align*}

Hence we get,
\begin{align*}
\begin{split}
    &\left(\frac{1}{2}\alpha- {\sqrt{M}L_V\alpha^2} \right)\expect{\norm{\nabla_{\theta}V(\theta_t)}_2^2}\\&\leq \expect{V(\theta_{t+1})} -\expect{V(\theta_t)} \\&+ \left(\frac{1}{2}\alpha+ {\sqrt{M}L_V\alpha^2} \right)\left(\frac{96(1+R_w)^2[1+(\kappa-1)\xi]}{B(1-\xi)} 
    + {24} \expect{\norm{(w_t-w^*_{\theta_t})}_2^2} + 24 \Delta_{critic}\right).
    \end{split}
\end{align*}
We put $\alpha=\frac{1}{4L_V\sqrt{M}}$ above to get, 
\begin{align*}
    \begin{split}
        \left(\frac{1}{16L_V\sqrt{M}} \right)\expect{\norm{\nabla_{\theta}V(\theta_t)}_2^2}&\leq \expect{V(\theta_{t+1})} -\expect{V(\theta_t)} \\&+ \left(\frac{1}{4L_V\sqrt{M}} \right)\left(\frac{96(1+R_w)^2[1+(\kappa-1)\xi]}{B(1-\xi)} 
    + {24} \expect{\norm{(w_t-w^*_{\theta_t})}_2^2} + 24 \Delta_{critic}\right).
    \end{split}
\end{align*}
which simplifies as
\begin{align*}
    \begin{split}
        &\expect{\norm{\nabla_{\theta}V(\theta_t)}_2^2}\\&\leq 16L_V\sqrt{M}\left(\expect{V(\theta_{t+1})} -\expect{V(\theta_t)}\right) + \frac{384(1+R_w)^2[1+(\kappa-1)\xi]}{B(1-\xi)} 
    + {96} \expect{\norm{(w_t-w^*_{\theta_t})}_2^2} + 96 \Delta_{critic}.
    \end{split}
\end{align*}
Taking summation over $t=0,1,2,\ldots, T-1$ and dividing by $T$,
\begin{align*}
\begin{split}
    &\expect{\norm{\nabla_{\theta}V(\theta_{\hat{T}})}_2^2}\\&=\frac{1}{T}\sum\limits_{t=0}^{T-1}\expect{\norm{\nabla_{\theta}V(\theta_t)}_2^2}\\
    &\leq \frac{16L_V\sqrt{M}\left(\expect{V(\theta_{T})} -\expect{V(\theta_0)}\right)}{T} + \frac{384(1+R_w)^2[1+(\kappa-1)\xi]}{B(1-\xi)} 
    + {96} \frac{1}{T}\sum\limits_{t=0}^{T-1}\expect{\norm{(w_t-w^*_{\theta_t})}_2^2} + 96 \Delta_{critic}\\
    &\leq \frac{16L_V\sqrt{M}}{(1-\gamma)T} + \frac{384(1+R_w)^2[1+(\kappa-1)\xi]}{B(1-\xi)} 
    + {96} \frac{1}{T}\sum\limits_{t=0}^{T-1}\expect{\norm{(w_t-w^*_{\theta_t})}_2^2} + 96 \Delta_{critic}
    \end{split}
\end{align*}
We now let $B\geq  \frac{1152}{(1+R_w)^2[1+(\kappa-1)\xi]}{(1-\xi)\epsilon}$, $\expect{\norm{(w_t-w^*_{\theta_t})}_2^2} \leq \frac{\epsilon}{288}$ and $T\geq \frac{48L_V\sqrt{M}}{(1-\gamma)\epsilon}$, then we have
\[\expect{\norm{\nabla_{\theta}V(\theta_{\hat{T}})}_2^2} \leq \epsilon + \cO(\Delta_{critic}).\]
This leads to the final sample complexity of $(B+HT_c)T=\left(\frac{1}{\epsilon}+\frac{\sqrt{M}}{\epsilon}\log\frac{1}{\epsilon} \right)\left(\frac{\sqrt{M}}{(1-\gamma)^2\epsilon} \right)=\cO\left(\frac{M}{(1-\gamma)^2\epsilon^2}\log\frac{1}{\epsilon}\right)$
\end{proof}

\subsection{Natural-actor-critic based improper learning}

\subsubsection{Proof of Theorem~\ref{thm:NAC main theorem}}
\begin{proof}
We first show that the natural actor-critic improper learner converges to a stationary point. We will then show convergence to the global optima which is what is different from that of \cite{Improving-AC-NAC}.

Let $v_t(w):=\frac{1}{B}\sum_{i=0}^{B-1} \cE_w(s_{t,i},m_{t,i}, s_{t,i+1})\psi_{\theta_t}(m_{t,i})$, $A_w(s,m):=\mathbb{E}_{\tilde{P}}[\cE(s,m,s')|s,m]$ and $g(w,\theta):=\mathbb{E}_{\nu_\theta}[A_w(s,m)\psi_\theta(m)]$ for $w\in \Real^d$ and $\theta\in \Real^M$. Also let $u_t(w):=[F_t(\theta_t)+\lambda I]^{-1}\left[\frac{1}{B}\sum_{i=0}^{B-1}\cE_w(s_{t,i},m_{t,i},s_{t,i+1})\psi_{\theta_t}(m_{t,i})\right] = [F_t(\theta_t)+\lambda I]^{-1}v_t(w)$.
    
Recall Prop~\ref{prop:restatement of Prop 1 in AC paper}. We have
\begin{lemma}\label{lemma:NAC stationary pt convergence}
Assume $\sup_{s\in \cS}\norm{\phi(s)}_2 \leq 1$. Under Assumptions~\ref{assumption:ergodicity} and \ref{assumption:bound on w} with step-sizes chosen as $\alpha=\left(\frac{\lambda^2}{2\sqrt{M}L_V(1+\lambda)}\right)$, we have
\begin{align*}
\begin{split}
    \mathbb{E}[\norm{\nabla_\theta V(\theta_{\hat{T}})}_2^2]
    &=\frac{1}{T}\sum_{t=0}^{T-1} \mathbb{E}[\norm{\nabla_\theta V(\theta_t)}_2^2] \\
    &\leq \frac{16\sqrt{M}L_V(1+\lambda)^2}{\lambda^2}\frac{\expect{V(\theta_T)}-V(\theta_0)}{T} + \frac{108}{\lambda^2}[2(1+\lambda)^2+\lambda^2]\frac{\sum_{t=0}^{T-1}\expect{\norm{w_t-w_{\theta_t}^*}_2^2}}{T} \\
    &+ [2(1+\lambda)^2 +\lambda^2]\left( \frac{32}{\lambda^4(1-\gamma)^2} + \frac{432(1+2R_w)^2}{\lambda^2}  \right) \frac{1+(\kappa-1)\xi}{(1-\xi)B} +\frac{216}{\lambda^2}[2(1+\lambda)^2+\lambda^2]\Delta_{critic}.
\end{split}
\end{align*}
\end{lemma}

\begin{proof}
Proof is similar to first part of proof of Thm 6 in \cite{Improving-AC-NAC} and similar to Thm~\ref{thm:AC main theorem}, along with using Prop ~\ref{prop:restatement of Prop 1 in AC paper} and Lemmas ~\ref{lemma:value of Cpi}, ~\ref{lemma:value of Cpsi} and \ref{lemma:value of Lpsi}.


\end{proof}    
\allowdisplaybreaks 
We now move to proving the global optimality of natural actor critic based improper learner. Let $KL(\cdot,\cdot)$ be the KL-divergence between two distributions. We denote $\tt{D}(\theta):=KL(\pi^*, \pi_\theta)$, $u_{\theta_t}^\lambda:=(F(\theta_t)+\lambda I)^{-1}\nabla_\theta V(\theta_t)$ and $u^\dagger_{\theta_t}:=F(\theta_t)^\dagger \nabla_\theta V(\theta_t)$. We see that
\allowdisplaybreaks
\begin{align*}
\allowdisplaybreaks
\begin{split}
\allowdisplaybreaks
    &\tt{D}({\theta_t}) -\tt{D}({\theta_{t+1}}) \\
    &= \sum\limits_{m=1}^M \pi^*(m) \log\pi_{\theta_{t+1}}(m) - \log\pi_{\theta_{t}}(m)\\
    &\stackrel{(i)}{=} \sum\limits_{s\in \cS}d_\rho^{\pi^*}(s)\sum\limits_{m=1}^M \pi^*(m) \log\pi_{\theta_{t+1}}(m) - \log\pi_{\theta_{t}}(m)\\
    &= \mathbb{E}_{\nu_{\pi^*}} [\log\pi_{\theta_{t+1}}(m) - \log\pi_{\theta_{t}}(m)]\\
    &\stackrel{(ii)}{\geq} \mathbb{E}_{\nu_{\pi^*}}\left[\nabla_\theta\log(\pi_{\theta_t}(m)) \right]^\top(\theta_{t+1}-\theta_t) -\frac{\norm{\theta_{t+1}-\theta_t}_2^2}{2}\\
    &=\mathbb{E}_{\nu_{\pi^*}}\left[\psi_{\theta_t}(m) \right]^\top(\theta_{t+1}-\theta_t) -\frac{\norm{\theta_{t+1}-\theta_t}_2^2}{2}\\
    &=\alpha\mathbb{E}_{\nu_{\pi^*}}\left[\psi_{\theta_t}(m) \right]^\top u_t(w_t) -\frac{\alpha^2\norm{u_t(w_t)}_2^2}{2}\\
    &=\alpha\mathbb{E}_{\nu_{\pi^*}}\left[\psi_{\theta_t}(m) \right]^\top u^\lambda_{\theta_t} + \alpha\mathbb{E}_{\nu_{\pi^*}}\left[\psi_{\theta_t}(m) \right]^\top (u_t(w_t)-u^\lambda_{\theta_t})-\frac{\alpha^2\norm{u_t(w_t)}_2^2}{2}\\
    &=\alpha\mathbb{E}_{\nu_{\pi^*}}\left[\psi_{\theta_t}(m) \right]^\top u^\dagger_{\theta_t} + \alpha\mathbb{E}_{\nu_{\pi^*}}\left[\psi_{\theta_t}(m) \right]^\top (u^\lambda_{\theta_t}-u^\dagger_{\theta_t}) + \alpha\mathbb{E}_{\nu_{\pi^*}}\left[\psi_{\theta_t}(m) \right]^\top (u_t(w_t)-u^\lambda_{\theta_t})-\frac{\alpha^2\norm{u_t(w_t)}_2^2}{2}\\
    &=\alpha\mathbb{E}_{\nu_{\pi^*}}[A_{\pi_{\theta_t}}(s,m)] + \alpha\mathbb{E}_{\nu_{\pi^*}}[\psi_{\theta_t}(m)^\top u^\dagger_{\theta_t} - A_{\pi_{\theta_t}}(s,m)] + \alpha\mathbb{E}_{\nu_{\pi^*}}\left[\psi_{\theta_t}(m) \right]^\top (u^\lambda_{\theta_t}-u^\dagger_{\theta_t}) \\&+ \alpha\mathbb{E}_{\nu_{\pi^*}}\left[\psi_{\theta_t}(m) \right]^\top (u_t(w_t)-u^\lambda_{\theta_t})-\frac{\alpha^2\norm{u_t(w_t)}_2^2}{2}\\
    &\stackrel{(iii)}{=} (1-\gamma) (V({\pi^*})-V(\pi_{\theta_t})) + \alpha\mathbb{E}_{\nu_{\pi^*}}[\psi_{\theta_t}(m)^\top u^\dagger_{\theta_t} - A_{\pi_{\theta_t}}(s,m)] + \alpha\mathbb{E}_{\nu_{\pi^*}}\left[\psi_{\theta_t}(m) \right]^\top (u^\lambda_{\theta_t}-u^\dagger_{\theta_t}) \\&+ \alpha\mathbb{E}_{\nu_{\pi^*}}\left[\psi_{\theta_t}(m) \right]^\top (u_t(w_t)-u^\lambda_{\theta_t})-\frac{\alpha^2\norm{u_t(w_t)}_2^2}{2}\\
    &\geq (1-\gamma) (V({\pi^*})-V(\pi_{\theta_t})) - \alpha\sqrt{\mathbb{E}_{\nu_{\pi^*}}[[\psi_{\theta_t}(m)^\top u^\dagger_{\theta_t} - A_{\pi_{\theta_t}}(s,m)]^2]} + \alpha\mathbb{E}_{\nu_{\pi^*}}\left[\psi_{\theta_t}(m) \right]^\top (u^\lambda_{\theta_t}-u^\dagger_{\theta_t}) \\&+ \alpha\mathbb{E}_{\nu_{\pi^*}}\left[\psi_{\theta_t}(m) \right]^\top (u_t(w_t)-u^\lambda_{\theta_t})-\frac{\alpha^2\norm{u_t(w_t)}_2^2}{2}\\
    &\stackrel{(iv)}{\geq} (1-\gamma) (V({\pi^*})-V(\pi_{\theta_t})) - \sqrt{\norm{\frac{\nu_{\pi^*}}{\nu_{\pi_{\theta_t}}}}_{\infty}}\alpha\sqrt{\mathbb{E}_{\nu_{\pi^*}}[[\psi_{\theta_t}(m)^\top u^\dagger_{\theta_t} - A_{\pi_{\theta_t}}(s,m)]^2]} + \alpha\mathbb{E}_{\nu_{\pi^*}}\left[\psi_{\theta_t}(m) \right]^\top (u^\lambda_{\theta_t}-u^\dagger_{\theta_t}) \\&+ \alpha\mathbb{E}_{\nu_{\pi^*}}\left[\psi_{\theta_t}(m) \right]^\top (u_t(w_t)-u^\lambda_{\theta_t})-\frac{\alpha^2\norm{u_t(w_t)}_2^2}{2}\\
    &\stackrel{(v)}{\geq} (1-\gamma) (V({\pi^*})-V(\pi_{\theta_t})) - \sqrt{\frac{1}{1-\gamma}\norm{\frac{\nu_{\pi^*}}{\nu_{\pi_{\theta_0}}}}_{\infty}}\alpha\sqrt{\mathbb{E}_{\nu_{\pi^*}}[[\psi_{\theta_t}(m)^\top u^\dagger_{\theta_t} - A_{\pi_{\theta_t}}(s,m)]^2]}  \\&+\alpha\mathbb{E}_{\nu_{\pi^*}}\left[\psi_{\theta_t}(m) \right]^\top (u^\lambda_{\theta_t}-u^\dagger_{\theta_t}) + \alpha\mathbb{E}_{\nu_{\pi^*}}\left[\psi_{\theta_t}(m) \right]^\top (u_t(w_t)-u^\lambda_{\theta_t})-\frac{\alpha^2\norm{u_t(w_t)}_2^2}{2}\\
    &\stackrel{(vi)}{\geq}  (1-\gamma) (V({\pi^*})-V(\pi_{\theta_t})) - \sqrt{\frac{1}{1-\gamma}\norm{\frac{\nu_{\pi^*}}{\nu_{\pi_{\theta_0}}}}_{\infty}}\alpha\sqrt{\mathbb{E}_{\nu_{\pi^*}}[[\psi_{\theta_t}(m)^\top u^\dagger_{\theta_t} - A_{\pi_{\theta_t}}(s,m)]^2]} - \alpha C_{soft}\lambda \\&- 2\alpha\norm{ u_t(w_t)-u^\lambda_{\theta_t}}_2-\frac{\alpha^2\norm{u_t(w_t)}_2^2}{2}.
    \end{split}
\end{align*}
where (i) is by taking an extra expectation without changing the inner summand, (ii) follows by Lemma~\ref{lemma:value of Lpsi} and Lemma 5 in \cite{Improving-AC-NAC}, (iii) follows by the value difference lemma (Lemma~\ref{lemma:value diffence lemma}), (iv) follows by defining $\norm{\frac{\nu_{\pi^*}}{\nu_{\pi_{\theta_t}}}}_{\infty}:=\max_{s,m}\frac{\nu_{\pi^*}(s,m)}{\nu_{\pi_{\theta_t}}(s,m)}$, (v) follows because $\nu_{\pi_{\theta_t}}(s,m) \geq (1-\gamma)\nu_{\pi_{\theta_0}}(s,m)$, (vi) follows by Lemma6 in \cite{Improving-AC-NAC}and Lemma~\ref{lemma:value of Cpsi}.

Next, we denote $\Delta_{actor}:=\max_{\theta\in \Real^M}\min_{w\in \Real^d}\mathbb{E}_{\nu_{\pi_\theta}}[[\psi_\theta^\top w - A_{\pi_\theta}(s,m)]^2]$ as the actor error. 
\begin{align*}
\allowdisplaybreaks
\begin{split}
    &\tt{D}({\theta_t}) -\tt{D}({\theta_{t+1}}) \\
    &\geq(1-\gamma) (V({\pi^*})-V(\pi_{\theta_t})) - \sqrt{\frac{1}{1-\gamma}\norm{\frac{\nu_{\pi^*}}{\nu_{\pi_{\theta_0}}}}_{\infty}}\alpha\sqrt{\Delta_{actor}} - \alpha C_{soft}\lambda \\&- 2\alpha\norm{ u_t(w_t)-u^\lambda_{\theta_t}}_2-\frac{\alpha^2\norm{u_t(w_t)}_2^2}{2}\\
    &\geq (1-\gamma) (V({\pi^*})-V(\pi_{\theta_t})) - \sqrt{\frac{1}{1-\gamma}\norm{\frac{\nu_{\pi^*}}{\nu_{\pi_{\theta_0}}}}_{\infty}}\alpha\sqrt{\Delta_{actor}} - \alpha C_{soft}\lambda \\&- 2\alpha\norm{ u_t(w_t)-u^\lambda_{\theta_t}}_2-\frac{\alpha^2\norm{u_t(w_t)-u^\lambda(\theta_t)}_2^2}{2}-\frac{\alpha^2}{\lambda^2}\norm{\nabla_{\theta}V(\theta_t)}_2^2.\\
\end{split}
\end{align*}
Rearranging and dividing by $(1-\gamma)\alpha$, and taking expectation both sides we get
\begin{align*}
\begin{split}
    & V({\pi^*})-\expect{V(\pi_{\theta_t})}\\
    &\leq \frac{\expect{\tt{D}(\theta_t)}-\expect{\tt{D}(\theta_{t+1})}}{(1-\gamma)\alpha} +\frac{2\sqrt{\mathbb{E}[\norm{ u_t(w_t)-u^\lambda_{\theta_t}}_2^2]}}{1-\gamma} + \frac{\alpha\expect{\norm{u_t(w_t)-u^\lambda(\theta_t)}_2^2}}{2(1-\gamma)} \\&+ \frac{\alpha}{\lambda^2(1-\gamma)}\expect{\norm{\nabla_{\theta}V(\theta_t)}_2^2} + \sqrt{\frac{1}{(1-\gamma)^3}\norm{\frac{\nu_{\pi^*}}{\nu_{\pi_{\theta_0}}}}_{\infty}}\sqrt{\Delta_{actor}} + \frac{C_{soft}\lambda}{1-\gamma}.\\
    \end{split}
\end{align*}
Next we use the same argument as in eq (33) and Lemma 2 in \citet{Improving-AC-NAC} to bound the second term.
\begin{align*}
    \expect{\norm{u_t(w_t)-u^\lambda_{\theta_t}}_2^2} &\leq \frac{C}{B}+ \frac{108\expect{\norm{w_t-w^*_{\theta_t}}_2^2}}{\lambda^2} + \frac{216\Delta_{critic}}{\lambda^2}.
\end{align*}
where $C:=\frac{18}{\lambda^2}\frac{24(1+2R_w)^2[1+(\kappa-1)\xi]}{B(1-\xi)} + \frac{4}{\lambda^4(1-\gamma)^2}.\frac{8[1+(\kappa-1)\xi]}{(1-\xi)B}.$
Using this in the bound and using $\sqrt{a+b}\leq \sqrt{a}+\sqrt{b}$ for positive $a,b$ above, we have,
\begin{align*}
    \begin{split}
        & V({\pi^*})-\expect{V(\pi_{\theta_t})}\\
        &\leq  \frac{\expect{\tt{D}(\theta_t)}-\expect{\tt{D}(\theta_{t+1})}}{(1-\gamma)\alpha} +\frac{2}{1-\gamma}\left( \sqrt{\frac{C}{B}}+ 11\sqrt{\frac{\expect{\norm{w_t-w^*_{\theta_t}}_2^2}}{\lambda^2}} + 15\sqrt{\frac{\Delta_{critic}}{\lambda^2}}  \right) \\&+ \frac{\alpha}{2(1-\gamma)}
        \left({\frac{C}{B}}+ {108\frac{\expect{\norm{w_t-w^*_{\theta_t}}_2^2}}{\lambda^2}} + {216\frac{\Delta_{critic}}{\lambda^2}}  \right)
        \\&+ \frac{\alpha}{\lambda^2(1-\gamma)}\expect{\norm{\nabla_{\theta}V(\theta_t)}_2^2} + \sqrt{\frac{1}{(1-\gamma)^3}\norm{\frac{\nu_{\pi^*}}{\nu_{\pi_{\theta_0}}}}_{\infty}}\sqrt{\Delta_{actor}} + \frac{C_{soft}\lambda}{1-\gamma}.
    \end{split}
\end{align*}
    Summing over all $t=0,1,\ldots,T-1$ and then dividing by $T$ we get,
    \begin{align*}
    \begin{split}
        &V({\pi^*})-\frac{1}{T}\sum\limits_{t=0}^{T-1}\expect{V(\pi_{\theta_t})}\\
        & \leq  \frac{{\tt{D}(\theta_0)}-\expect{\tt{D}(\theta_{T})}}{(1-\gamma)\alpha T} +\frac{2}{(1-\gamma)}\left( \sqrt{\frac{C}{B}} + 15\sqrt{\frac{\Delta_{critic}}{\lambda^2}}  \right) +\frac{22}{(1-\gamma)T}\sum_{t=0}^{T-1}\sqrt{\frac{\expect{\norm{w_t-w^*_{\theta_t}}_2^2}}{\lambda^2}} \\&+ \frac{\alpha}{2(1-\gamma)}
        \left({\frac{C}{B}} + {216\frac{\Delta_{critic}}{\lambda^2}}  \right) + \frac{54\alpha}{(1-\gamma)T}\sum\limits_{t=0}^{T-1}{\frac{\expect{\norm{w_t-w^*_{\theta_t}}_2^2}}{\lambda^2}}
        \\&+ \frac{\alpha}{\lambda^2(1-\gamma)T}\sum\limits_{t=0}^{T-1}\expect{\norm{\nabla_{\theta}V(\theta_t)}_2^2} + \sqrt{\frac{1}{(1-\gamma)^3}\norm{\frac{\nu_{\pi^*}}{\nu_{\pi_{\theta_0}}}}_{\infty}}\sqrt{\Delta_{actor}} + \frac{C_{soft}\lambda}{1-\gamma}.
    \end{split}
    \end{align*}
    We now put the value of $\alpha \leq\frac{\lambda^2}{2\sqrt{M}L_V(1+\lambda)}$, we get,
    \begin{align*}
        \begin{split}
          &V({\pi^*})-\frac{1}{T}\sum\limits_{t=0}^{T-1}\expect{V(\pi_{\theta_t})}\\
          &\leq C_1\frac{\sqrt{M}}{T} + \frac{C_2}{\sqrt{B}} + C_3\sqrt{\Delta_{critic}} + \frac{C_4}{T}\sum_{t=0}^{T-1}\sqrt{\expect{\norm{w_t-w^*_{\theta_t}}_2^2}} \\
          &+ \frac{C_5}{B} + C_6\sqrt{\Delta_{critic}} +  \frac{C_7}{T}\sum\limits_{t=0}^{T-1}\expect{\norm{w_t-w^*_{\theta_t}}_2^2} + \frac{C_8}{T}\sum\limits_{t=0}^{T-1}\expect{\norm{\nabla_{\theta}V(\theta_t)}_2^2} \\
          &+ \sqrt{\frac{1}{(1-\gamma)^3}\norm{\frac{\nu_{\pi^*}}{\nu_{\pi_{\theta_0}}}}_{\infty}}\sqrt{\Delta_{actor}} + C_9\lambda. 
        \end{split}
    \end{align*}
    Letting $T=\mathcal{O}\left( \frac{\sqrt{M}}{(1-\gamma)^2\epsilon} \right)$, $B=\cO\left( \frac{1}{(1-\gamma)^2\epsilon^2} \right)$ then $\expect{\norm{\nabla_{\theta}V(\theta_t)}_2^2}\leq \epsilon^2$ and
    \[V({\pi^*})-\frac{1}{T}\sum\limits_{t=0}^{T-1}\expect{V(\pi_{\theta_t})} \leq \epsilon + \cO\left(\sqrt{\frac{\Delta_{actor}}{(1-\gamma)^3} }\right) + \cO(\Delta_{critic}) + \cO(\lambda). \]
    This leads to the total sample complexity as 
    \[(B+HT_c)T = \cO\left(\left( \frac{1}{(1-\gamma)^2\epsilon^2} + \frac{\sqrt{M}}{\epsilon^2}\log\frac{1}{\epsilon} \right)\frac{\sqrt{M}}{(1-\gamma)^2\epsilon}\right) = \cO\left(\frac{M}{(1-\gamma)^4\epsilon^3} \log\frac{1}{\epsilon}\right).\]
\end{proof}


\section{Simulation Details}\label{sec:simulation-details}
In this section we describe the details of the Sec.~\ref{sec:simulation-results}. Recall that since neither value functions nor value gradients are available in closed-form, we modify SoftMax PG (Algorithm~\ref{alg:mainPolicyGradMDP}) to make it generally implementable using a combination of (1) rollouts to estimate the value function of the current (improper) policy and (2) a stochastic approximation-based approach to estimate its value gradient. 

The Softmax PG with Gradient Estimation or \textbf{SPGE} (Algorithm ~\ref{alg:mainPolicyGradMDPgradEst}), and the gradient estimation algorithm ~\ref{alg:gradEst}, \emph{\color{blue}GradEst}, are shown below. 

 \begin{minipage}{0.38\textwidth}
\begin{algorithm}[H]
    \centering
      \caption{Softmax PG with Gradient Estimation ({\color{blue}SPGE})}
  \label{alg:mainPolicyGradMDPgradEst}
    \begin{algorithmic}[1]
    \STATE {\bfseries Input:} learning rate $\eta>0$, perturbation parameter $\alpha>0$, Initial state distribution $\mu$
   \STATE Initialize each $\theta^1_m=1$, for all $m\in [M]$, $s_1\sim \mu$
   \FOR{$t=1$ {\bfseries to} $T$}
   \STATE Choose controller $m_t\sim \pi_t$.
   \STATE Play action $a_t \sim K_{m_t}(s_{t},:)$.
   \STATE Observe $s_{t+1}\sim \tP(.|s_t,a_t)$.
   \STATE ${\widehat{\nabla_{\theta^t} V^{\pi_{\theta_t}}(\mu)}}= \text{\tt GradEst}({\theta_t},\alpha,\mu)$
   \STATE Update: \\$\theta^{t+1} = \theta^{t}+ \eta. \widehat{\nabla_{\theta^t} V^{\pi_{\theta_t}}(\mu)}$.
   \ENDFOR    
    \end{algorithmic}
\end{algorithm}
\end{minipage}
\hfill
\begin{minipage}{0.55\textwidth}
\begin{algorithm}[H]
    \centering
      \caption{GradEst ({\color{blue}subroutine for SPGE})}
  \label{alg:gradEst}
    \begin{algorithmic}[1]
    \STATE {\bfseries Input:} Policy parameters $\theta$, parameter $\alpha>0$, Initial state distribution $\mu$.
   \FOR{$i=1$ {\bfseries to} $\#\tt{runs}$}
   \STATE $u^i\sim Unif(\mathbb{S}^{M-1}).$
   \STATE $\theta_\alpha = \theta+\alpha.u^i$
   \STATE $\pi_\alpha = \softmax(\theta_\alpha)$
   \FOR{$l=1$ {\bfseries to} $\#\tt{rollouts}$}
   \STATE Generate trajectory $(s_0,a_0,r_0,s_1,a_1,r_1,\ldots, s_{\tt{lt}},a_{\tt{lt}},r_{\tt{lt}})$ using the policy $\pi_\alpha:$ and $s_0 \sim \mu$. 
   \STATE $\tt{reward}^l=\sum\limits_{j=0}^{lt}\gamma^jr_j$
   \ENDFOR
   \STATE  $\tt{mr}(i) = \tt{mean}(\tt{reward})$
   \ENDFOR
   \STATE $\tt{GradValue} = \frac{1}{\#runs}\sum\limits_{i=1}^{\#\tt{runs}}{\tt{mr}}(i).u^i.\frac{M}{\alpha}.$
   \STATE {\bfseries Return:} $\tt{GradValue}$.
\end{algorithmic}
\end{algorithm}
\end{minipage}

Next we report some extra simulations we performed under different environments.

\subsection{State Dependent controllers -- Chain MDP}\label{subsec:State Dependent controllers -- Chain MDP} We consider a linear chain MDP as shown in Figure \ref{fig:linearChainMDP}. As evident from the figure, $\abs{\cS}=10$ and the learner has only two actions available, which are $\cA = \{\tt{left}, \tt{right}\}$. Hence the name `chain'. The numbers on the arrows represent the reward obtained with the transition. The initial state is $s_1$. We let $s_{100}$ as the terminal state. Let us define 2 base controllers, $K_1$ and $K_2$, as follows. 
\begin{align*}
    K_1(\tt{left}\given s_j) &=\begin{cases}
    1, & j\in[9]\backslash \{5\}\\
    0.1, & j = 5\\
    0, & j=10.
    \end{cases}
\end{align*}
\begin{align*}
    K_2(\tt{left}\given s_j) &=\begin{cases}
    1, & j\in[9]\backslash \{6\}\\
    0.1, & j = 6\\
    0, & j=10.
    \end{cases}
\end{align*}
and obviously $K_i(\tt{right}|s_j) = 1-K_i(\tt{left}|s_j)$ for $i=1,2$.
An improper mixture of the two controllers, i.e., $(K_1+K_2)/2$ is the optimal in this case. We show that our policy gradient indeed converges to the `correct' combination, see Figure ~\ref{fig:Chainfigure}. 
\begin{figure}
    \centering
    \includegraphics[scale=0.7]{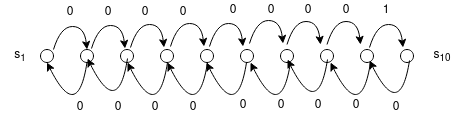}
    \caption{A chain MDP with 10 states.}
    \label{fig:linearChainMDP}
\end{figure}
\begin{figure}
    \centering
    \includegraphics[scale=0.75]{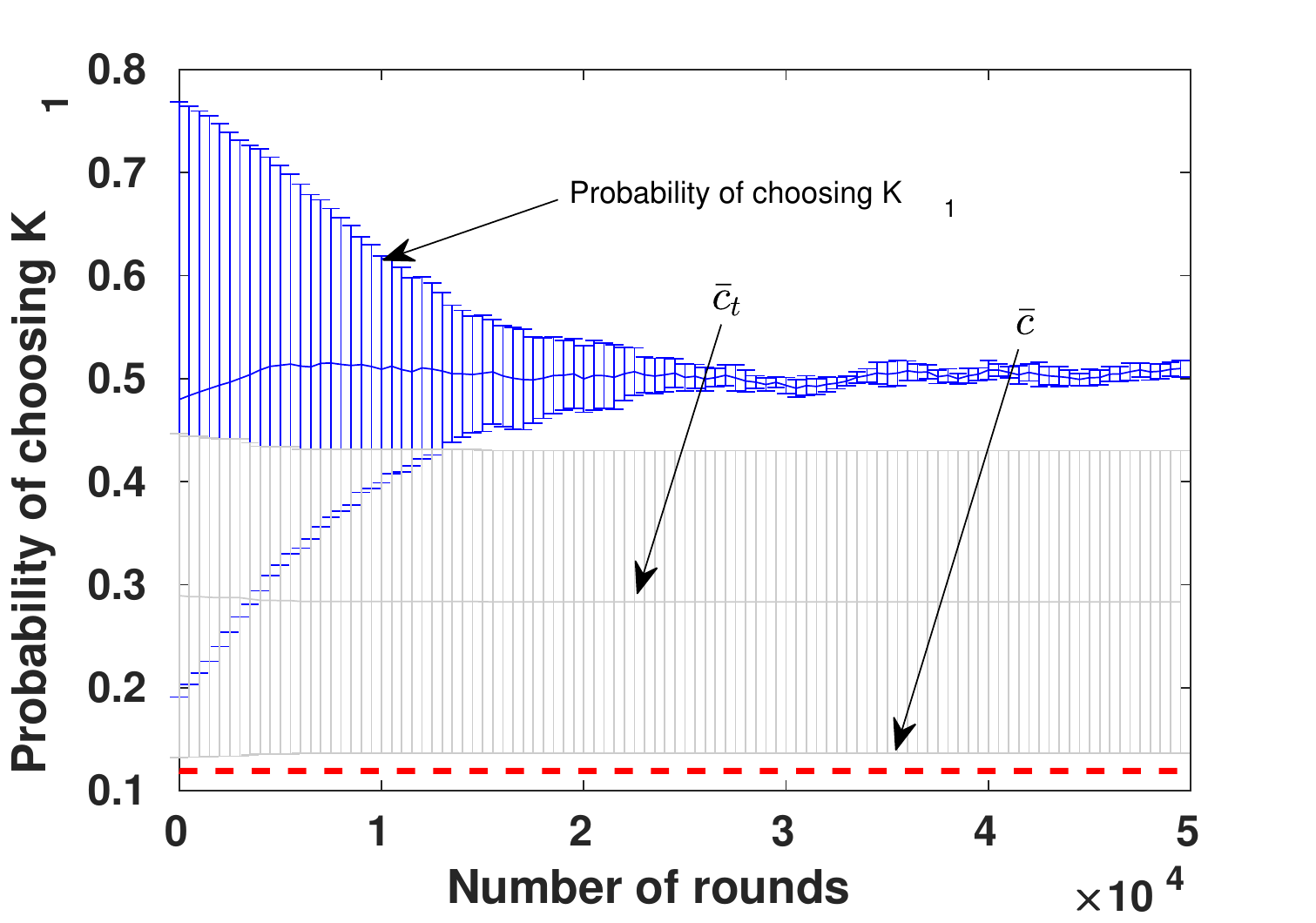}
    \caption{Softmax PG alg applied to the linear Chain MDP with various randomly chosen initial distribution. Plot shows probability of choosing controller $K_1$ averaged over $\#\tt{trials}$}.
    \label{fig:Chainfigure}
\end{figure}
We here provide an elementary calculation of our claim that the mixture $K_{\tt{mix}}:=(K_1+K_2)/2$ is indeed better than applying $K_1$ or $K_2$ for all time. We first analyze the value function due to $K_i, i=1,2$ (which are the same due to \emph{symmetry} of the problem and the probability values described).
\begin{align*}
    V^{K_i}(s_1) &= \expect{\sum\limits_{t\geq0} \gamma^tr_t(a_t,s_t)}\\
    &=0.1\times \gamma^9+0.1\times 0.9\times 0.1\times \gamma^{11} + 0.1\times 0.9\times 0.1\times0.9\times 0.1\times \gamma^{13} \ldots\\
    &=0.1\times \gamma^9 \left(1+\left(0.1\times 0.9 \gamma^2\right)+ \left(0.1\times 0.9 \gamma^2\right)^2+\ldots \right)=\frac{0.1\times \gamma^9}{1-0.1\times0.9\times \gamma^2}.
\end{align*}
We will next analyze the value if a true mixture controller i.e., $K_{\tt{mix}}$ is applied to the above MDP. The analysis is a little more intricate than the above. We make use of the following key observations, which are elementary but crucial. 
\begin{enumerate}
    \item Let $\tt{Paths}$ be the set of all sequence of states starting from $s_1$, which terminate at $s_{10}$ which can be generated under the policy $K_{\tt{mix}}$. Observe that 
    \begin{equation}
        V^{K_{\tt{mix}}}(s_1) = \sum\limits_{\underline{p}\in \tt{Paths}}\gamma^{{\tt{length}}(\underline{p})}\prob{\underline{p}}.1.
    \end{equation}
    Recall that reward obtained from the transition $s_9\to s_{10}$ is 1.
    \item Number of distinct paths with exactly $n$ loops: $2^n$.
    \item Probability of each such distinct path with $n$ cycles:
    \begin{align*}
        &=\underbrace{(0.55\times 0.45)\times (0.55\times 0.45)\times \ldots (0.55\times 0.45)}_{n \,\tt{ times}}\times 0.55\times 0.55\times \gamma^{9+2n}\\
        &=\left(0.55\right)^2\times \gamma^9\left(0.55\times 0.45\times \gamma^2 \right)^n.
    \end{align*}
    \item Finally, we put everything together to get:
    \begin{align*}
        V^{K_{\tt{mix}}}(s_1) = &\sum\limits_{n=0}^{\infty} 2^n\times \left(0.55\right)^2 \times \gamma^9\times \left(0.55\times 0.45\times \gamma^2 \right)^n\\
        &=\frac{\left(0.55\right)^2\times \gamma^9}{1-2\times 0.55\times0.45\times \gamma^2} > V^{K_i}(s_1).
    \end{align*}
\end{enumerate}
This shows that a mixture performs better than the constituent controllers. The plot shown in Fig. ~\ref{fig:Chainfigure} shows the Softmax PG algorithm (even with estimated gradients and value functions) converges to a (0.5,0.5) mixture correctly.
\subsection{Stationary Bernoulli Queues}

\begin{figure}[t]
 \centering
\subfigure[Arrival rate:$(\lambda_1,\lambda_2)=(0.49,0.49)$]{\label{subfig:equal arrival rate}\includegraphics[scale = 0.32]{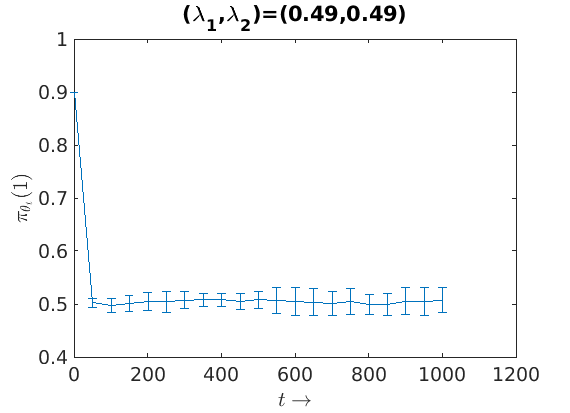}}
\qquad
\subfigure[Arrival rate:$(\lambda_1,\lambda_2)=(0.49,0.49)$]{\label{subfig:equal arrival rate}\includegraphics[scale = 0.32]{plots/queuiung_example_corrected.png}}
\qquad        
\subfigure[Arrival rate:$(\lambda_1,\lambda_2)=(0.3,0.4)$]{\label{subfig:unequal arrival rate}\includegraphics[scale = 0.32]{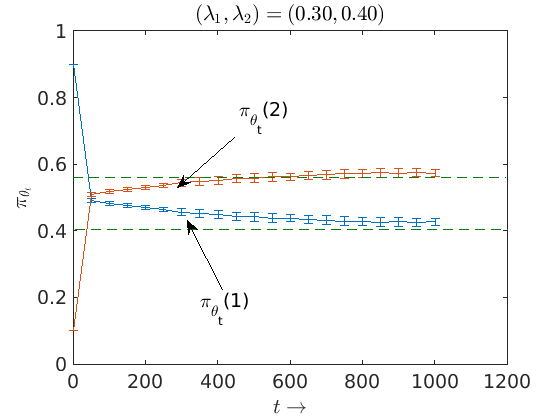}}
\qquad
\subfigure[(Estimated) Value functions for case with the two base policies and Longest Queue First (``LQF") ]{\label{subfig:Value function comparisom}        \includegraphics[scale = 0.32]{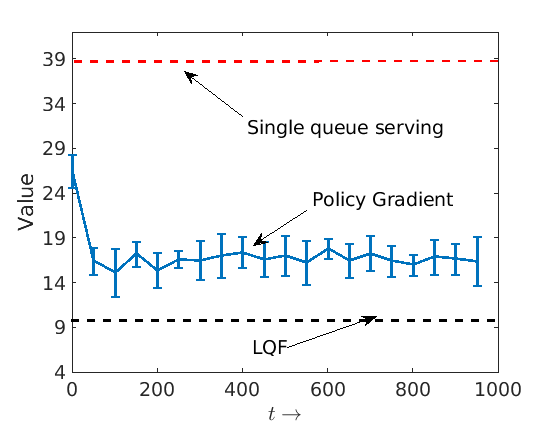}}
\qquad
\subfigure[Case with 3 experts: Always Queue~1, Always Queue~2 and LQF.]{\label{subfig:3experts}\includegraphics[scale = 0.32]{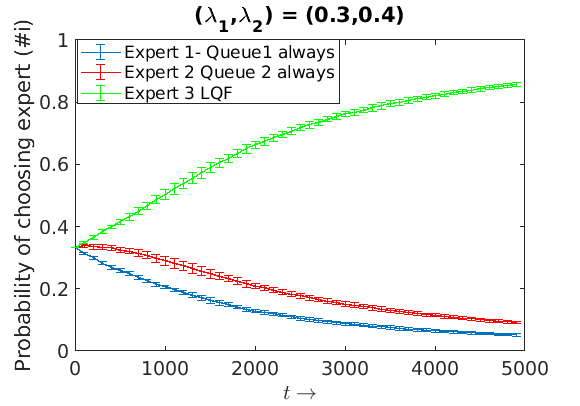}}
\caption{Softmax policy gradient algorithm applies show convergence to the best mixture policy.}
\label{fig:simulations}
\end{figure}


We study two different settings (1) where in the first case the optimal policy is a strict improper combination of the available controllers and (2) where it is at a corner point, i.e., one of the available controllers itself is optimal. Our simulations show that in both the cases, PG converges to the correct controller distribution. 

Recall the example that we discussed in Sec.~\ref{sec:motivating-queueing-example}. We consider the case with Bernoulli arrivals with rates $\boldsymbol{\lambda}=[\lambda_1,\lambda_2]$ and are given two base/atomic controllers $\{K_1,K_2\}$, where controller $K_i$ serves Queue~$i$ with probability $1$, $i=1,2$. As can be seen in Fig.~\ref{subfig:equal arrival rate} when $\boldsymbol{\lambda}=[0.49,0.49]$ (equal arrival rates), GradEst converges to an improper mixture policy that serves each queue with probability $[0.5,0.5]$. Note that this strategy will also stabilize the system whereas both the base controllers lead to instability (the queue length of the unserved queue would obviously increase without bound). Figure \ref{subfig:unequal arrival rate}, shows that with unequal arrival rates too, GradEst quickly converges to the best policy. 

Fig.~\ref{subfig:Value function comparisom} shows the evolution of the value function of GradEst (in blue) compared with those of the base controllers (red) and the \emph{Longest Queue First} policy (LQF) which, as the name suggests, always serves the longest queue in the system (black). LQF, like any policy that always serves a nonempty queue in the system whenever there is one\footnote{Tie-breaking rule is irrelevant.}, is known to be optimal in the sense of delay minimization for this system \cite{LQF2016}. See Sec.~\ref{sec:simulation-details} in the Appendix for more details about this experiment.

Finally, Fig.~\ref{subfig:3experts} shows the result of the second experimental setting with three base controllers, one of which is delay optimal. The first two are $K_1,K_2$ as before and the third controller, $K_3$, is LQF. Notice that $K_1,K_2$ are both queue length-agnostic, meaning they could attempt to serve empty queues as well. LQF, on the other hand, always and only serves nonempty queues. Hence, in this case the optimal policy is attained at one of the corner points, i.e., $[0,0,1]$. The plot shows the PG algorithm converging to the correct point on the simplex. 

Here, we justify the value of the two policies which always follow one fixed queue, that is plotted as straight line in Figure \ref{subfig:Value function comparisom}. Let us find the value of the policy which always serves queue 1. The calculation for the other expert (serving queue 2 only) is similar. Let $q_i(t)$ denote the length of queue $i$ at time $t$. We note that since the expert (policy) always recommends to serve one of the queue, the expected \textit{cost} suffered in any round $t$ is $c_t=q_1(t)+q_2(t) = 0 + t.\lambda_2$. Let us start with empty queues at $t=0$. 
 \begin{align*}
     V^{Expert 1}(\mathbf{\underline{0}}) &= \expect{\sum\limits_{t=0}^T \gamma^tc_t\given Expert 1}\\
     &= \sum\limits_{t=0}^T \gamma^t. t. \lambda_2\\
     &\leq\lambda_2.\frac{\gamma}{(1-\gamma)^2}.
 \end{align*}
With the values, $\gamma=0.9$ and $\lambda_2=0.49$, we get $V^{Expert 1}(\mathbf{\underline{0}})\leq 44$, which is in good agreement with the bound shown in the figure.

\subsection{Details of Path (Interference) Graph Networks}

\begin{figure*}[t]
\centering
    \subfigure[A basic path-graph interference system with $N=4$ communication links.]{\label{subfig:PGN-TxRx}
        \includegraphics[scale = 0.37]{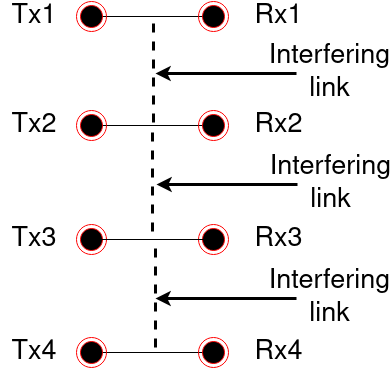}}
    \qquad
    \subfigure[The associated conflict (interference) graph is a \emph{path-graph}.]{\label{subfig:PGN-conflict graph}
        \includegraphics[scale = 0.3]{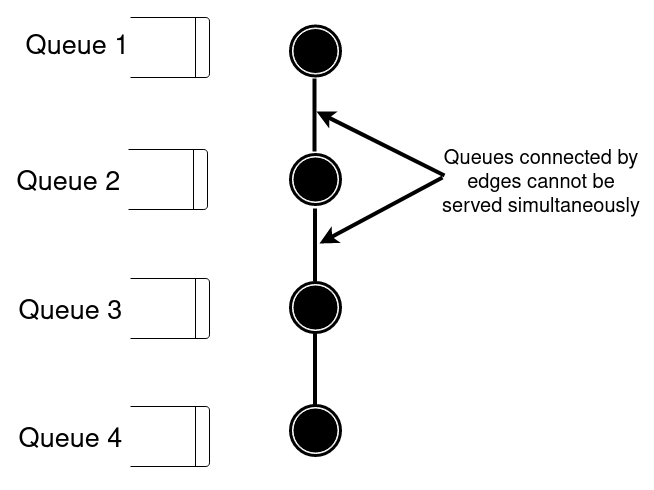}}
\caption{An example of a path graph network. The interference constraints are such that physically adjacent queues cannot be served simultaneously.}
\label{fig:PGN}
\end{figure*}
Consider a system of parallel transmitter-receiver pairs as shown in Figure ~\ref{subfig:PGN-TxRx}. Due to the physical arrangement of the Tx-Rx pairs, no two adjacent systems can be served simultaneously because of interference. This type of communication system is commonly referred to as a \textit{path graph network} \cite{Mohan2020ThroughputOD}. Figure ~\ref{subfig:PGN-conflict graph} shows the corresponding \emph{conflict graph}. Each Tx-Rx pair can be thought of as a queue, and the edges between them represent that the two connecting queues, cannot be served simultaneously. On the other hand, the sets of queues which can be served simultaneously are called \emph{independent sets} in the queuing theory literature. In the figure above, the independent sets are $\{\emptyset,\{1\}, \{2\}, \{3\}, \{4\}, \{1,3\}, \{2,4\}, \{1,4\} \}$. 

Finally, in Table \ref{table:meandelay}, we report the mean delay values of the 5 base controllers we used in our simulation Fig. \ref{subfig:BQ5}, Sec.\ref{sec:simulation-results}. We see the controller $K_2$ which was chosen to be MER, indeed has the lowest cost associated, and as shown in Fig. ~\ref{subfig:BQ5}, our Softmax PG algorithm (with estimated value functions and gradients) converges to it.

\begin{table}[ht]\caption{Mean Packet Delay Values of Path Graph Network Simulation.}
\centering 
\begin{tabular}{c c c}
\hline\hline                     
Controller & Mean delay (\# time slots) over 200 trials & Standard deviation \\ [0.5ex]
\hline              
$K_1(MW)$ & 22.11  & 0.63  \\
${\color{blue}K_2(MER)}$ & \textbf{\color{blue}20.96} &  0.65  \\
$K_3(\{1,3\})$ & 80.10 &  0.92  \\
$K_4(\{2,4\})$ & 80.22 & 0.90 \\
$K_5(\{1,4\})$ & 80.13 &  0.91 \\ [1ex]      
\hline
\end{tabular}\label{table:meandelay}
\end{table}

\subsection{Cartpole Experiments} 
We investigate further the example in our simulation in which the two constituent controllers are $K_{opt}+\Delta$ and $K_{opt}-\Delta$. We use OpenAI gym to simulate this situation. In the Figure ~\ref{subfig:cartpole--symm}, it was shown our Softmax PG algorithm (with estimated values and gradients) converged to a improper mixture of the two controllers, i.e., $\approx(0.53,0.47)$. Let $K_{\tt{conv}}$ be defined as the (randomized) controller which chooses $K_1$ with probability 0.53, and $K_2$ with probability 0.47. Recall from Sec.~\ref{sec:motivating-cartpole-example} that this control law converts the linearized cartpole into an Ergodic Parameter Linear System (EPLS).  In Table~\ref{table:meanfalls} we report the average number of rounds the pendulum stays upright when different controllers are applied for all time, over trajectories of length 500 rounds. The third column displays an interesting feature of our algorithm. Over 100 trials, the base controllers do not stabilize the pendulum for a relatively large number of trials, however, $K_{\tt{conv}}$ successfully does so most of the times.
\begin{table}[ht]\caption{A table showing the number of rounds the constituent controllers manage to keep the cartpole upright.}
\centering 
\begin{tabular}{c c c}
\hline\hline                     
Controller & \makecell{Mean number of rounds \\before the pendulum falls $\land$ 500} & \makecell{$\#$ Trials out of 100 in which \\the pendulum falls before 500 rounds} \\ [0.5ex]
\hline              
$K_1(K_{opt}+\Delta)$ & 403  & 38  \\

${K_2(K_{opt}-\Delta)}$ & 355 &  46  \\
${\color{blue}K_{\tt{conv}}}$ & 465 &  \textbf{\color{blue}8}  \\ [1ex]
\hline
\end{tabular}\label{table:meanfalls}
\end{table}

We mention here that if one follows $K^*$, which is the optimum controller matrix one obtains by solving the standard Discrete-time Algebraic Riccatti Equation (DARE) \cite{bertsekas11dynamic}, the pole does not fall over 100 trials. However, as indicated in Sec.\ref{sec:Introduction}, constructing the optimum controller for this system from scratch requires exponential, in the number of state dimension, sample complexity \cite{chen-hazan20black-box-control-linear-dynamical}. On the other hand $K_{\tt{conv}}$ performs very close to the optimum, while being sample efficient.


\textbf{Choice of hyperparameters.} In the simulations, we set learning rate to be $10^{-4}$, $\#\tt{runs}=10,\#\tt{rollouts}=10, \tt{lt}= 30$, discount factor $\gamma=0.9$ and $\alpha=1/\sqrt{\#\tt{runs}}$. All the simulations have been run for 20 trials and the results shown are averaged over them. We capped the queue sizes at 1000. 

 \subsection{Some extra simulations for natural-actor-critic based improper learner NACIL}
 
 \begin{itemize}
     \item First we show a queuing theory where we have 2 queues to be served and we have two base controllers similar to as we discussed in the Sec~\ref{sec:motivating-examples}. However,  here we have two different arrival rates for the two queues $(\lambda_1,\lambda_2)\equiv (0.4,0.3)$, i.e., the arrival rates are unequal. We plot in Fig.~\ref{fig:NAC unequal arrivals} the probability of choosing the two different controllers. We see that ACIL converges to the ``correct" mixture of the base controllers.
     \item Next, we show a simulation on the setting in Sec.~\ref{subsec:State Dependent controllers -- Chain MDP}, which we called a Chain MDP. We recall that this setting consists of two base controllers $K_1$ and $K_2$, however a $(1/2,1/2)$ mixture of the two controllers was shown (analytically) to perform better than each individual ones. As the plot in Fig.~\ref{fig:NAC chain MDP} shows NACIL identifies the correct combination and follows it. 
 \end{itemize}
 {\bfseries{Choice of hyperparameters.}} For the queuing theoretic simulations of Algorithm~\ref{alg:actor-critic improper RL alg} ACIL, we choose $\alpha=10^{-4}$, $\beta=10^{-3}$. We choose the identity mapping $\phi(s)\equiv s$, where $s$ is the current state of the system which is a $N-$length vector, which consists of the $i^{th}$ queue length at the $i^{th}$ position. $\lambda$ was chosen to be $0.1$. The other parameters are chosen as $B=50$, $H=30$ and $T_c=20$. We choose a buffer of size 1000 to keep the states bounded, i.e., if a queue exceeds a size 1000, those arrivals are ignored and queue length is not increased. This is used to normalize the $\norm{\phi(s)}_2$ across time. 
\begin{figure*}
    \centering
    \includegraphics[scale=0.75]{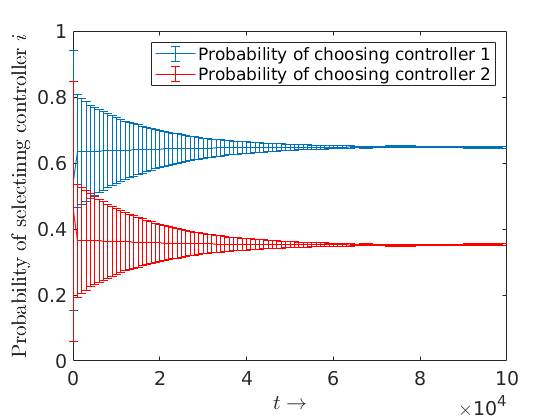}
    \caption{NACIL alg applied to the a queuing system with two queues, having arrival rates $(\lambda_1,\lambda_2)\equiv(0.4,0.3)$. Plot shows probability of choosing controllers $K_1$ and $K_2$ averaged over 20 trials}.
    \label{fig:NAC unequal arrivals}
\end{figure*}

\begin{figure*}
    \centering
    \includegraphics[scale=0.75]{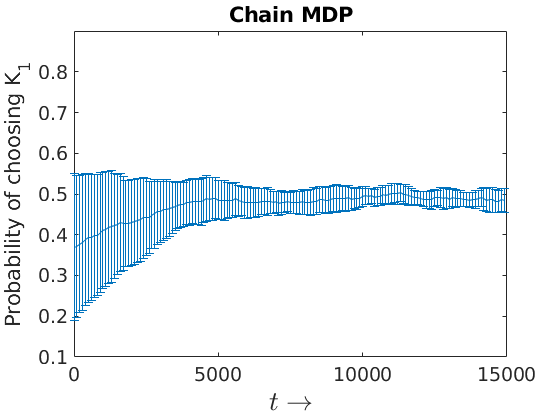}
    \caption{NACIL alg applied to the linear Chain MDP with various randomly chosen initial distribution. Plot shows probability of choosing controller $K_1$ averaged over 20 trials}.
    \label{fig:NAC chain MDP}
\end{figure*}
\section{Additional Comments}
\begin{itemize}
    
\item \textit{Comment on the `simple' experimental settings.} The motivating examples may seem ``simple" and trainable from scratch with respect to progress in the field of RL. However, our main point is that there are situations where, for example, one may have trained controllers for a range of environments {\em in simulation}. However, the real life environment may differ from the simulated ones. We demonstrate that exploiting such basic pre-learnt controllers via our approach can help in generating a better (meta) controller for a new, unseen environment, instead of learning a new controller for the new environment from scratch.

\item \textit{On characterizing the performance of the optimal mixture policy.} As correctly noticed by the reviewer, the inverted pendulum experiment showed that the optimal mixture policy can vastly outperform the component controllers. Currently, however, we do not provide any theoretical guarantees regarding this, since this depends on the structure of the policy space and the underlying MDP, which is very challenging. We hope to explore this task in our future work.
\end{itemize}
\section{Discussion}
 \label{sec:conclusion-and-discussion}
  We have considered the problem of using a menu of baseline controllers and combining them using improper probabilistic mixtures to form a superior controller. In many relevant MDP learning settings, we saw that this is indeed possible, and the policy gradient and actor-critic based analyses indicate that this approach may be widely applicable. %
This work opens up a plethora of avenues. One can consider a richer class of mixtures that can look at the current state and mix accordingly. For example, an attention model can be used to choose which controller to use, or other state-dependent models can be relevant. 
Another example is to artificially force switching across controllers to occur less frequently than in every round. The can help create \emph{momentum} and allow the controlled process to 'mix' better, when using complex controllers.

A few caveats are in order regarding the potential societal impact and consequences of this work. As such, this paper offers a way of combining or `blending' a given class of decision-making entities in the hope of producing a `better' one. In this process, the definitions of what constitutes `optimal' or `expected' behavior from a policy are likely to be subjective, and may encode biases and attitudes of the system designer(s). More importantly, it is possible that the base policy class (or some elements of it) have undesirable properties to begin with (e.g., bias or insensitivity), which could get amplified in the improper learning process as an unintended outcome. We sound ample caution to practitioners who contemplate adopting this method. 

Finally, in the present setting, the base controllers are fixed. It would be interesting to consider adding adaptive, or 'learning' controllers as well as the fixed ones. Including the base controllers can provide baseline performance below which the performance of the learning controllers would not drop.

\end{document}